\documentclass{article}

\usepackage{microtype}
\usepackage{graphicx}
\usepackage{subfigure}
\usepackage{booktabs} 

\usepackage{hyperref}

\usepackage[accepted]{icml2021}

\usepackage[english]{babel}
\usepackage{amsfonts}
\usepackage{wrapfig}
\usepackage{bbold}
\usepackage{mathtools}
\usepackage{algorithmic}
\usepackage{algorithm}
\usepackage{custom_commands}
\usepackage[dvipsnames]{xcolor}
\usepackage{enumitem}
\usepackage{amsmath,amssymb,amsfonts,amsthm}

\usepackage{textcomp}
\usepackage{xcolor}
\usepackage{multirow}
\usepackage{natbib}
\usepackage{thmtools, thm-restate}
\usepackage{textcomp}
\usepackage[textsize=scriptsize, textwidth=1.3cm]{todonotes}

\newtheorem{theorem}{Theorem}

\newtheorem{lemma}{Lemma}
\newtheorem{definition}{Definition}
\newtheorem{assumption}{Assumption}
\newtheorem{proposition}{Proposition}
\theoremstyle{remark}
\newtheorem*{remark}{Remark}

\def\E{\mathbb{E}}
\newcommand{\markovkernel}{Markov Kernel}
\newcommand{\tadv}{\mathsf {t_{adv}}}

\newcommand{\flalgo}{\mathcal{M}}
\newcommand{\clsfier}{h}
\newcommand{\smoothclsfier}{h_s}
\newcommand{\cerradius}{\mathsf{RAD}}

\newcommand{\ourframework}{CRFL}
\newcommand{\fdivergence}{\textit{f}-divergence}
\newcommand{\kldivergence}{KL divergence}
\newcommand{\locals}{\underline{s}}
\newcommand{\mnist}{MNIST}
\newcommand{\emnist}{EMNIST}
\newcommand{\loan}{LOAN}

\newcommand{\ceracc}{certified accuracy}
\newcommand{\cerrate}{certified rate}

\makeatletter
\renewcommand{\maketag@@@}[1]{\hbox{\m@th\normalsize\normalfont#1}}%
\makeatother

\icmltitlerunning{CRFL: Certifiably Robust Federated Learning against Backdoor Attacks}

\begin{document}

\twocolumn[
\icmltitle{CRFL: Certifiably Robust Federated Learning against Backdoor Attacks}




\begin{icmlauthorlist}
\icmlauthor{Chulin Xie}{uiuc}
\icmlauthor{Minghao Chen}{zju}
\icmlauthor{Pin-Yu Chen}{ibm}
\icmlauthor{Bo Li}{uiuc}
\end{icmlauthorlist}

\icmlaffiliation{uiuc}{University of Illinois at Urbana-Champaign}
\icmlaffiliation{zju}{Zhejiang University}
\icmlaffiliation{ibm}{IBM Research}

\icmlcorrespondingauthor{Chulin Xie}{chulinx2@illinois.edu}
\icmlcorrespondingauthor{Pin-Yu Chen}{pin-yu.chen@ibm.com}
\icmlcorrespondingauthor{Bo Li}{lbo@illinois.edu}

\icmlkeywords{Machine Learning, ICML}

\vskip 0.3in
]



\printAffiliationsAndNotice{}  

\begin{abstract}
Federated Learning (FL) as a distributed learning paradigm that aggregates information from diverse clients to train a shared global model, has demonstrated great success. 
However, malicious clients can perform poisoning attacks and  model replacement to introduce backdoors into the trained global model.
Although there have been intensive studies designing robust aggregation methods and empirical robust federated training protocols against backdoors, existing approaches lack  \textit{robustness certification}. 
This paper provides the first general framework, Certifiably Robust Federated Learning (CRFL), to train certifiably robust FL models against backdoors. 
Our method exploits clipping and smoothing on model parameters to control the global model smoothness, which yields a sample-wise robustness certification on backdoors with limited magnitude. 
Our certification also specifies the relation to federated learning parameters, such as poisoning ratio on instance level, number of attackers, and training iterations. 
Practically, we conduct comprehensive experiments across a range of federated datasets, and provide the ﬁrst benchmark for certiﬁed robustness against backdoor attacks in federated learning. Our code is publicaly available at \href{https://github.com/AI-secure/CRFL}{https://github.com/AI-secure/CRFL}.
\end{abstract}

\section{Introduction}
\label{submission}

Federated learning (FL) has been widely applied to different applications given its high efficiency and privacy-preserving properties~\cite{smith2017federated,mcmahan2016communication,zhao2018federated}.
However, recent studies show that it is easy for the local client to add adversarial perturbation such as ``backdoors" during training to compromise the final aggregated model~\cite{bhagoji2018analyzing,bagdasaryan2020backdoor,wang2020attackthetails,xie2019dba}. Such attacks raise great security concerns and have become the roadblocks towards the real-world deployment of federated learning.

Although there have been intensive studies on robust FL by designing robust aggregation methods~\cite{fung2020limitations,pillutla2019robustrfa,fu2019attack,blanchard2017machinekrum,el2018hidden,chen2017distributed,yin2018byzantine}, developing empirically robust federated training protocols (e.g., gradient clipping~\cite{sun2019can}, leveraging noisy perturbation~\cite{sun2019can} and additional evaluation during training~\cite{andreina2020baffle}), current defense approaches lack robustness guarantees against the backdoor attacks under certain conditions. To the best of our knowledge, certified robustness analysis and algorithms for FL against backdoor attacks remain elusive.



\begin{figure}[t]
\setlength{\belowcaptionskip}{-4mm}
\centering
	\includegraphics[width=1.05\linewidth]{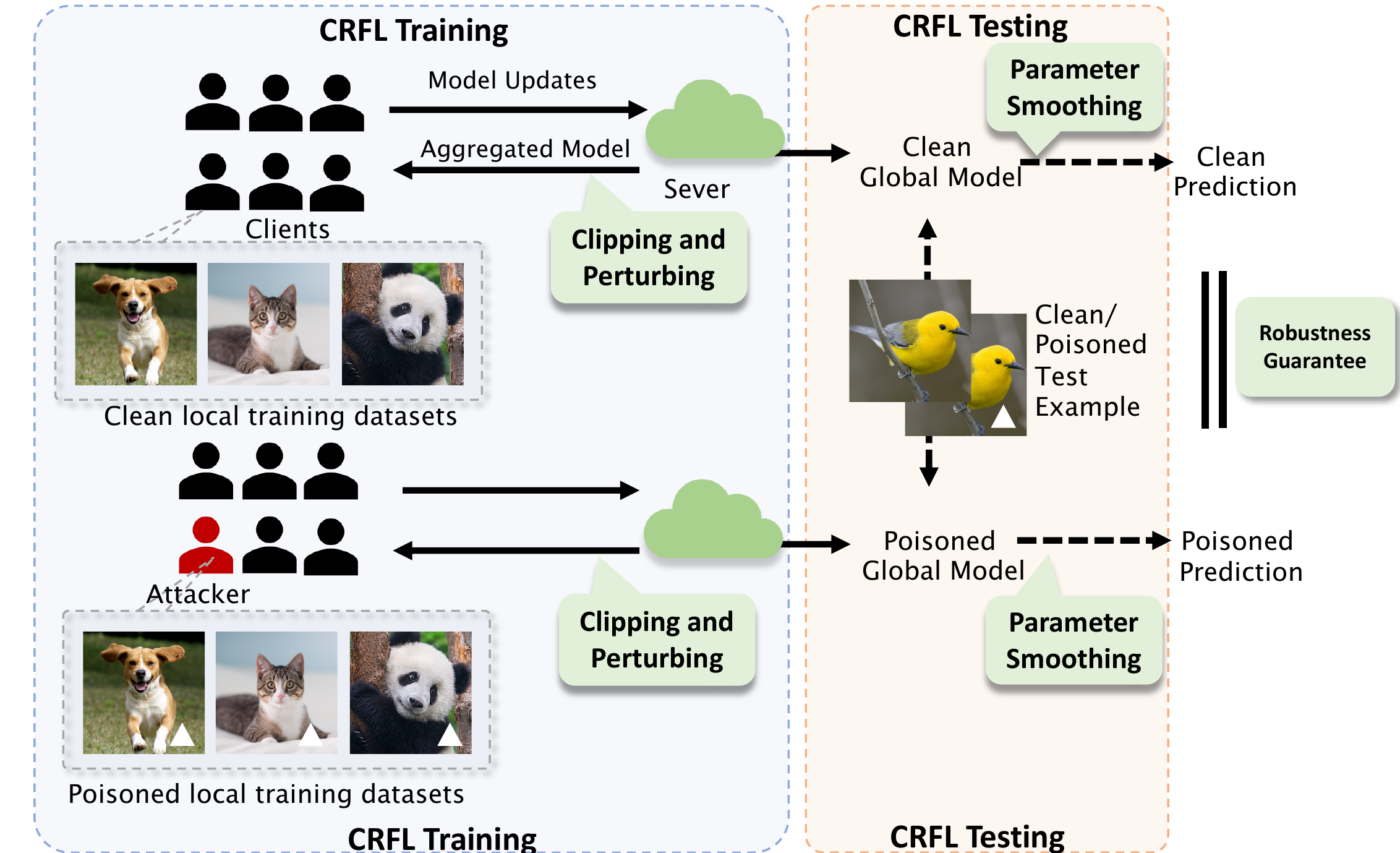}
\vspace{-5mm}
\caption{Overview of certifiably robust federated learning (CRFL)}
\label{fig:teaser} 
\vspace{-6mm}
\end{figure}

To bridge this gap, in this work we propose a certifiably robust federated learning (CRFL) framework as illustrated in Figure~\ref{fig:teaser}.
In particular, during \textit{training}, we allow the local agent to update their model parameters to the center server, and the server will: 1) aggregate the collected model updates, 2) clip the norm of the aggregated model parameters, 3) add a random noise to the clipped model, and finally 4) send the new model parameters back to each agent. 
Note that all of the operations are conducted on the server side to reduce the load for local clients and to prevent malicious clients.
During \textit{testing}, the server will smooth the final global model with randomized parameter smoothing and make the final prediction based on the parameter-smoothed model.



Using CRFL, we theoretically prove that the trained global model would be certifiably robust against backdoors as long as the backdoor is within our certified bound. To obtain such robustness certification, we first quantify the closeness of models aggregated in each step by viewing this process as a Markov Kernel~\cite{asoodeh2020differentially,makur2019informationphdmit,polyanskiy2015dissipation,polyanskiy2017strong}. We then leverage the model closeness together with the parameter smoothing procedure to certify the final prediction. 
Empirically, we conduct extensive evaluations on MNIST, EMNIST, and financial datasets to evaluate the certified robustness of CRFL and study how FL parameters affect certified robustness. 


\underline{\bf Technical Contributions.} In this paper, we take the \textit{first} step
towards providing certified robustness for FL against backdoor attacks.
We make contributions on both theoretical
and empirical fronts.   
\begin{itemize}[leftmargin=*,itemsep=-0.5mm]
    \vspace{-1em}
    \item We propose the first certifiably robust federated learning (CRFL) framework against backdoor attacks.
    \item Theoretically, we analyze the training dynamics of the aggregated model via Markov Kernel, and propose parameter smoothing for model inference. Altogether, we prove the certified robustness of CRFL.
    \item We conduct extensive experiments on MNIST, EMNIST, and financial datasets to show the effect of different FL parameters (e.g. poisoning ratio, number of attackers, and training iterations) on certified robustness.
\end{itemize}

\section{Related work}
\paragraph{Backdoor Attacks on Federated Learning}

The goal of backdoor attacks against federated learning is to train strong poisoned local models and submit malicious model updates to the central server, so as to mislead the global model~\cite{bhagoji2018analyzing}.
\cite{bagdasaryan2020backdoor} studies the model replacement approach, where the attacker scales malicious model updates to replace the global model with local backdoored one.
\cite{xie2019dba} exploit the decentralized nature of federated learning and propose a distributed backdoor attack.

\paragraph{Robust Federated Learning}
In order to nullify the effects of attacks while aggregating client updates, a number of robust aggregation algorithms have been proposed for distributed learning~\cite{fung2020limitations,pillutla2019robustrfa,fu2019attack,blanchard2017machinekrum,el2018hidden,chen2017distributed,yin2018byzantine}. These methods either identify and down-weight the malicious updates through certain distance or similarity metrics, or estimate a true ``center’’ of the received model updates rather than taking a weighted average. However, many of those methods assume that the data distribution is i.i.d cross distributed clients, which is not the case in FL setting.
Other defenses are several robust federated protocols that mitigate poisoning attacks during training. 
\cite{andreina2020baffle} incorporates an additional validation phase to each round of FL to detect backdoor.
\cite{sun2019can} show that clipping the norm of model updates and adding Gaussian noise can mitigate backdoor attacks that are based on the model replacement paradigm. None of these provides certified robustness guarantees.

A concurrent work~\cite{cao2021provably} proposes Ensemble FL for provable secure FL against malicious clients, which requires training \emph{hundreds of} FL models and  focuses on client-level certification. Our work allows standard FL protocol, and our certification is applicable to feature, sample, and client levels.

\section{Preliminaries}
\label{sec_prelim}

\subsection{Federated Averaging}
\paragraph{Learning Objective}
Suppose the model parameters are denoted by $w\in \mathbb{R}^d$, we consider the following distributed optimization problem:
$
    \min_{  w\in \mathbb{R}^d} \{F(  w) \triangleq \sum_{i=1}^N p_i F_i( w) \ \},
$
where $N$ is the number of clients, and $p_i$ is the aggregation weight of the $i$-th client such that $p_i\ge 0$ and $\sum_{i=1}^N p_i=1$.
Suppose the $i$-th client holds $n_i$ training data in its local dataset $S_i = \{z_{1}^i, z_{2}^i,\ldots, z_{n_i}^i\}$. The local objective $F_i(\cdot)$ is defined by
$
    F_i( w) \triangleq \frac{1}{n_i} \sum_{j=1}^{n_i} \ell( w; z_j^i),
$
where $\ell(\cdot;\cdot)$ is a defined learning loss function.

\paragraph{One Round of Federated Learning (Periodic Averaging SGD)}
In federated learning, the clients are able to perform multiple local iterations to update the local models~\cite{mcmahan2017communication}. So we formulate the SGD problem in FL as Periodic Averaging SGD~\cite{wang2019cooperative,li2019convergence}. 
Specifically, at round $t$, first, the central sever sends current global model $w_{t-1}$ to all clients. 
Second, every client $i$ initializes its local model $  w_{(t-1)\tau_i}^i =   w_{t-1}$ and then performs $\tau_i~( \tau_i \ge 1)$ local updates, such that
    $w_{s}^i \gets   w_{s-1}^i - \eta_i g_i(  w_{s-1}^i; \xi_{s-1}^i),~s= (t-1)\tau_i+1, (t-1)\tau_i+2, \ldots, t\tau_i,$
where $\eta_i$ is the learning rate, $\xi_{s}^i \subset S^i$ are randomly sampled mini-batches with batch size $n_{B_i}$, and $g_i(  w; \xi^i) = \frac{1}{n_{B_i}} \sum_{z_j^i \in \xi^i } \nabla \ell (  w; z_j^i) $ denotes the stochastic gradient. 
The local clients send the local model updates $w_{t\tau_i}^i -  w_{t-1}$ to the server. 
Finally, the server aggregates over the local model updates into the new global model $ w_{t}$ such that
 $     w_{t} \gets  w_{t-1} + \sum_{i=1}^N p_i   (w_{t\tau_i}^i -  w_{t-1}) .$

\subsection{Threat Model}\label{subsection:threat_model}
The goal of backdoor is to inject a backdoor pattern during training such that any test input with such pattern will be mis-classified as the target label~\cite{gu2019badnets}. The purpose of backdoor attacks in FL is to manipulate local models and simultaneously fit the main task and backdoor task, so that the global model would behave normally on untampered data samples while achieving high attack success rate on backdoored data samples. 
We consider the backdoor attack via model replacement approach~\cite{bagdasaryan2020backdoor} where the attackers train the local models using the poisoned datasets, and scale the malicious updates before sending it to the sever.
Suppose there are $R$ adversarial clients out of $N$ clients, we assume that each of them only attack once and they perform model replacement attack together at the same round $\tadv$. Such distributed yet coordinated backdoor attack is shown to be effective in~\cite{xie2019dba}.

Let $D:=\{S_1,S_2,\ldots,S_N \}$ be the union of original benign local datasets in all clients.
For a data sample $z_j^i:=\{x_j^i,y_j^i\}$ in $S_i$, 
we denote its backdoored version as 
${z'}_j^i:=\{x_j^i+{{\delta_i}_x}, y_j^i+{{\delta_i}_y}\}$,
and the backdoor as $\delta_i :=\{{\delta_i}_x,{\delta_i}_y \}$. We assume an adversarial client $i$ has $q_i$ backdoored samples in its local dataset $S'^i$ with size $n_i$.
Let $D':=\{{S'}_1,\ldots, {S'}_{R-1}, {S'}_{R}, {S}_{R+1},\ldots,S_N \}$ be the union of local datasets in the adversarial round $\tadv$.
Then we have
$D'= D+ \{\{\delta_i\}_{j=1}^{q_i}\}_{i=1}^R. $

Before $\tadv$, the adversarial clients train the local model using original benign datasets. 
When $t=\tadv$, for adversarial client $i$, each local iteration is trained on the backdoored local dataset $S'_i$ such that
  ${w'}_{s}^i \gets   {w'}_{s-1}^i - \eta_i g_i( {w'}_{s-1}^i; 
  {\xi'}_{s-1}^i), \\
  s=(t-1)\tau_i+1,(t-1)\tau_i+2,\ldots,t\tau_i, $
where $w'$ is the malicious model parameters, ${\xi'}^i$ is the mini-batch sampled from $S'_i$ and the local model is initialized as $ {w'}_{(t-1)\tau_i}^i = w_{t-1}$.
Following~\cite{bagdasaryan2020backdoor}, we assume the attacker add a fixed number of backdoored samples $q_{B_i}$ in each training batch, then the mini-batch gradient is  ${g}_i(w'; {\xi'}^i) = \frac{1}{n_{B_i}} (\sum_{j=1}^{q_{B_i}} \nabla \ell (  w'; {z'}_j^i) + \sum_{j=q_{B_i}+1}^{n_{B_i}} \nabla \ell (  w'; {z}_j^i) ) $. The poison ratio of dataset $S'_i$ is $q_{B_i}/n_{B_i} = q_{i}/n_{i}$.
Since for each local iteration, the local model is updated with backdoored mini-batch samples, more local iterations will drive the local model ${w'}_{s}^i$ farther from the corresponding one ${w}_{s}^i$ in benign training process. 
Then the adversarial clients scale their malicious local updates before submitting to the server. Let the scale factor be $\gamma_i$ for $i$-th adversarial client, then the scaled update is
$
 {\gamma_i}( {w'}_{\tadv\tau_i}^i -  {w}_{\tadv-1}).
$
The server aggregates over the malicious and benign updates into an infected global model ${w'}_t$ such that
 $     {w'}_{t} \gets  w_{t-1} + \sum_{i=1}^R p_i  \gamma_i ({w'}_{t\tau_i}^i -  w_{t-1}) + \sum_{i=R+1}^N p_i  (w_{t\tau_i}^i -  w_{t-1}). $
 In fact, even though the adversarial clients only attack at round $\tadv$ and in the later rounds ($t \textgreater \tadv$) they use the original benign datasets, the global model is already infected starting from $\tadv$, so we still denote the global model parameters as ${w'}_t$ in later rounds.

\section{Methodology}
In this Section, we introduce our proposed framework \ourframework{}, which is composed of a training-time subroutine (Algorithm~\ref{algo:parameters_perturbation}) and a test-time subroutine (Algorithm~\ref{alg:certify_parameters_perturbation}) for achieving certified robustness.

\subsection{CRFL Training: Clipping and Perturbing}
During training, at round $t=1,\ldots, T-1$, local clients update their models, 
and the server performs aggregation.
Then, in our training protocol, the server clip the model parameters 
$\mathrm{Clip}_{\rho_t}(w_{t}) \gets w_{t}/ \max(1, \frac{\|w_{t}\|}{\rho_t})$ so that its norm is bounded by $\rho_t$,
and then add isotropic Gaussian noise $\epsilon_t \sim \mathcal N(0,\sigma_t^2\bf I) $ directly on the aggregated global model parameters (coordinate-wise noise):
 $ \widetilde w_{t} \gets \mathrm{Clip}_{\rho_t}({w_{t}}) + \epsilon_t.$
Throughout this paper, $\|\cdot\| $ denotes the $\ell_2$ norm $\|\cdot\|_2$.
In the next round $t+1$, client $i$ initializes its local model with noisy new global model $  w_{t\tau_i}^i \gets  \widetilde w_{t}$.
In the final round $T$, we only clip the global model parameters. 

The procedure is summarized in Algorithm \ref{algo:parameters_perturbation} and denoted by
 $\mathcal{M}$, which outputs the global model parameters $\mathrm{Clip}_{\rho_T}(w_{T})$. Then we define
$
\mathcal{M}(D) := \mathrm{Clip}_{\rho_T}(w_{T}).
$

\begin{algorithm}[ht]
    \caption{\small Federated averaging with parameters clipping and perturbing}
    \label{algo:parameters_perturbation}
     \scalebox{0.85}{
    \begin{minipage}{1.0\linewidth}
    \begin{algorithmic}
      \STATE {\bfseries Server's input:} initial model parameters $w_0$, $\widetilde w_0  \gets w_0$
      \STATE {\bfseries Client $i$'s input:} local dataset $S_i$ and learning rate $\eta_i$
      \FOR{each round $t=1, \ldots, T$}
      \STATE The server sends $\widetilde w_{t-1}$ to the $i$-th client
      \FOR{client $i=1, 2,  \ldots, N$ in parallel}
      \STATE initialize local model $w_{(t-1)\tau_i} ^i \gets \widetilde w_{t-1} $
      \FOR{local iteration $ s =(t-1)\tau_i+1,  \ldots, t\tau_i  $}
      \STATE compute mini-batch gradient ${g_i}(\cdot;\cdot)$
        \STATE $w_{s}^i \gets w_{s-1}^i - \eta_i g_i(w_{s-1}^i; \xi_{s-1}^i)  $
      \ENDFOR
      \STATE The $i$-th client sends $w_{t\tau_i}^i - \widetilde w_{t-1}$ to the server
      \ENDFOR
      \STATE The server updates the model parameters
      \STATE  $w_t  \gets \widetilde w_{t-1}+ \sum \limits_{i = 1}^N p_i (w_{t\tau_i}^i-\widetilde w_{t-1}) $
      \STATE The server clips the model parameters
      \STATE $\mathrm{Clip}_{\rho_t}(w_{t}) \gets w_{t}/ \max(1, \frac{\|w_{t}\|_2}{\rho_t})$
      \STATE The server adds noise
      \IF{$t \le T-1$}
      \STATE $\epsilon_t \gets$ \text{a sample drawn from } $\mathcal N(0,\sigma_t^2\bf I)$
      \STATE $\widetilde w_{t} \gets \mathrm{Clip}_{\rho_t}(w_{t}) + \epsilon_t$ 
      \ENDIF
      \ENDFOR
      \STATE {\bfseries Output:}  Clipped global model parameters $\mathrm{Clip}_{\rho_T}(w_{T})$
    \end{algorithmic}
    \end{minipage}%
    }
\end{algorithm}

\subsection{CRFL Testing: Parameter Smoothing}

\paragraph{Smoothed Classifiers} 
We study multi-class classification models and define a classifier $\clsfier:\mathcal {(W,X)} \rightarrow \mathcal{Y} $ with finite set of label $ \mathcal{Y}=\{1,\ldots, C\}$, where $C$ denotes the number of classes.
We extend the randomized smoothing method~\cite{cohen2019certifiedrandsmoothing} to \emph{parameter smoothing} for constructing a new, ``smoothed'' classifier $\smoothclsfier$ from an arbitrary base classifier $\clsfier$. The robustness properties can be verified using the smoothed classifier $\smoothclsfier$.
Given the model parameter $w$ of $\clsfier$, when queried at a test sample $x_{test}$, we first take a majority vote over the predictions of the base classifier $\clsfier$ on random model parameters drawn from a probability distribution $\mu$, i.e., the smoothing measure, to obtain the  ``votes'' ${H_s^{c}}(w;x_{test})$ for each class $c \in  \mathcal{Y}$. Then the label returned by the smoothed classifier $\smoothclsfier$ is the mostly probable label among all classes (the majority vote winner). Formally,
\begin{equation} \label{eq:define_smoothed_cls}
\begin{aligned}
  &{\smoothclsfier}(w;x_{test})= \arg \max_{c\in \mathcal{Y}  } {H_s^{c}}(w;x_{test}), \\
   & \text{where } {H_s^{c}}(w;x_{test})= \mathbb{P}_{W\sim \mu(w)}[\clsfier(W;x_{test})=c].  
\end{aligned}
\end{equation}
To be aligned with the training time Gaussian noise (perturbing), we also adopt Gaussian smoothing measures $\mu(w) = \mathcal N(w,{\sigma_T}^2\bf I)$ during testing time. In practice, the exact value of the probability $p_c=\mathbb{P}_{W\sim \mu(w)}[\clsfier(W;x_{test})=c]$ for label $c$ is difficult to obtain for neural networks, and hence we resort to Monte Carlo estimation~\cite{cohen2019certifiedrandsmoothing,lecuyer2019certifieddp} to get its approximation $\hat p_c$. 
At round $t=T$, given the clipped aggregated global model $\mathrm{Clip}_{\rho_T}({w_T})$, we add Gaussian noise $\epsilon_T^k \sim \mathcal N(0,\sigma_T^2\bf I)$ for $M$ times to get $M$ sets of noisy model parameters ($M$ Monte Carlo samples for estimation), such that
  $\widetilde w_T^k \gets \mathrm{Clip}_{\rho_T}({w_T}) + \epsilon_T^k,~k=1, 2, \ldots, M.$

In Algorithm~\ref{alg:certify_parameters_perturbation},
The function $\mathrm{ GetCounts }$ runs the classifier with each set of noisy model parameters $w_T^k$ for one test sample $x_{test}$, and returns a vector of class counts. 
Then we take the most probable class $\hat c_A$ and the runner-up class $\hat c_B$ to calculate the corresponding $\hat p_A$ and $\hat p_B$.
The function $\mathrm{CalculateBound}$ calibrates the empirical estimation to bound the probability $\alpha$ of $\smoothclsfier$ returning an incorrect label. Given the error tolerance $\alpha$, we use Hoeffding’s inequality~\cite{hoeffding1994probability} to compute a lower bound $\underline{p_A}$ on the probability $H_s^{c_A}(w;x_{test})$ and a upper bound $\overline{p_B}$ on the probability $H_s^{c_B}(w;x_{test})$ according to
$\underline{p_A}= \hat p_A - \sqrt{\frac{\log(1/\alpha)}{2N}}$, $\overline{p_B}= \hat p_B + \sqrt{\frac{\log(1/\alpha)}{2N}}$.
We leave the function $\mathrm{CalculateRadius}$ to be defined with our main results in later sections and we will analyze the robustness properties of the model trained and tested under our framework \ourframework{}.

\begin{algorithm}[ht]
  \caption{\small Certification of parameters smoothing}
  \label{alg:certify_parameters_perturbation}
   \scalebox{0.85}{
    \begin{minipage}{1.0\linewidth}
        \begin{algorithmic}
          \STATE {\bfseries Input:} a test sample $x_{test}$ with true label $y_{test}$, the global model parameters $\mathrm{Clip}_{\rho_T}(w_{T})$, the classifier $h(\cdot, \cdot)$
          \FOR{$k=0,1, \ldots M$}
          \STATE $\epsilon_T^k \gets $ \text{a sample drawn from } $\mathcal N(0,\sigma_T^2\bf I)$ 
          \STATE $\widetilde w_{T}^{k} = \mathrm{Clip}_{\rho_T}(w_{T}) + \epsilon_T^k$ 
          \ENDFOR
          \STATE Calculate empirical estimation of $p_A,p_B$ for $x_{test}$
          \STATE $ \texttt{counts} \gets \mathrm{ GetCounts }(x_{test}, \{ \widetilde w_{T}^{1}, \ldots, \widetilde w_{T}^{M}\} )$
          \STATE $\hat c_A , \hat c_B  \gets $ top two indices in $\texttt{counts}$     
          \STATE $\hat p_A,  \hat p_B  \gets \texttt{counts}[\hat c_A]/M,  \texttt{counts}[\hat c_B]/M $
          \STATE Calculate lower and upper bounds of $p_A,p_B$
          \STATE $ \underline{p_A} , \overline{p_B}  \gets \mathrm{ CalculateBound }$($\hat p_A$, $\hat p_B$, $N, \alpha$) 
          \IF{$ \underline{p_A} > \overline{p_B} $}
          \STATE $\cerradius =\mathrm{CalculateRadius}(\underline{p_A} , \overline{p_B})$
          \STATE {\bfseries Output:} Prediction $\hat c_A$ and certified radius $\cerradius$
          \ELSE
          \STATE {\bfseries Output:} $\texttt{ABSTAIN}$  and 0
          \ENDIF
        \end{algorithmic}
    \end{minipage}%
}
\end{algorithm}
\paragraph{Comparison with Certifiably Robust Models in Centralized Setting} 
Our method is different from previous certifiably robust models in centralized learning against evasion attacks~\cite{cohen2019certifiedrandsmoothing} and backdoors~\cite{weber2020rab}.
Once the $M$ noisy models (at round $T$, with $\sigma_T$) are generated, they are fixed and used for every test sample during test time, just like RAB \cite{weber2020rab} in the centralized setting. However, RAB actually trains $M$ models using $M$ noise-corrupted datasets, while we just train one model through FL and finally generated $M$ noise-corrupted copies of it.
For every test sample, randomized smoothing \cite{cohen2019certifiedrandsmoothing} generates $M$ noisy samples. Suppose the test set size is $m$. Then during testing, there are $m \cdot M$ times noise addition on test samples for randomized smoothing, and $M$ times noises addition on trained model for CRFL. To our best knowledge, this is the first work to study \textit{parameter} smoothing rather than input smoothing, which is an open problem motivated by the FL scenario, since the sever directly aggregates over the model parameters.

\section{Certified Robustness of CRFL}

\subsection{Pointwise Certified Robustness}
\paragraph{Goal of Certification}
In the context of data poisoning in federated learning, the goal is to protect the global model against adversarial data modification made to the local training sets of distributed clients. Thus, the goal of certifiable robustness in federated learning is for each test point, to return a prediction as well as a certificate that the prediction would not change had some features in (part of) local training data of certain clients been modified.

Following our threat model in Section~\ref{subsection:threat_model} and our training protocol in Algorithm~\ref{algo:parameters_perturbation}, we define the trained global model
$
\mathcal{M}(D') := \mathrm{Clip}_{\rho_T}({w'}_{T}).
$
For the FL training process that is exposed to model replacement attack, when the distance between $D'$ (backdoored dataset) and $D$ (clean dataset) is under certain threshold (i.e., the magnitude of $\{\{\delta_i\}_{j=1}^{q_i}\}_{i=1}^R$ is bounded), we can certify that $\mathcal{M}(D')$ is ``close'' to $\mathcal{M}(D)$  and thus is robust to backdoors.
The rationale lies in the fact that we perform clipping and noise perturbation on the model parameters to control the global model deviation during training.
During testing, intuitively, under the Gaussian smoothing measures $\mu$ as described in Algorithm~\ref{alg:certify_parameters_perturbation}, for two close distribution $\mu(\mathcal{M}(D'))$ and $\mu(\mathcal{M}(D))$, we would expect that even though the probabilities for each class $c$, i.e., $ {H_s^{c}}(\mathcal{M}(D');x_{test})$ and $ {H_s^{c}}(\mathcal{M}(D);x_{test})$, may not be equal, the returned most likely label ${\smoothclsfier}(\mathcal{M}(D');x_{test})$ and ${\smoothclsfier}(\mathcal{M}(D);x_{test})$ should be consistent.

In summary, we aim to develop a robustness certificate by studying under what condition for $\{\{\delta_i\}_{j=1}^{q_i}\}_{i=1}^R$ that the prediction for a test sample is consistent between the smoothed FL models trained from $D$ and $D'$ separately, i.e., ${\smoothclsfier}(\mathcal{M}(D');x_{test})= {\smoothclsfier}(\mathcal{M}(D);x_{test})$. 
To put forth our certified robustness analysis, we make the following assumptions on the loss function of all clients.
Then we present our main theorem and explain its derivation through \textit{model closeness} and \textit{parameter smoothing}. 
Throughout this paper, we denote $ \nabla_w \ell(w;z)$ as $ \nabla \ell(w;z)$ for simplicity.

\begin{assumption}[Convexity and Smoothness] \label{assumption:Smoothness}
The loss function $\ell(w;z)$ is $\beta$-smoothness,
i.e, $\forall w_1,w_2$,
$$
 \| \nabla \ell(w_1;z) -  \nabla \ell(w_2;z)\| \leq \beta \| w_1 - w_2\|.
$$
In addition, the loss function $\ell(w;z)$ is convex. Then co-coercivity of the gradient states:
\begin{align*}
& \| \nabla \ell(w_1;z) -\nabla \ell(w_2;z)\|^2  \\ &\leq \beta \langle w_{1} - w_{2},\nabla \ell(w_1;z)-  \nabla \ell(w_2;z) \rangle.
\end{align*}
\end{assumption}
\begin{assumption}[Lipschitz Gradient w.r.t. Data] \label{assumption:data_lipschitz}
The gradient $\nabla_w l(z;w)$ is $L_{\mathcal Z}$ Lipschitz with respect to the argument $z$ and norm distance $\|\cdot\|$, i.e, $\forall z_1,z_2$,
\begin{equation}
    \| \nabla \ell(w;z_1) -  \nabla \ell(w;z_2)\| \leq L_{\mathcal Z} \| z_1 - z_2\|. \nonumber
\end{equation}
\end{assumption}

\begin{assumption}\label{assumption:fl_system_train_test}
The whole FL system follows Algorithm~\ref{algo:parameters_perturbation} to train and Algorithm~\ref{alg:certify_parameters_perturbation} to test.
\end{assumption}

The assumptions on convexity and smoothness are common in the analysis of distributed SGD \cite{li2019convergence,wang2019cooperative}. We also make assumption on the Lipschitz gradient w.r.t. data, which is used in~\cite{fallah2020personalized,reisizadeh2020robustdistributionshifts}  for analyzing the heterogeneous data distribution across clients.

\vspace{-2mm}
\paragraph{Main Results}\label{sec:main_results}

\begin{restatable}[General Robustness Condition]{theorem}{thmrobustnessthreelevel}
\label{theorem_robustness_3_level}

Let $\smoothclsfier$ be defined as in Eq.~\ref{eq:define_smoothed_cls}. When $\eta_i \le \frac{1}{\beta}$ and Assumptions~\ref{assumption:Smoothness},~\ref{assumption:data_lipschitz}, and~\ref{assumption:fl_system_train_test} hold, suppose $c_A \in \mathcal{Y} $ and $\underline{p_A}, \overline{p_B} \in [0,1]$ satisfy
\begin{equation} 
 {H_s^{c_A}}(\flalgo(D');x_{test}) \ge \underline{p_A} \ge \overline{p_B} \ge \max_{c\ne c_A} {H_s^{c}}(\flalgo(D');x_{test}), \nonumber
\end{equation}
then if 
\begin{small}
\begin{equation}
\begin{aligned}
&R\sum_{i=1}^R  (p_i \gamma_i \tau_i \eta_i \frac{{{q_B}_i}}{{{n_B}_i}}  \|\delta_i\|)^2 
&\le \frac{-\log \left(  1- (\sqrt{\underline{p_A}} - \sqrt{\overline{p_B} })^2 \right) \sigma_{\tadv}^2 }{2    L_{\mathcal Z}^2 \prod\limits_{t=\tadv+1}^{T}  \left(2\Phi \left (\frac{\rho_t }{\sigma_{t}}\right)-1 \right)}, \nonumber \\
\end{aligned}
\end{equation}
\end{small}
it is guaranteed that 
\begin{equation} 
     {\smoothclsfier}(\flalgo(D');x_{test}) = {\smoothclsfier}(\flalgo(D);x_{test}) = c_A, \nonumber
\end{equation}
where $\Phi$ is standard Gaussian's cumulative density function (CDF) and the other parameters are defined in Section \ref{sec_prelim}.
\end{restatable}

In practice, since the server does not know the global model in the current FL system is poisoned or not, we assume the model is already backdoored and  derive the condition when its prediction will be certifiably consistent with the prediction of the clean model. 
Our certification is on three levels: \textit{feature}, \textit{sample}, and \textit{client}. 
If the magnitude of the backdoor is upper bounded for every attackers, then we can re-write the Theorem~\ref{theorem_robustness_3_level} as the following corollary.
\begin{restatable}[Robustness Condition in Feature Level]{corollary}{corrobustpixellevel}
\label{corollary_robustness_pixel_level}
Using the same setting as in Theorem \ref{theorem_robustness_3_level} but further assume identical backdoor magnitude $\|\delta\|= \|\delta_i\| $ for $i=1,\ldots,R$.
Suppose $c_A \in \mathcal{Y} $ and $\underline{p_A}, \overline{p_B} \in [0,1]$ satisfy
$$
 {H_s^{c_A}}(\flalgo(D');x_{test}) \ge \underline{p_A} \ge \overline{p_B} \ge \max_{c\ne c_A} {H_s^{c}}(\flalgo(D');x_{test}),
$$
then $ {\smoothclsfier}(\flalgo(D');x_{test})= {\smoothclsfier}(\flalgo(D);x_{test}) = c_A$ for all $\|\delta\| < \cerradius$, where
\begin{small}
\begin{equation}\label{eq:certified_radius_R}
\begin{aligned}
\cerradius = \sqrt{\frac{-\log \left(  1- (\sqrt{\underline{p_A}} - \sqrt{\overline{p_B} })^2 \right) \sigma_{\tadv}^2 }{2 R  L_{\mathcal Z}^2 \sum\limits_{i=1}^R ( p_i \gamma_i \tau_i  \eta_i \frac{{{q_B}_i}}{{{n_B}_i}})^2 \prod\limits_{t=\tadv+1}^{T}  \left(2\Phi \left (\frac{\rho_t }{\sigma_{t}}\right)-1 \right) }} \\  
\end{aligned}
\end{equation}
\end{small}
\vspace{-4mm}
\end{restatable}
The function $\mathrm{CalculateRadius}$ in our Algorithm~\ref{alg:certify_parameters_perturbation} can calculate the certified radius $\cerradius$ according to Corollary~\ref{corollary_robustness_pixel_level}.

We now make several remarks about Corollary~\ref{corollary_robustness_pixel_level} and {will verify them in our experiments}: 1)  The noise level $\sigma_t$ and the parameter norm clipping threshold $\rho_t$  are hyper-parameters that can be adjusted to control the robustness-accuracy trade-off. For instance, the certified radius $\cerradius$ would be large when: $\sigma_t$ is high; $\rho_t$ is small; the margin between $\underline{p_A}$ and $\overline{p_B}$ is large; the number of attackers $R$ is small; the poison ratio $\frac{{{q_B}_i}}{{{n_B}_i}}$ is small; the scale factor $\gamma_i$ is small; the aggregation weights for attackers $p_i$ is small; the local iteration $\tau_i$ is small; and the local learning rate $\eta_i$ small. 2) Since $0 \le 2\Phi(\cdot) -1 \le 1$, the certified radius $\cerradius$ goes to $\infty$ as $ T \rightarrow \infty$ when $\Phi(\cdot)<1$. Ituitively,  the benign fine-tuning after backdoor injection round $\tadv$ would mitegate the poisoning effect. Thus, with infinite rounds of such fine-tuning, the model is able to tolerate backdoors with arbitrarily large magnitude.  In practice, we note that the continued multiplication in the denominator may not approach 0 due to numerical issues, which we will verify in the experiments section.  4) Large number of clients $N$ will decrease the aggregation weights $p_i$ of attackers, thus it can tolerate backdoors with large magnitude, resulting in higher $\cerradius$. 5) For general neural networks, efficient computation of Lipschitz gradient constant (w.r.t. data input) is an open question, especially when the data dimension is high. We will provide a closed-form expression for $ L_{\mathcal Z}$ under some constraints next. 

As mentioned in Section~\ref{subsection:threat_model}, the backdoor for data sample $z_j^i$ includes both the backdoor pattern $\delta_{i_x}$ and adversarial target label flipping $\delta_{i_y}$. In Assumption \ref{assumption:data_lipschitz} we define $ L_{\mathcal Z}$ with  $z=\{x,y\}$ (concatenation of $x$ and $y$) to certify against both backdoor patterns and label-flipping effects. Without loss of generality, here we focus on backdoor patterns considering bounded model parameters in Lemma~\ref{lm:L_z_for_multiclass_logistic_regression}, which provides a closed-form expression for $ L_{\mathcal Z}$ in the case of multi-class logistic regression. By applying $ L_{\mathcal Z}$ from Lemma~\ref{lm:L_z_for_multiclass_logistic_regression} to Theorem~\ref{theorem_robustness_3_level}, it indicates that the prediction for a test sample is independent with the backdoor pattern so the backdoor pattern is disentangled from the adversarial target label. 

\begin{restatable}[]{lemma}{lemmalzforlogreg}
\label{lm:L_z_for_multiclass_logistic_regression}
Given the upper bound on model parameters norm, i.e., $\|w\| \leq \rho$, and two data samples $z_1$ and $z_2$  with  $x_1\neq x_2$ ($y_1=y_2$), for multi-class logistic regression (i.e., one linear layer followed by a softmax function and trained by cross-entropy loss), its Lipschitz gradient constant  w.r.t data is $ L_{\mathcal Z}  = \sqrt{2+2\rho + \rho^2}  $. That is,
\begin{equation}
 \| \nabla \ell(w;z_1) -  \nabla \ell(w;z_2)\| \leq \sqrt{2+2\rho + \rho^2}   \| z_1 - z_2\|. \nonumber
\end{equation}
\end{restatable}

Proof for Lemma~\ref{lm:L_z_for_multiclass_logistic_regression} is provided in the Appendix \ref{sec:proof l_z}. 

In order to formally derive the main theorem, there are two key results. We first quantify the closeness between the FL trained models $\mathcal{M}(D')$ and $\mathcal{M}(D))$ using \markovkernel{}, and then connect the model closeness to the prediction consistency through parameter smoothing.

\subsection{Model Closeness}

As described in Algorithm \ref{algo:parameters_perturbation}, owing to the Gaussian noise perturbation mechanism, in each iteration the global model can be viewed as a random vector with the Gaussian smoothing measure $\mu$.
We use
the \fdivergence{}  between $\mu(\mathcal{M}(D'))$ and $\mu(\mathcal{M}(D))$ as a statistical distance for measuring model closeness of the final FL model. 
Based on the data post-processing inequality, when we interpret each round of \ourframework{} as a probability transition kernel, i.e., a \markovkernel{}, the contraction coefficient of \markovkernel{} can help bound the divergence over multiple training rounds of FL.

Let $f:(0,\infty) \rightarrow \mathbb{R}$ be a convex function with $f(1)=0$, $\mu$ and $\nu$ be two probability distributions. Then the \fdivergence{} is defined as $D_f(\mu || \nu) = E_{W \sim \nu }[f(\frac{\mu(W)}{\nu(W)})].$ 
Common choices of \fdivergence{} include total variation ($f(x)=\frac{1}{2}\|x-1\|$) and Kullback-Leibler (KL) divergence ($f(x)=x \log x$).
The data processing inequality \cite{raginsky2016strong, polyanskiy2015dissipation,polyanskiy2017strong} for the relative entropy states that, for any convex function $f$ and any probability transition kernel (\markovkernel{}), $D_f(\mu K||\nu K) \le D_f(\mu ||\nu)$, where $\mu K$ denotes the push-forward of $\mu$ by $K$, i.e.,  $\mu K = \int \mu(dW) K(W)$.
In other words, $D_f(\mu || \nu)$ decreases by post-processing via $K$.
\cite{asoodeh2020differentially} extend it to analyze SGD.

In our setting,
all the operations in one round of our \ourframework{}, including SGD, clipping and noise perturbations, are incorporated as a \markovkernel{}. We note that in the single-round attack setting, the adversarial clients use clean datasets to train the local models after $\tadv$, so the Markov operator is the same as the one in the benign training process. Therefore the \fdivergence{} of the two global models (backdoored and benign) of interest decreases over rounds, which is characterized by a contraction coefficient defined in Appendix \ref{sec_app_modelclossness}.
We quantify such contraction property of \markovkernel{} for each round with the help of two hyperparameters in the server side: model parameter norm clipping threshold $\rho_t$ and the noise level $\sigma_t$, and finally bound \fdivergence{} of global models in round $T$.
Although our analysis can be adopted to general \fdivergence{}, we here use \kldivergence{} as an instantiation to measure the model closeness. 


\begin{restatable}[]{theorem}{thmdivergenceroundtmain}
\label{therom:divergence_round_T_main}
When $\eta_i \le \frac{1}{\beta}$ and Assumptions~\ref{assumption:Smoothness},~\ref{assumption:data_lipschitz}, and~\ref{assumption:fl_system_train_test} hold, the \kldivergence{} between  $\mu(\mathcal{M}(D))$ and  $\mu(\mathcal{M}(D'))$ with $\mu(w) = \mathcal N(w,{\sigma_T}^2\bf I)$ is bounded as: 
\begin{small}
\begin{equation}
\begin{split}
    &D_{KL}( \mu(\mathcal{M}(D)) || \mu(\mathcal{M}(D')) ) \\
    &\le \frac{ 2R\sum_{i=1}^R\left
    (p_i \gamma_i \tau_i \eta_i \frac{{{q_B}_i}}{{{n_B}_i}}  L_{\mathcal Z} \|\delta_i\|  \right)^2 }{\sigma_{\tadv}^2} \prod_{t=\tadv+1}^{T}  \left(2\Phi \left (\frac{\rho_t }{\sigma_{t}}\right)-1 \right) \nonumber
\end{split}
\end{equation}
\end{small}
\end{restatable}
The proof is provided in the Appendix \ref{sec_app_modelclossness}.

\subsection{Parameter Smoothing}
We connect the model closeness to the prediction consistency by the following theorem. 
The smoothed classifier $h_s$ is robustly certified at $\mu(w')$ with respect to the bounded \kldivergence{}, $D_{KL}(\mu(w),\mu(w'))\le \epsilon$.
\begin{theorem}\label{theorem:param_smoothing}
Let $\smoothclsfier$ be defined as in Eq.~\ref{eq:define_smoothed_cls}. Suppose $c_A \in \mathcal{Y} $ and $\underline{p_A}, \overline{p_B} \in [0,1]$ satisfy
\begin{equation} 
 {H_s^{c_A}}(w';x_{test}) \ge \underline{p_A} \ge \overline{p_B} \ge \max_{c\ne c_A} {H_s^{c}}(w';x_{test}), \nonumber
\end{equation}
then ${\smoothclsfier}(w';x_{test}) = {\smoothclsfier}(w;x_{test}) = c_A$ for all ${w}$  such that $D_{KL}(\mu(w),\mu(w'))\le \epsilon$, where 
\begin{equation} 
\begin{aligned}
   \epsilon = - \log \Big(1-(\sqrt{\underline{p_A}} - \sqrt{\overline{p_B}})^2\Big) \nonumber
\end{aligned}
\end{equation}
\end{theorem}

The proof is provided in the Appendix \ref{sec_app_param_smothing}.

Finally, combining Theorem~\ref{therom:divergence_round_T_main} and ~\ref{theorem:param_smoothing} leads to our main Theorem \ref{theorem_robustness_3_level}. In detail,
Theorem~\ref{therom:divergence_round_T_main} states that $D_{KL}( \mu(\mathcal{M}(D)) || \mu(\mathcal{M}(D')))$ under our \ourframework{} framework is bounded by certain value that depends on the difference between $D$ and $D'$. Theorem~\ref{theorem:param_smoothing} states that for a test sample $x_{test}$, as long as the \kldivergence{} is smaller than $- \log (1-(\sqrt{\underline{p_A}} - \sqrt{\overline{p_B}})^2) $, the prediction from the poisoned smoothed classifier $\smoothclsfier$ that is built upon the base classifier with model parameter $\mathcal{M}(D')$ will be consistent with the prediction from $\smoothclsfier$ that is built upon $\mathcal{M}(D)$. Therefore, we derive the condition for $D$ and $D'$  in Theorem~\ref{theorem_robustness_3_level}, under which  $D_{KL}( \mu(\mathcal{M}(D)) || \mu(\mathcal{M}(D'))) \leq - \log (1-(\sqrt{\underline{p_A}} - \sqrt{\overline{p_B}})^2) $. This condition also indicates that $\smoothclsfier$ built upon the model parameter $\mathcal{M}(D')$ is certifiably robust.

\paragraph{Defend against Other Potential Attack} 
Here we discuss the potentials to generalize our method against other training-time attacks.
1) Our method can naturally extend to \emph{fixed-frequency} attack by applying our analysis for each attack period. In particular, we can repeatedly apply our Theorem~\ref{therom:divergence_round_T_main} to analyze model closeness for each attack period, and the different initializations of each period can be bounded based on its last period. Then Theorem~\ref{theorem:param_smoothing} can be applied to connect model closeness to certify the prediction consistency.
2)
\cite{wang2020attackthetails} introduce edge-case adversarial training samples to enforce the model to misclassify inputs on the tail of input distribution. The edge-case attack essentially conducts a special semantic attack~\cite{bagdasaryan2020backdoor} by selecting rare images instead of directly adding backdoor patterns. It is possible to apply our framework against such attack by viewing it as the whole \emph{sample} manipulation.

\paragraph{Comparison with Differentially Private Federated Learning}
In order to protect the privacy of each client, differentially private federated learning (DPFL) mechanisms are proposed~\cite{geyer2017differentially, mcmahan2018learning, agarwal2018cpsgd} to ensure that the learned FL model is essentially unchanged when one individual client is modified.
Compared with DPFL, our method has several fundamental differences and addresses additional challenges: 1) Mechanisms: DPFL approaches add training-time noise to provide privacy guarantee, while ours add smoothing noise during training and testing to provide certified robustness against data poisoning. In general, the added noise in CRFL does not need to be as large as that in DPFL to provide \textit{strong} privacy guarantee, and therefore preserve higher model utility.
2) Certification goals: DPFL approaches provide client-level privacy guarantee for the learned model parameters, while in CRFL the robustness guarantee is derived for certified pointwise prediction which could be on the feature, samples and clients levels. 3) Technical contributions: DPFL approaches derive DP guarantee via DP composition theorems~\cite{dwork2014algorithmic, abadi2016deep}, while we quantify the global model deviation via Markov Kernel and verify the robustness properties of the smoothed model via parameter smoothing.

\section{Experiments}

In our experiments, the attackers perform the model replacement attack at round $\tadv$ during our \ourframework{} training, and the server performs parameter smoothing on a possibly backdoored FL model at round $T$ to calculate the certified radius $\cerradius$ for each test sample based on Corollary~\ref{corollary_robustness_pixel_level}.
Specifically, we evaluate the effect of the training time noise $\sigma_t$, the attacker's ability which includes the number of attackers $R$, the poison ratio $\frac{q_{B_i}}{n_{B_i}}$ and the scale factor $\gamma_i$, robust aggregation protocol, the number of total clients $N$ and the number of training rounds $T$. Moreover, we evaluate the model closeness empirically to justify Theorem \ref{therom:divergence_round_T_main}.

\subsection{Experiment Setup}
We focus on multi-class logistic regression (one linear layer with softmax function and cross-entropy loss), which is a convex classification problem.
We train the FL system following our \ourframework{} framework with three datasets: Lending Club Loan Data (\loan{})~\citep{loandataset}, \mnist{}~\cite{lecun-mnisthandwrittendigit-2010}, and \emnist{}~\cite{cohen2017emnist}.
 We refer the readers to Appendix \ref{sec:app_exp_details} for more details about the datasets, parameter setups and attack setting.
We train the FL global model until convergence and then use our certification in Algorithm~\ref{alg:certify_parameters_perturbation} for robustness evaluation. 

The metrics of interest are \emph{\cerrate{}} and \emph{\ceracc{}}. Given a test set of size $m$, for $i$-th test sample, the ground truth label is $y_i$, and the output prediction is either $c_i$ with the certified radius $\cerradius_i$ or $c_i=\texttt{ABSTAIN}$ with $\cerradius_i=0$. 
Then we calculate 
\textbf{\cerrate{}} at $r$ as $\frac{1}{m}\sum_{i=1}^m \mathbb{1}\{\cerradius_i \ge r\}$ ,
and \textbf{\ceracc{}} at $r$ as $\frac{1}{m} \sum_{i=1}^m\mathbb{1}\{c_i=y_i$ and $\cerradius_i  \ge r\}$.
The \cerrate{} is the fraction of the test set that can be certified at radius $\cerradius \ge r $, which reveals how consistent the possibly backdoored classifier's prediction with the clean classifier's prediction. 
The \ceracc{} is the fraction of the test set for which the possibly backdoored classifier makes correct and consistent predictions with the clean model.
In the displayed figures, there is a critical radius beyond which the \ceracc{} and \cerrate{} are dropped to zero. 
Since each test sample has its own calculated certified radius $\cerradius_i$, this critical value is a threshold that none of them have a larger radius than it, similar to the findings in \cite{cohen2019certifiedrandsmoothing}.
We certified 10000/5000/10000 samples from the \loan{}/\mnist{}/\emnist{} test sets. In all experiments, unless otherwise stated, we use $\sigma_T=0.01$ to generate $M=1000$ noisy models in parameter smoothing procedure, and use the error tolerance $\alpha=0.001$. 
In our experiments, we adopt the expression of $L_{\mathcal Z}$ in Lemma~\ref{lm:L_z_for_multiclass_logistic_regression}. $L_{\mathcal Z}$ can be generalized to other poisoning settings by specifying $z_1, z_2$ in Assumption~\ref{assumption:data_lipschitz} under the case of ``$x_1\neq x_2$ and $y_1 \neq y_2$'' or ``$x_1 = x_2$ and $y_1 \neq y_2$''.


\subsection{Experiment Results}
We only change one factor in each experiment and keep others the same as the experiment setup. We plot the \ceracc{} and \cerrate{} on the clean test set, and report the results on the  backdoored test set in Appendix~\ref{sec:app_exp_details}.

\begin{figure}[ht]
\setlength{\belowcaptionskip}{-4mm}
\subfigure[Certified rate on \mnist{}] 
{
	\includegraphics[scale=0.24]{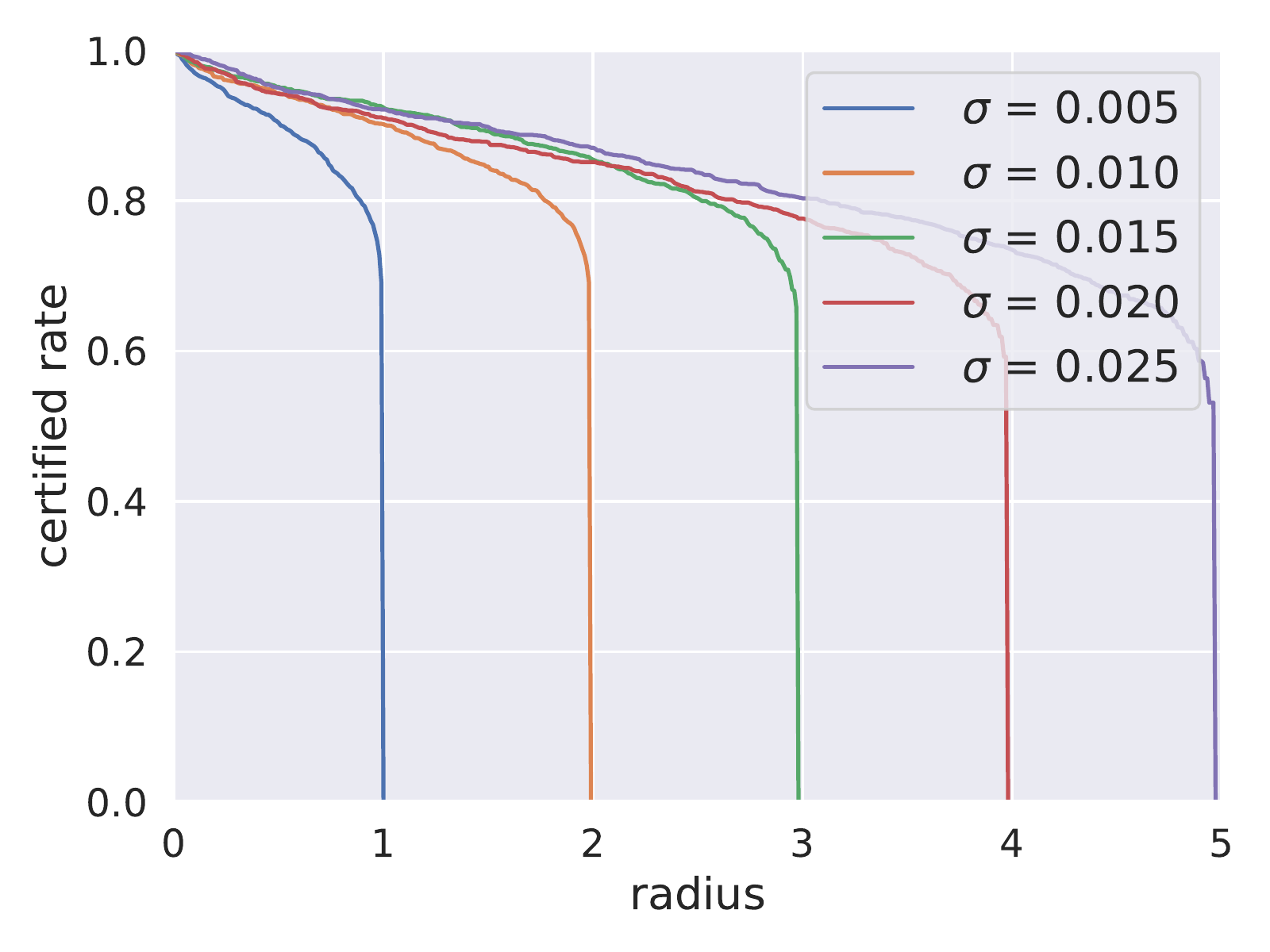}
}
\subfigure[Certified acc. on \mnist{}] 
{
	\includegraphics[scale=0.24]{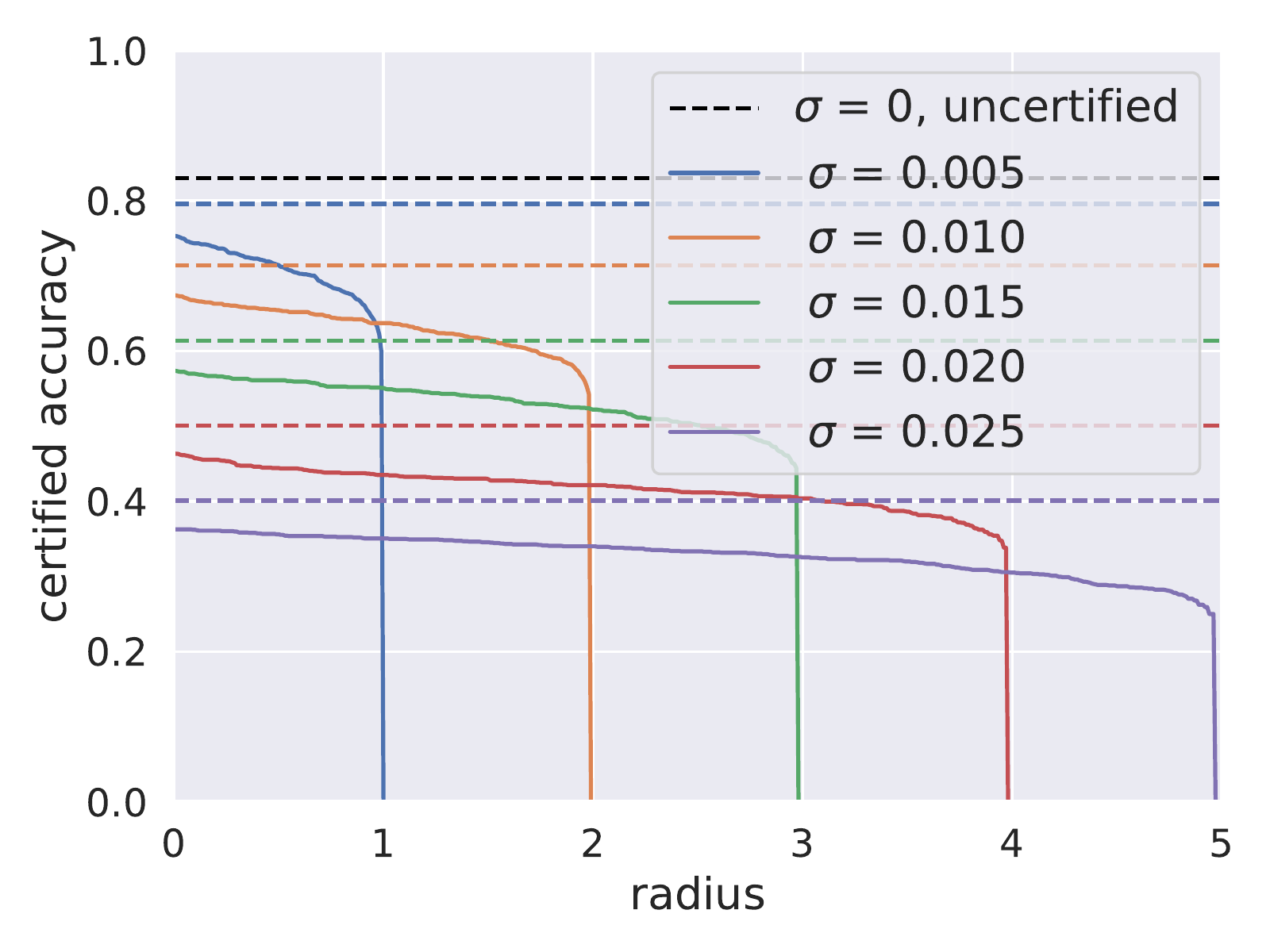}
}
\subfigure[Certified acc. on \loan{}] 
{
	\includegraphics[scale=0.24]{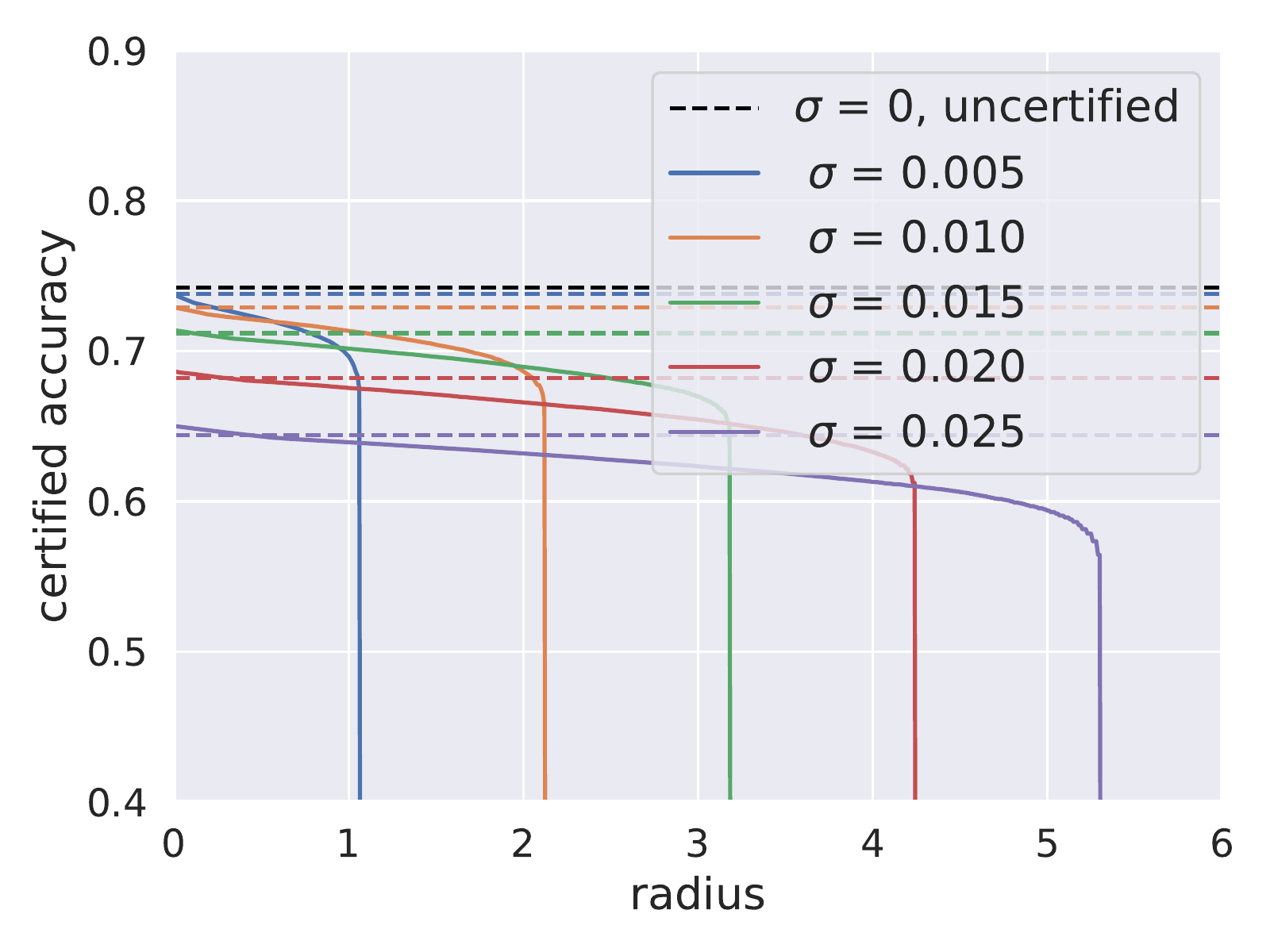}
}
\subfigure[Certified acc. on \emnist{}] 
{
	\includegraphics[scale=0.24]{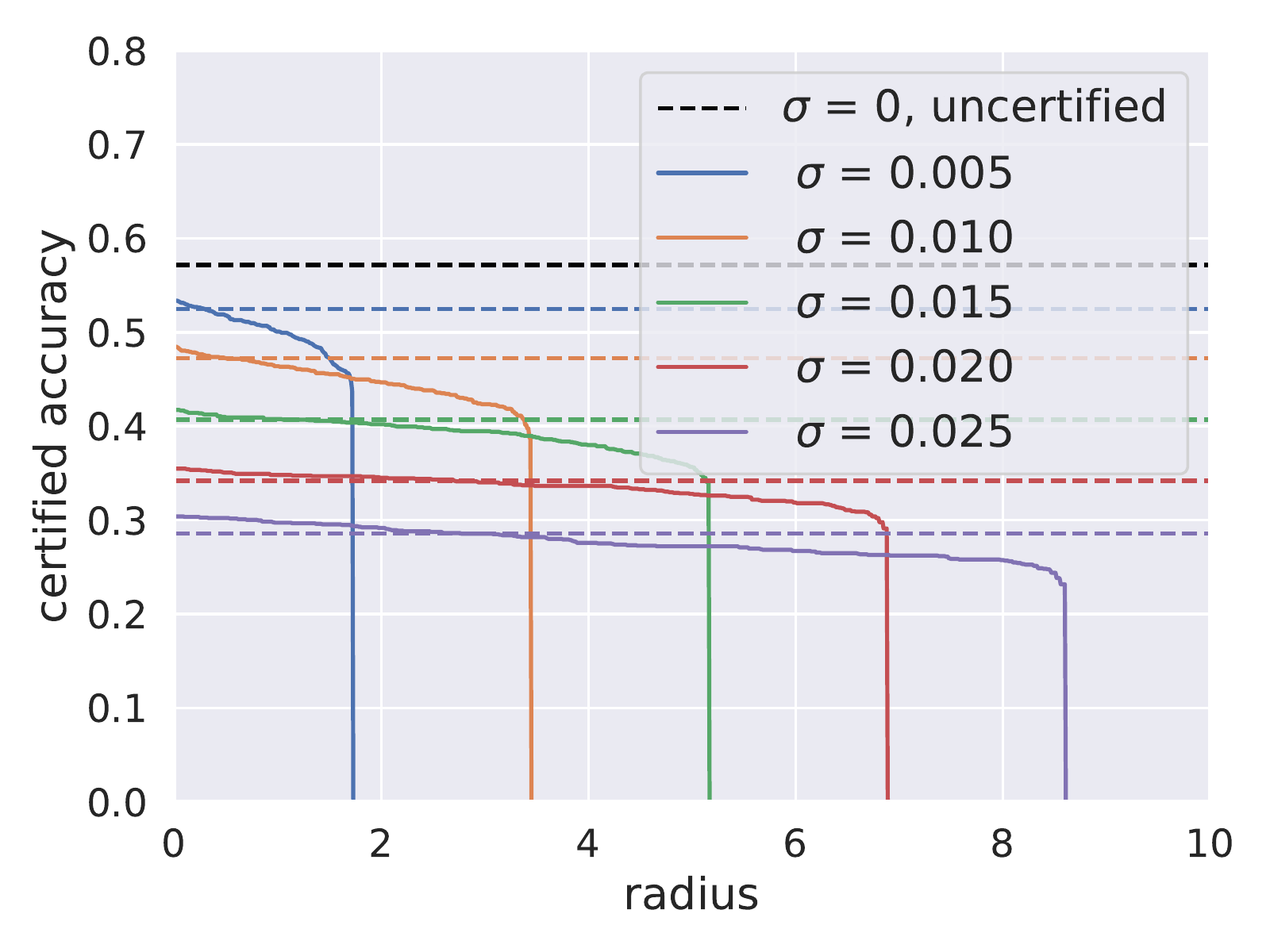}
}
\caption{Certified accuracy and certified rate on \mnist{}, \loan{}, and \emnist{} with different training-time noise $\sigma$. Solid lines represent certified accuracy; dashed lines of the same color show the accuracy of base classifier trained with $\sigma$; black dashed line presents the accuracy of the classifier trained without noise.} 
\label{fig:training_sigma_noise} 
\end{figure}

\paragraph{Effect of Training Time Noise}
Since we aim to defend against backdoor attack, the training time noise $\sigma$ ($\sigma=\sigma_t, t<T$) in our Algorithm~\ref{algo:parameters_perturbation} is more essential than $\sigma_T$ in parameter smoothing (Algorithm~\ref{alg:certify_parameters_perturbation}). The reason is that $\sigma$ can nullify the malicious model updates at early stage.
Figure~\ref{fig:training_sigma_noise} plots the \ceracc{} and \cerrate{} attained by training FL system with different $\sigma$.
In Figure~\ref{fig:training_sigma_noise}(a), when $\sigma$ is high, \cerrate{} is high at every $r$ and large radius can be certified.  Figure~\ref{fig:training_sigma_noise}(b)(c)(d) show that large radius is certified but at a low accuracy, so the parameter noise $\sigma$ controls the trade-off between certifiability and accuracy, which echoes the property of evasion-attack certification~\cite{cohen2019certifiedrandsmoothing}. 
Comparing the solid line with the dashed line for each color, we can see that the parameter smoothing with $\sigma_T$ does not hurt the accuracy much. 

\paragraph{Effect of Attacker Ability}
From the perspective of attackers, the larger number of attackers $R$, the larger poison ratio $\frac{q_{B_i}}{n_{B_i}}$ and the larger scale factor $\gamma_i$ result in the stronger attack. Figure~\ref{fig:poison_ratio}, Figure~\ref{fig:gamma}, and Figure~\ref{fig:number_of_attackers} show that in the three datasets, the stronger the attack, the smaller radius can be certified. After training sufficient number of rounds with clean datasets after $\tadv$, we show that the certified radius is not sensitive to the attack timing $\tadv$ in Appendix~\ref{ap:more_results_clean_testset}.

\begin{figure}[ht]
\setlength{\belowcaptionskip}{-4mm}
\subfigure
{
	\includegraphics[scale=0.24]{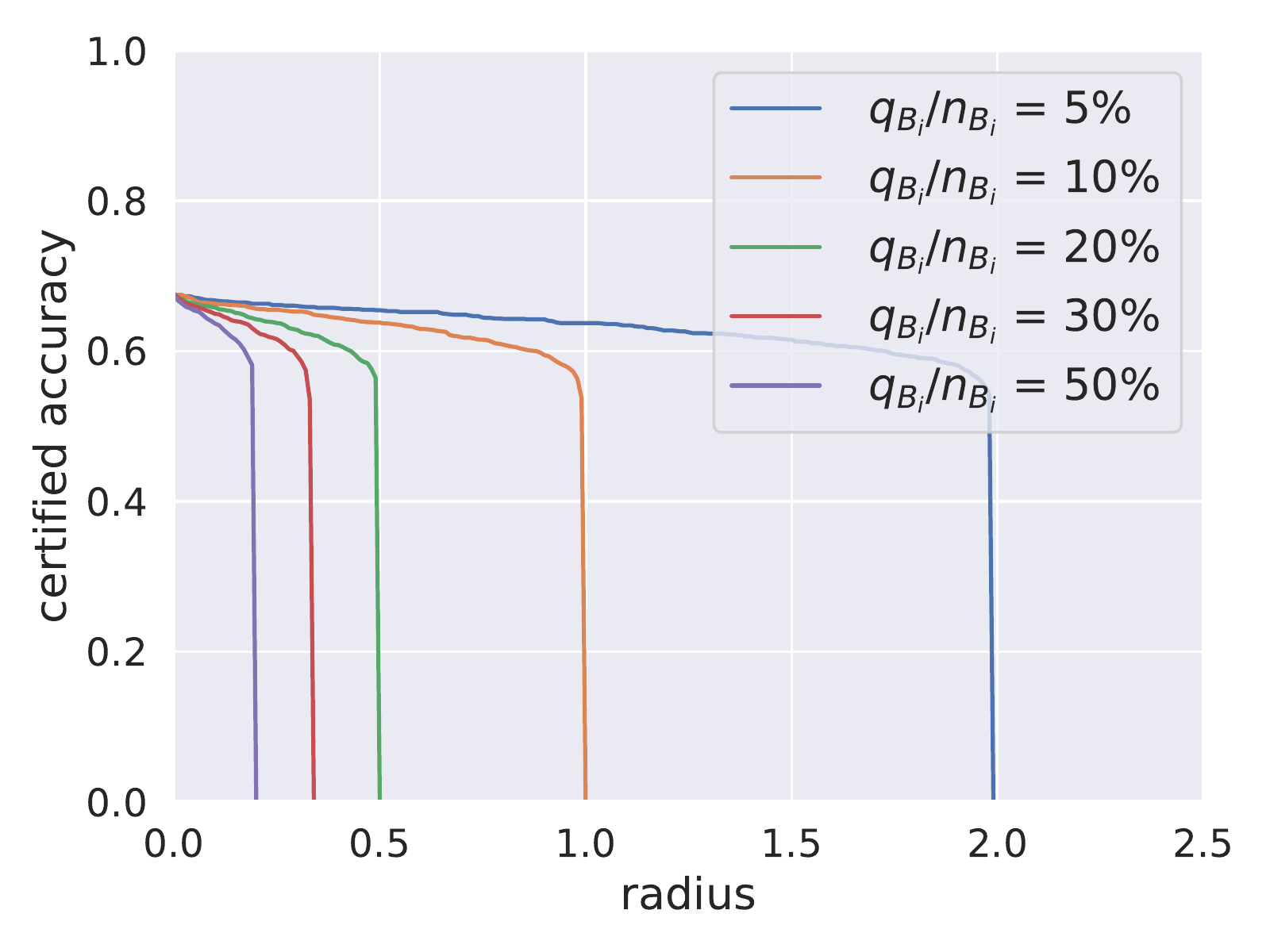}
}
\subfigure
{
	\includegraphics[scale=0.24]{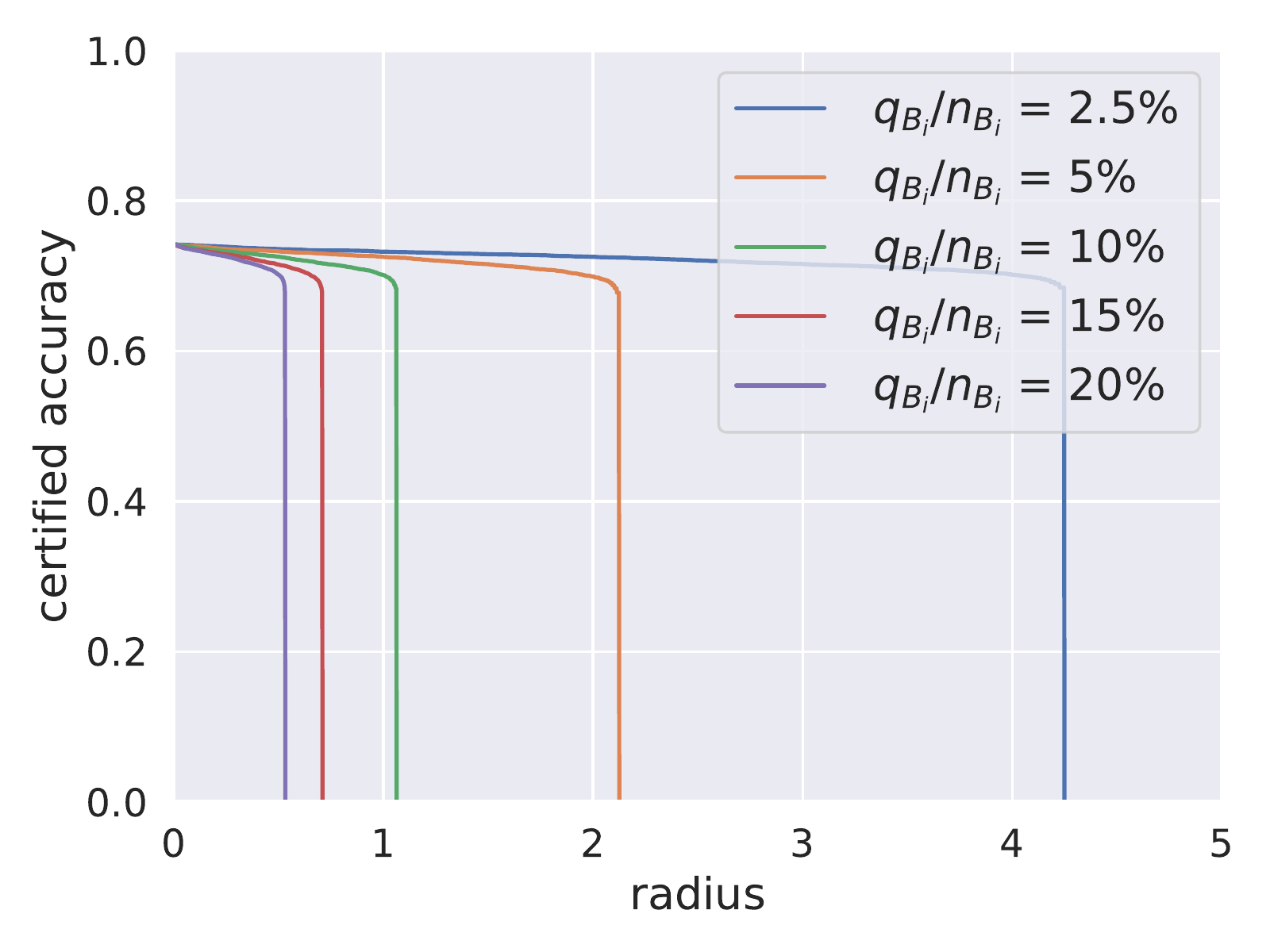}
}
\vspace{-5mm}
\caption{\mnist{}~(left) and \loan{}~(right) test set certified accuracy as the poison ratio  ${q_{B_i}}/{n_{B_i}}$ is varied.} 
\label{fig:poison_ratio} 
\end{figure}

\begin{figure}[ht]
\setlength{\belowcaptionskip}{-4mm}
\subfigure
{
	\includegraphics[scale=0.24]{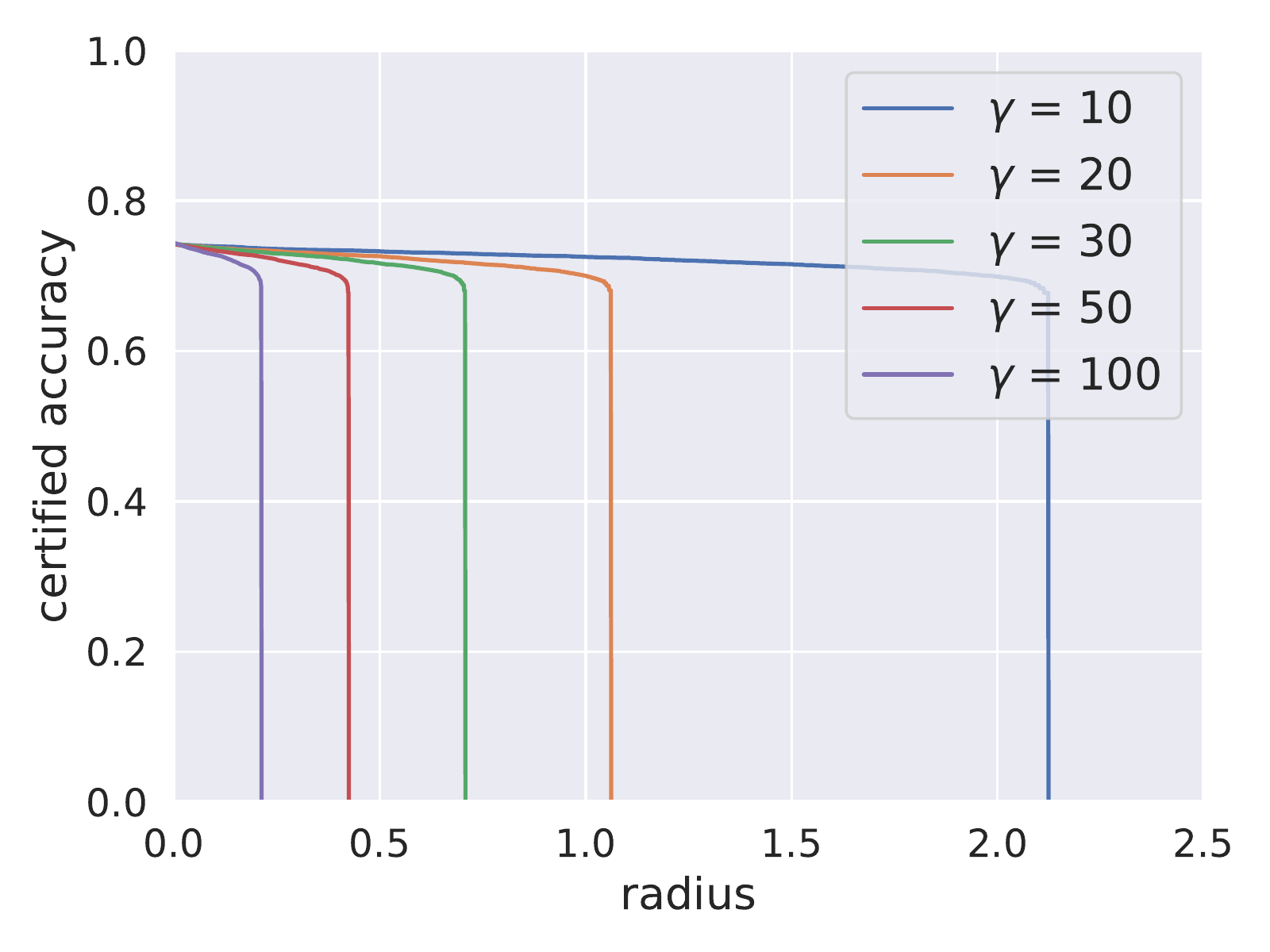}
}
\subfigure
{
	\includegraphics[scale=0.24]{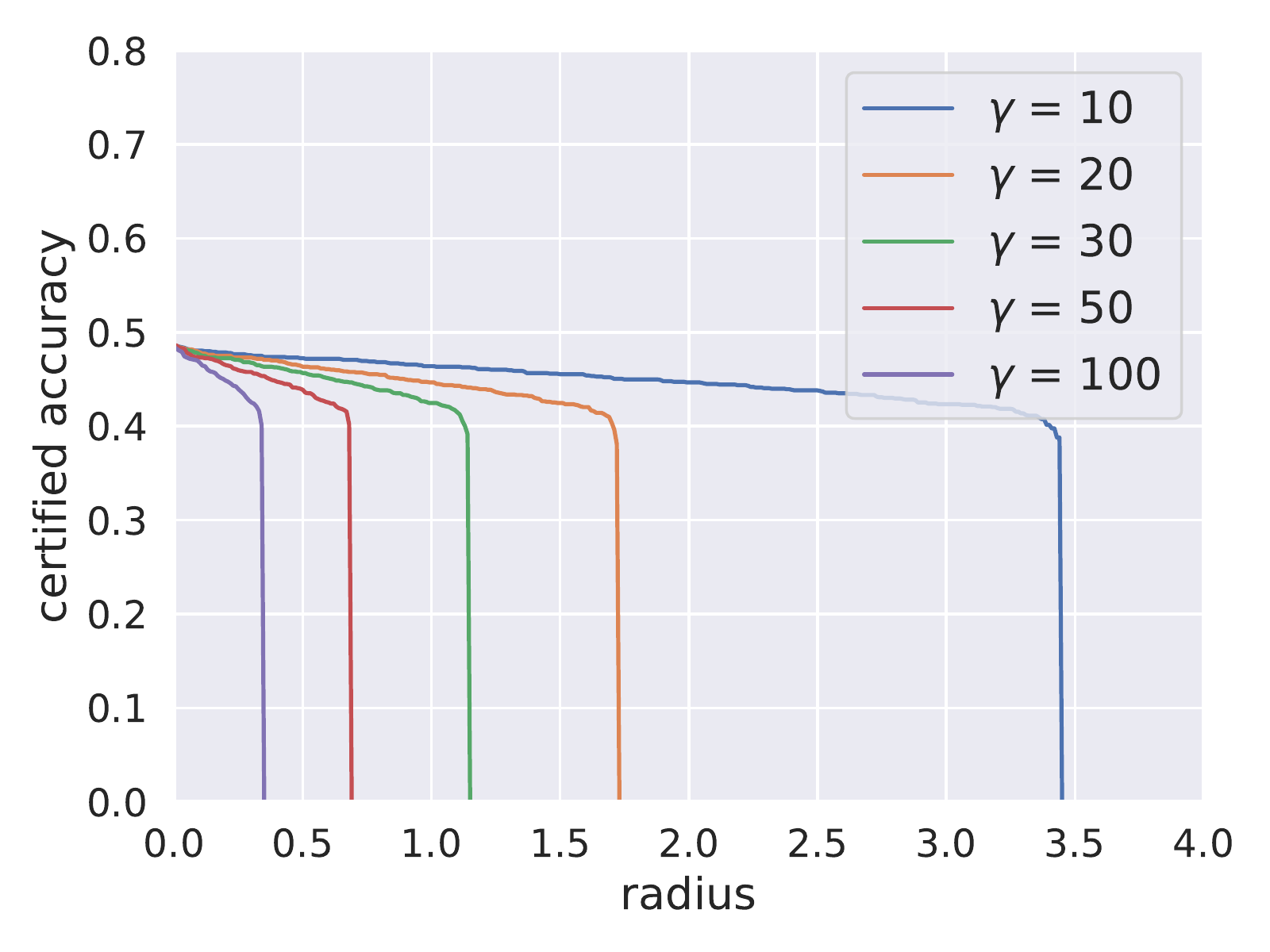}
}
\vspace{-5mm}
\caption{Certified accuracy with different scaling factor $\gamma$ on \loan{}~(left) and \emnist{}~(right).} 
\label{fig:gamma} 
\end{figure}

\begin{figure}[ht]
\setlength{\belowcaptionskip}{-4mm}
\subfigure
{
	\includegraphics[scale=0.24]{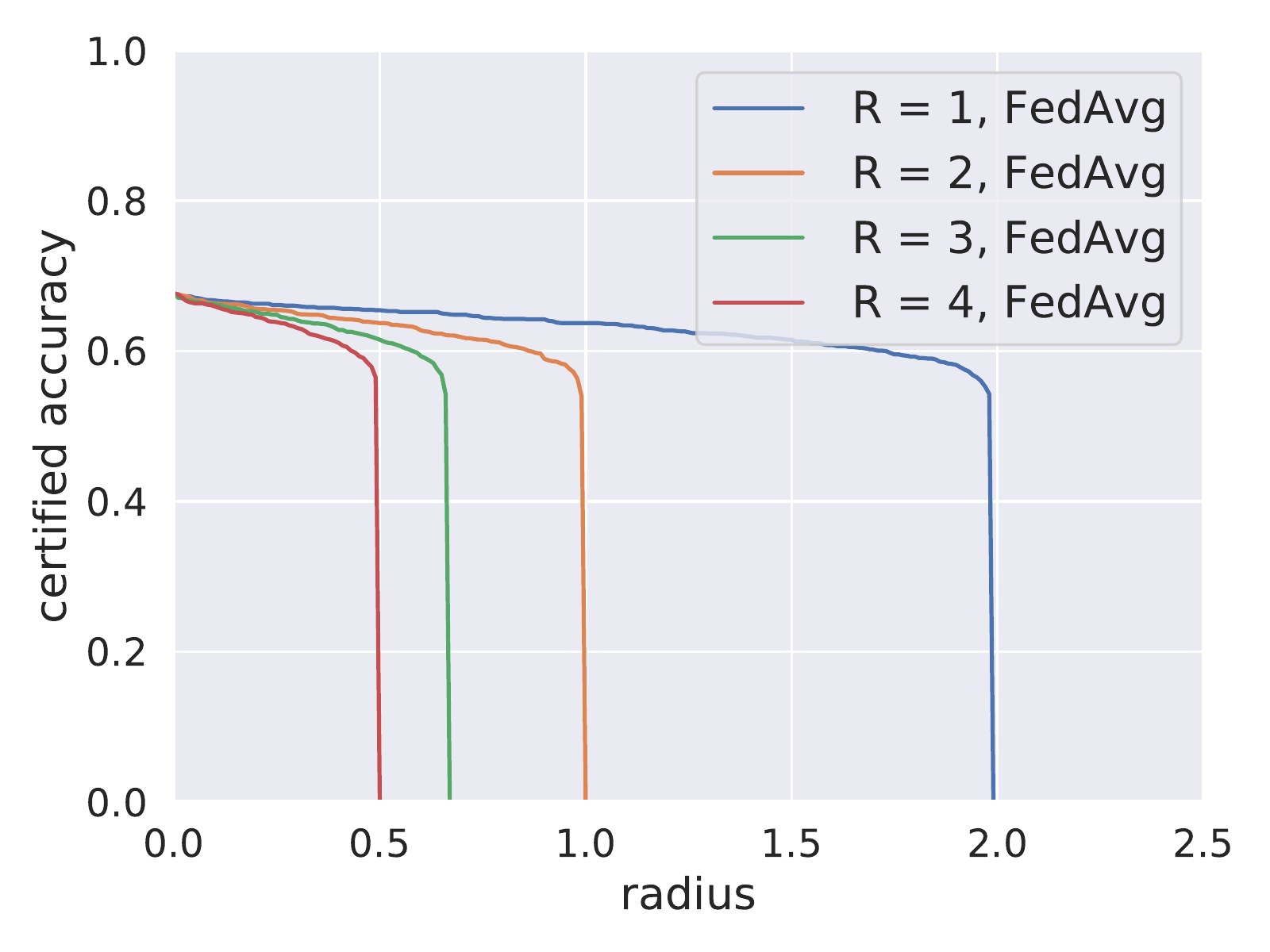}
}
\subfigure
{
	\includegraphics[scale=0.24]{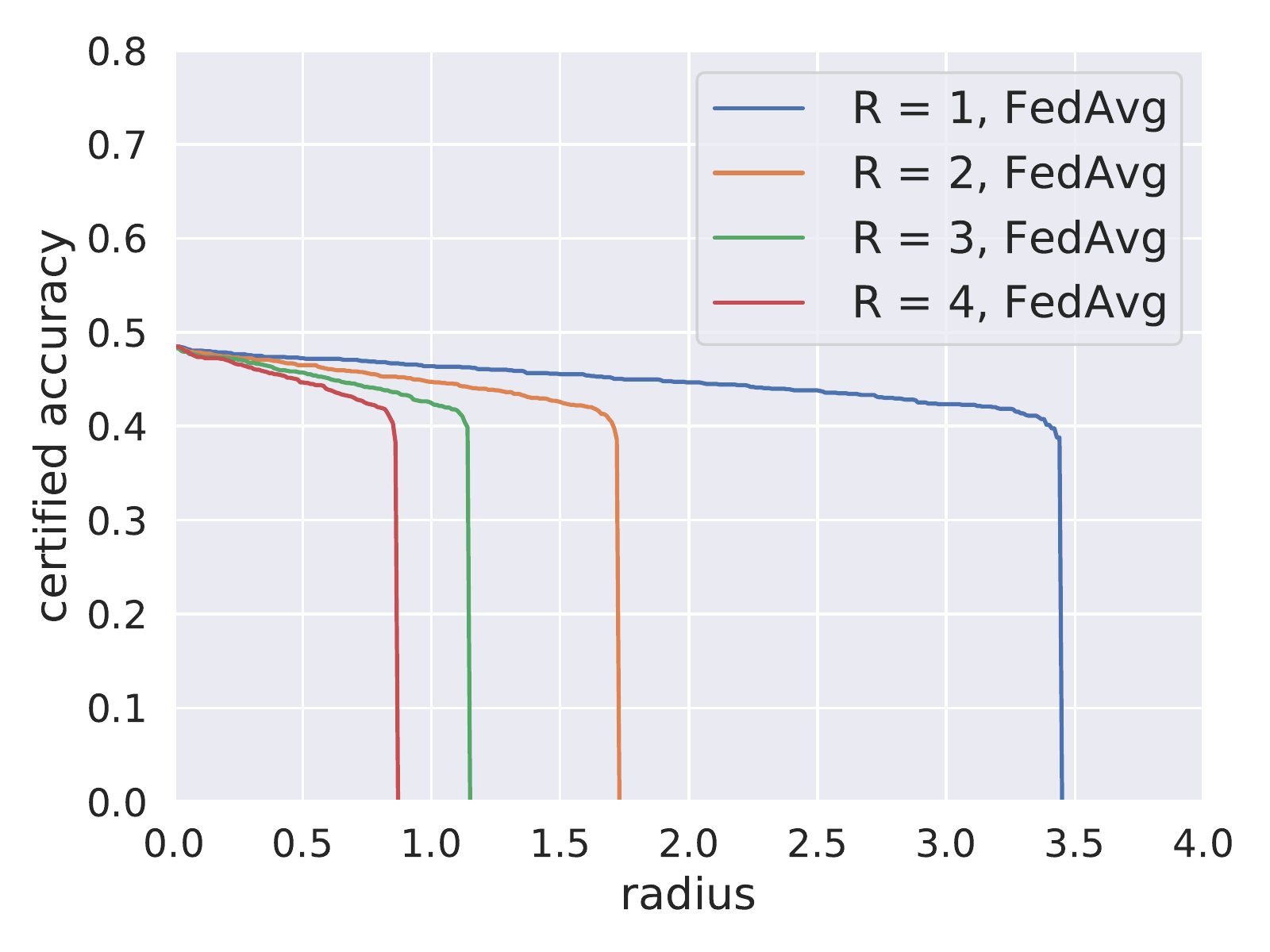}
}
\vspace{-5mm}
\caption{\mnist{}~(left) and \emnist{}~(right) test set certified accuracy as the number of adversarial clients $R$ is varied.} 
\label{fig:number_of_attackers} 
\end{figure}

\begin{figure}[ht]
\setlength{\belowcaptionskip}{-5mm}
\subfigure
{
	\includegraphics[scale=0.24]{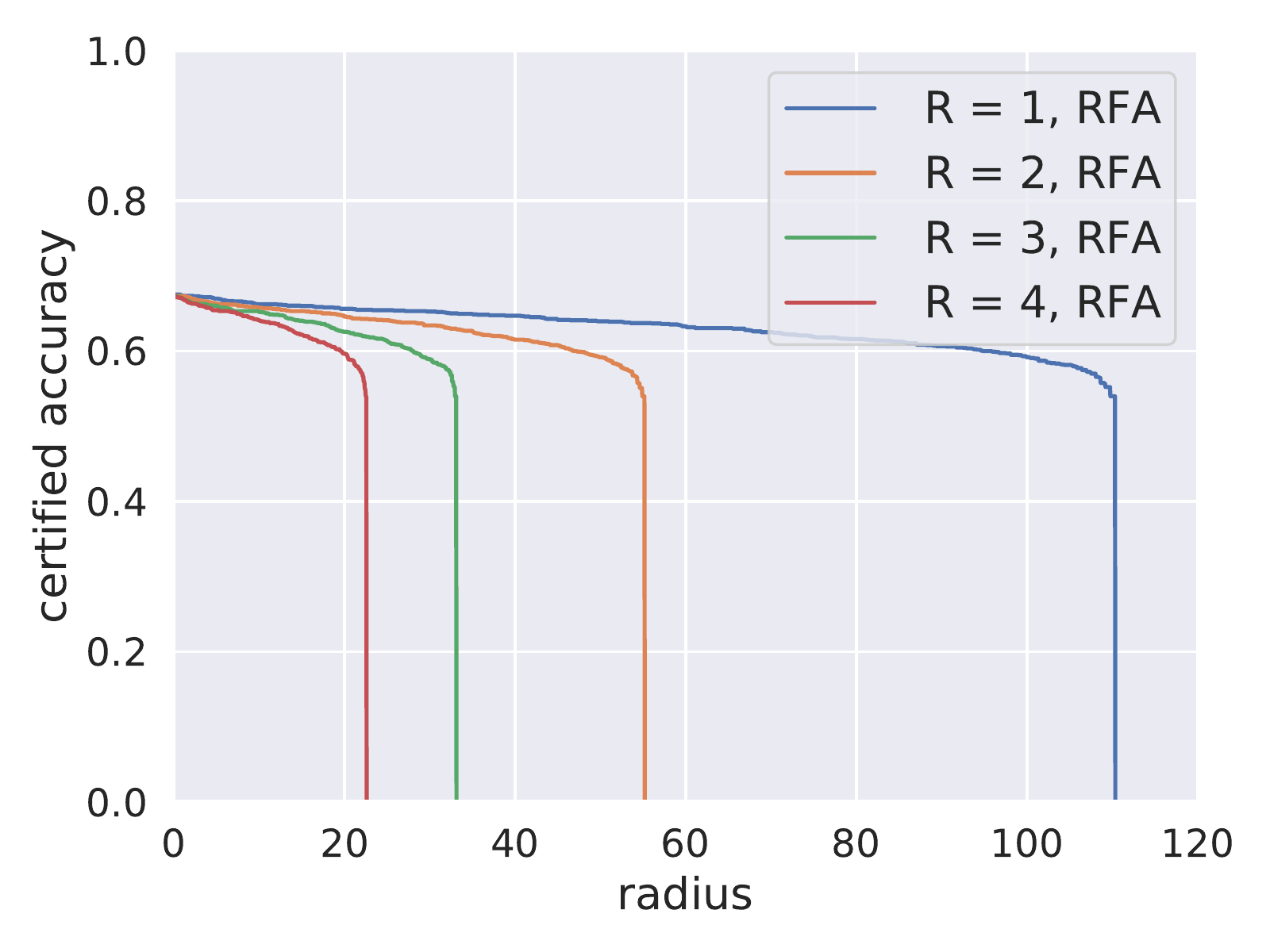}
}
\subfigure
{
	\includegraphics[scale=0.24]{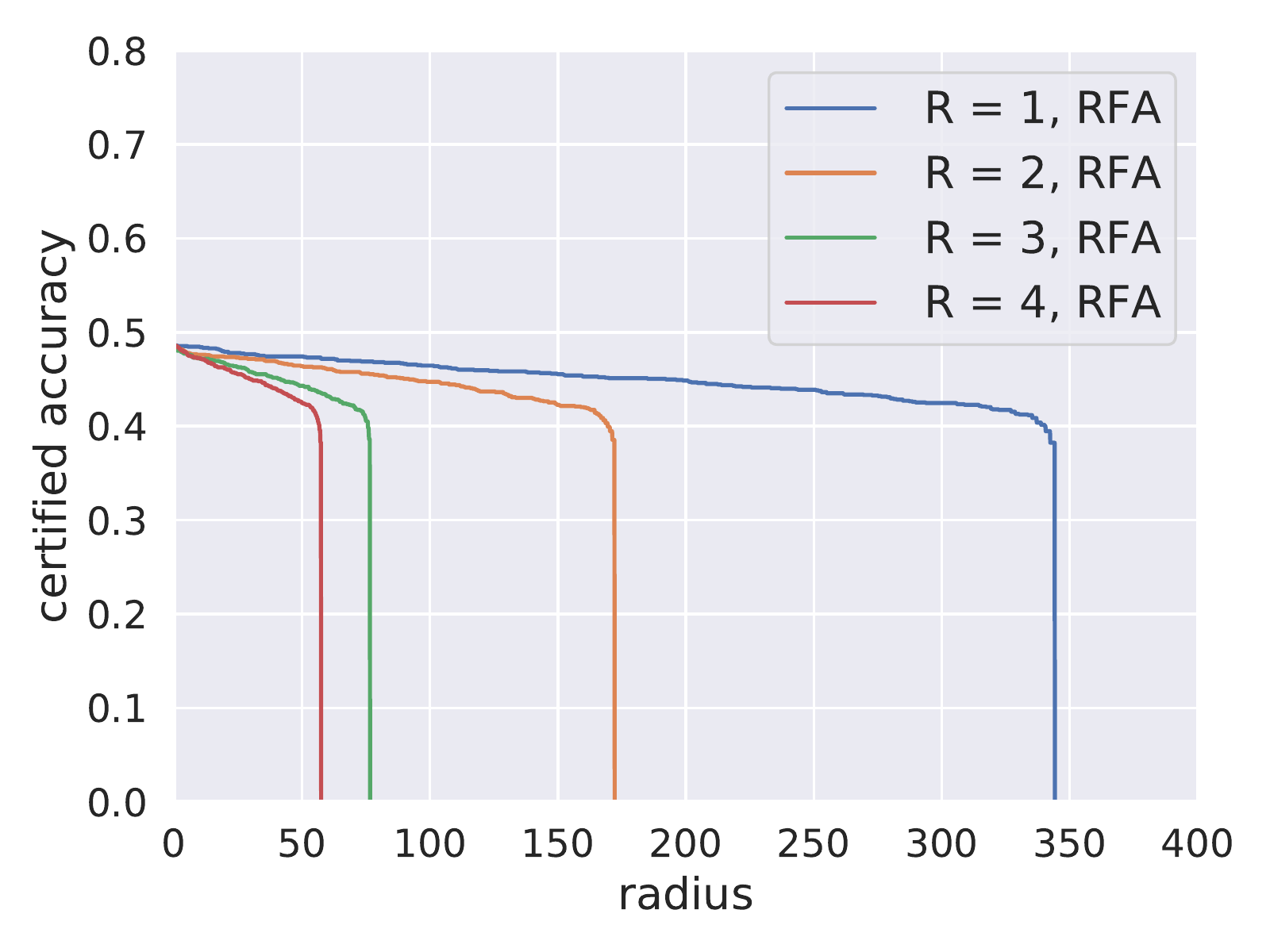}
}
\vspace{-5mm}
\caption{Certified accuracy on \mnist{}~(left) and \emnist{}~(right) with different $R$ when FL is trained under the robust aggregation RFA \cite{pillutla2019robustrfa}. }
\label{fig:robustagg} 
\end{figure}

\paragraph{Effect of Robust Aggregation}
Our \ourframework{} can be used to assess different robust aggregation rules. Figure~\ref{fig:robustagg} presents the \ceracc{} on \mnist{} and \emnist{} as $R$ is varied, when our \ourframework{} adopts the robust aggregation algorithm RFA~\cite{pillutla2019robustrfa}, which detects outliers and down-weights the malicious updates during aggregation.
Comparing FedAvg in Figure~\ref{fig:number_of_attackers} with RFA in Figure~\ref{fig:robustagg} (the magnitude of x-axis is different), we observe that very large radius can be certified under RFA. This is because that the attacker is assigned with very low aggregation weights $p_i$, which is part of our bound in Eq.~\ref{eq:certified_radius_R}. Our certified radius reveals that RFA is much robust than FedAvg, which shows the potential usage of our certified radius as an evaluation metric for the robustness of other robust aggregation rules.

\begin{figure}[ht]
\setlength{\belowcaptionskip}{-4mm}
\subfigure
{
	\includegraphics[scale=0.24]{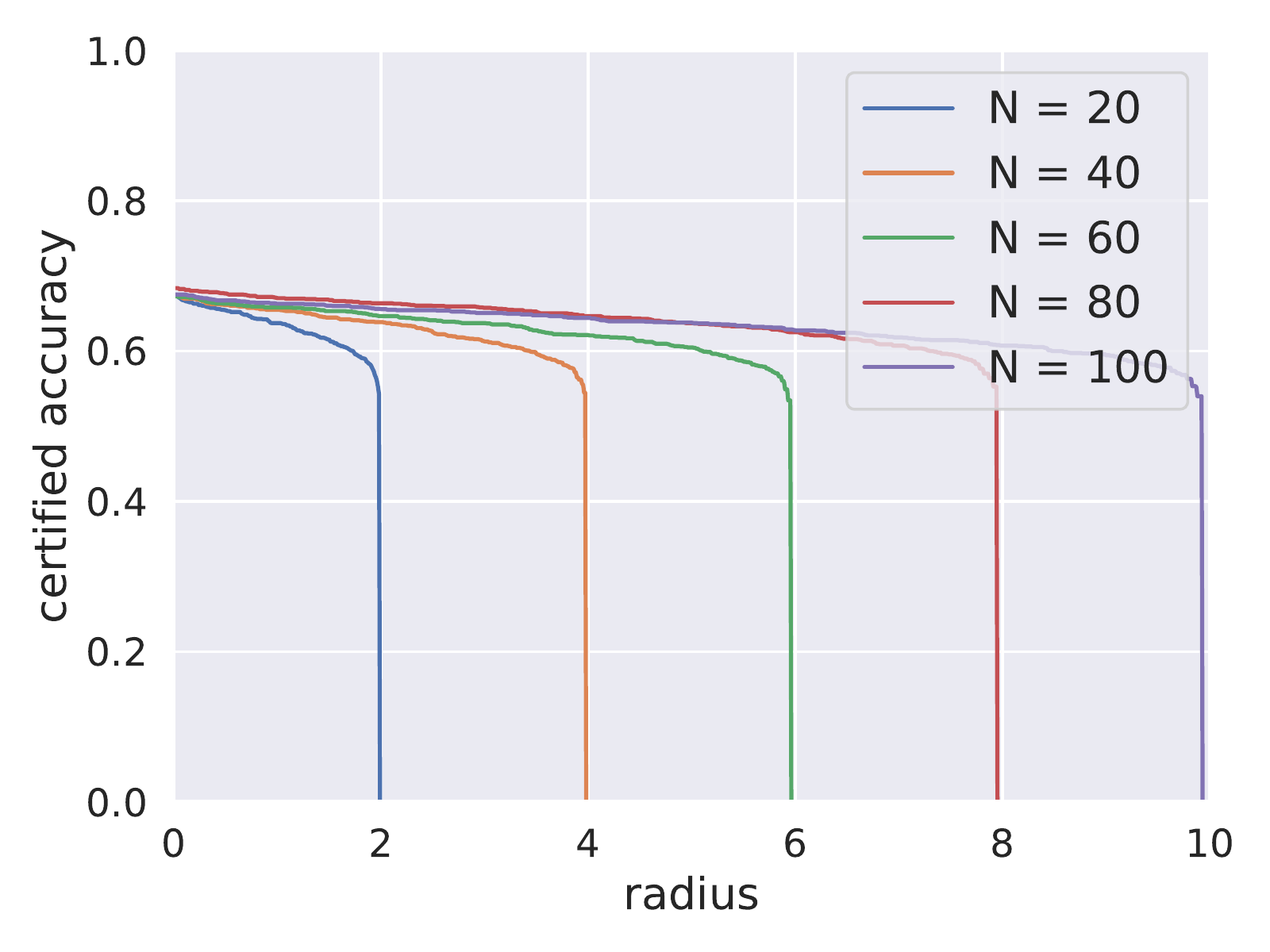}
}
\subfigure
{
	\includegraphics[scale=0.24]{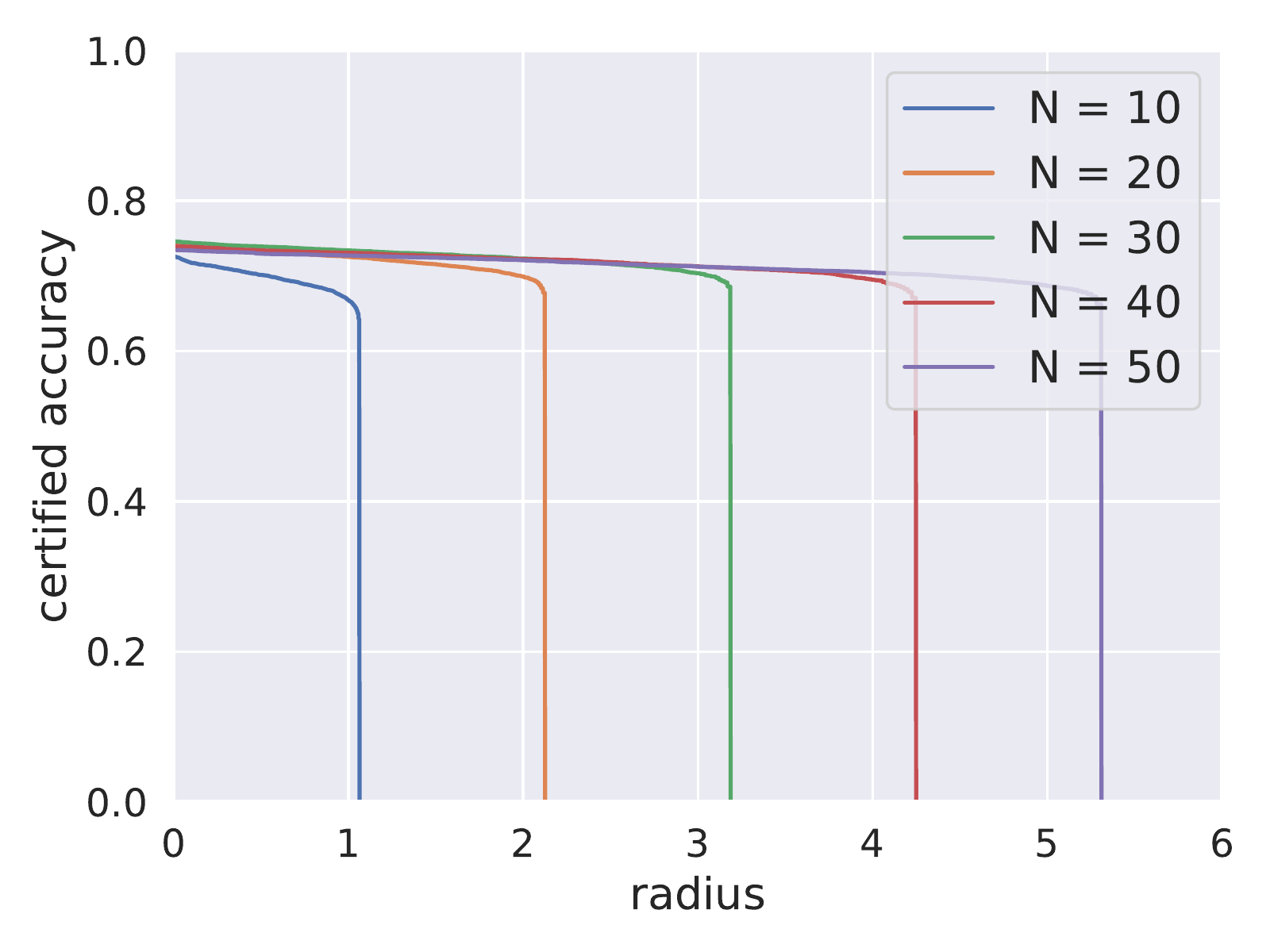}
}
\vspace{-5mm}
\caption{Certified accuracy on \mnist{}~(left) and \loan{}~(right) as the number of total clients $N$ is varied.} 
\label{fig:total_clients} 
\end{figure}

\paragraph{Effect of Client Number}
Distributed learning across a large number of clients is an important property of FL. Figure~\ref{fig:total_clients} shows that large radius can be certified when $N$ is large (i.e., more clients can tolerant larger backdoor magnitude), because it decreases the aggregation weights $p_i$ of attackers. Moreover, the backdoor effect could be mitigated by more benign model updates during training.

\begin{figure}[ht]
\setlength{\belowcaptionskip}{-4mm}
\subfigure
{
	\includegraphics[scale=0.24]{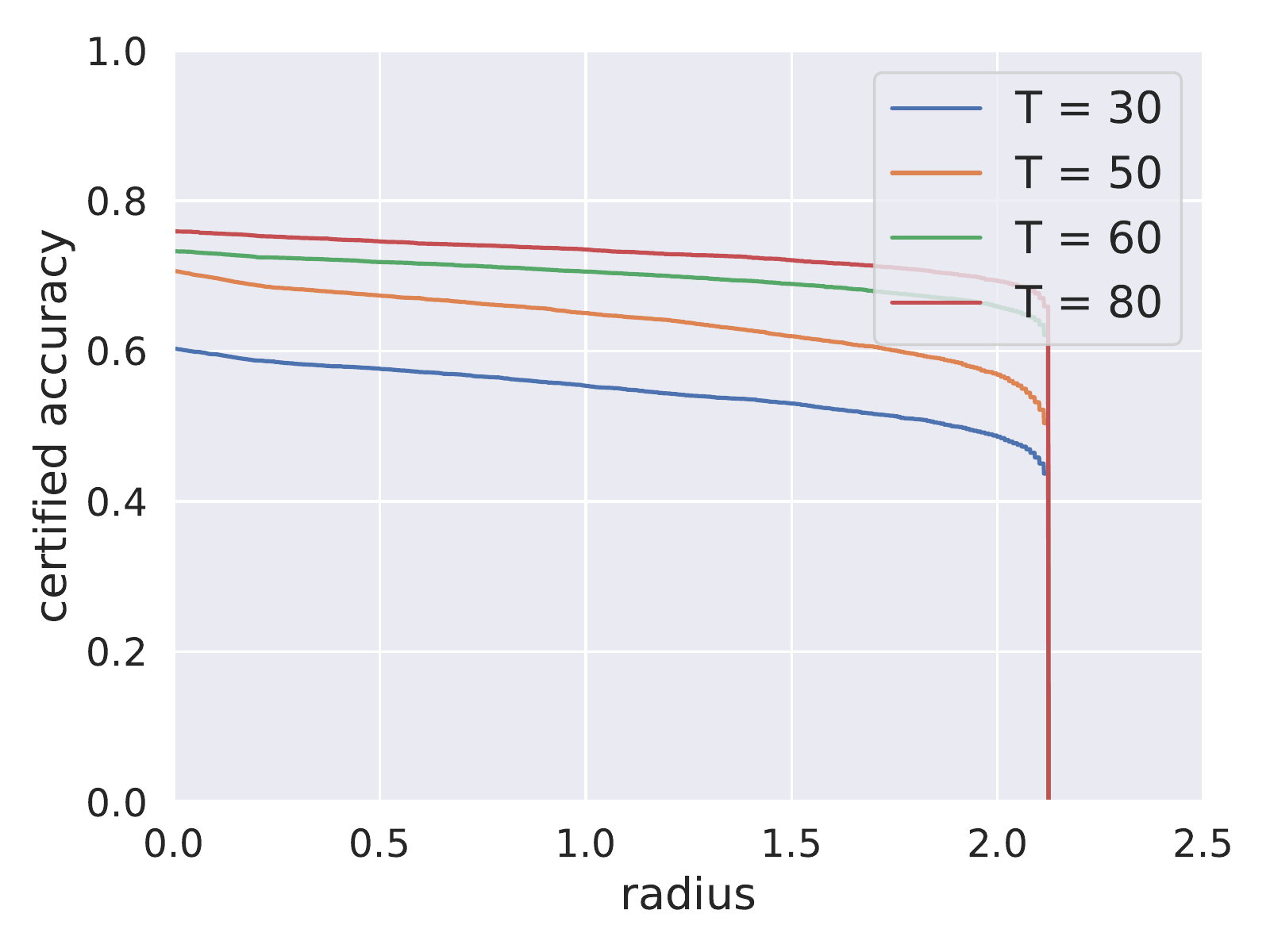}
	
}
\subfigure
{
	\includegraphics[scale=0.24]{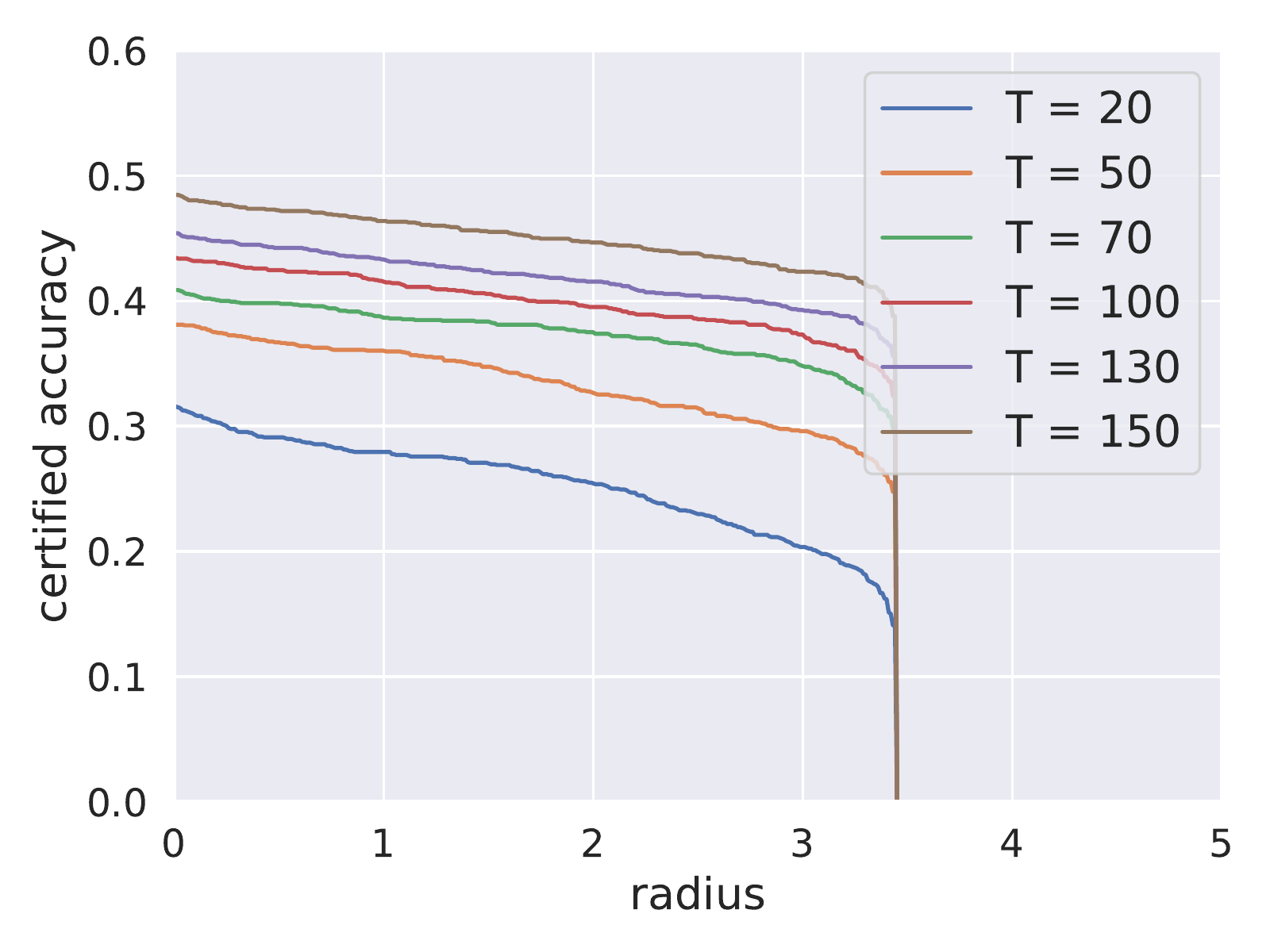}
	
}
\vspace{-5mm}
\caption{\loan{}~(left) and \emnist{}~(right) test set certified accuracy as $T$ is varied.} 
\label{fig:cer_acc_T} 
\end{figure}

\begin{figure}[t!]
\setlength{\belowcaptionskip}{-4mm}
\subfigure
{
	\includegraphics[scale=0.24]{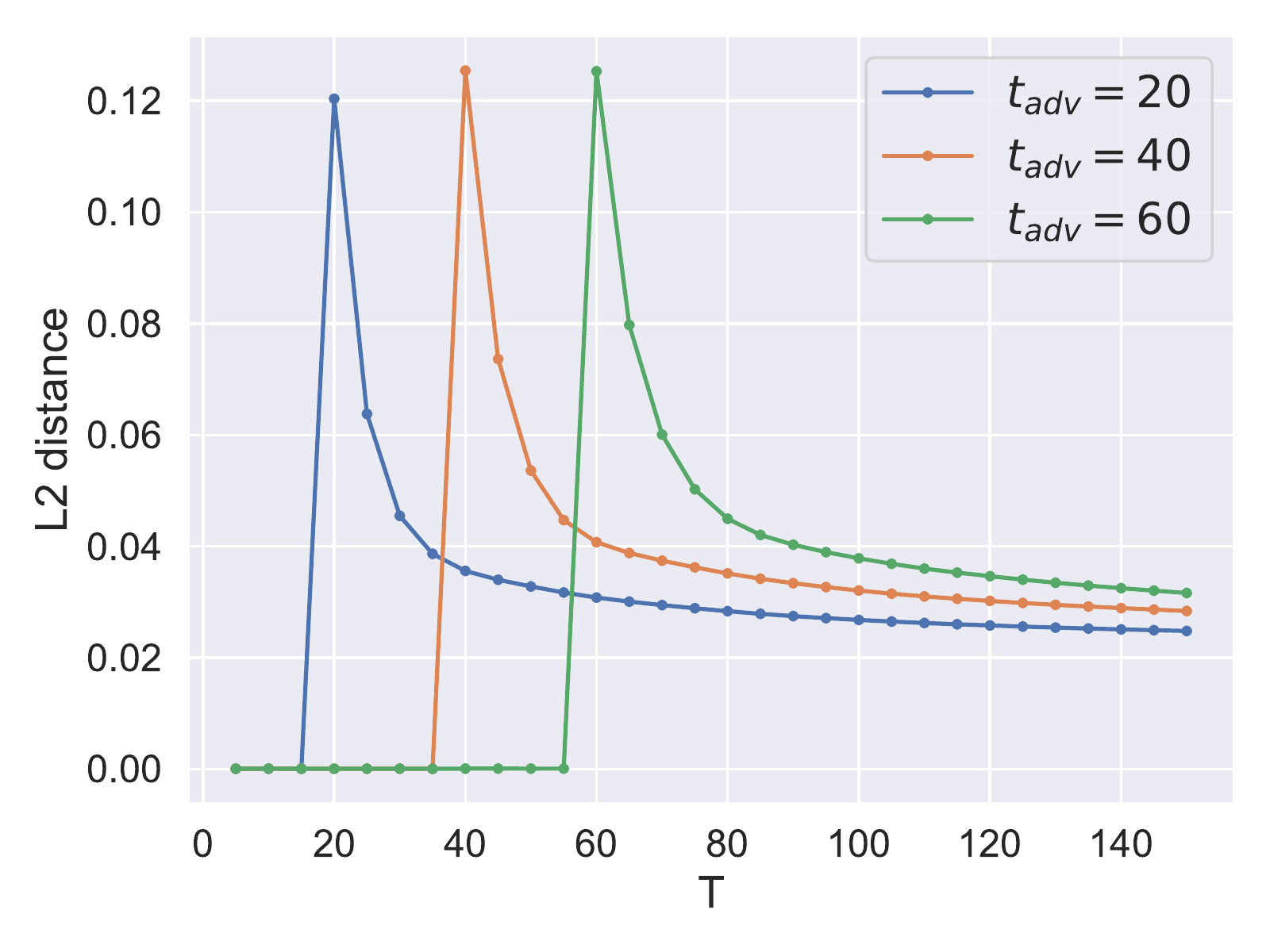}
	
}
\subfigure
{
	\includegraphics[scale=0.24]{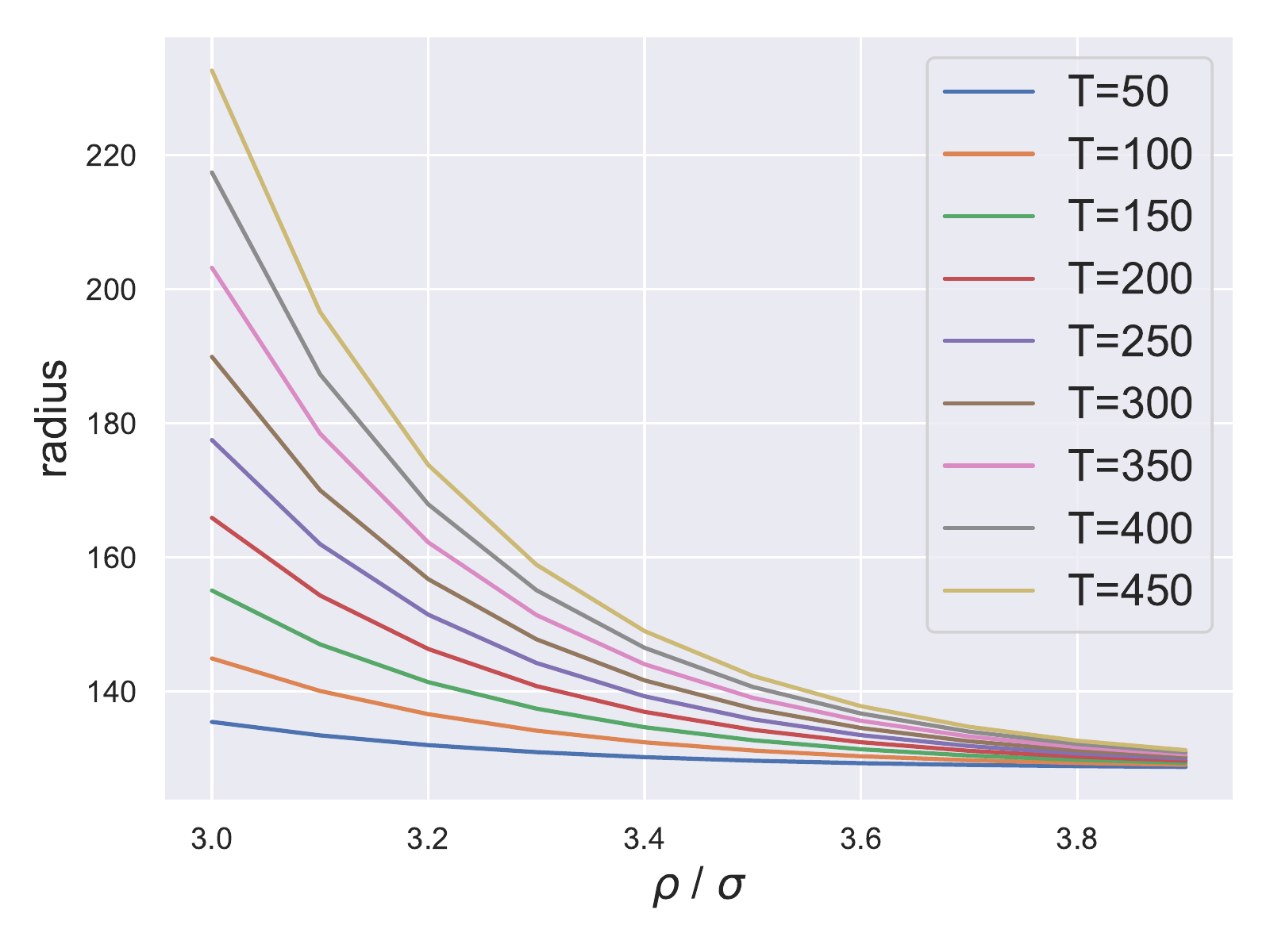}
}
\vspace{-5mm}
\caption{(a) The $\ell_2$ distance of the global models between the backdoored training process and the benign training process. (b) Numerical analysis of the standard Gaussian CDF $\Phi(\cdot)$.} 
\label{fig:global_l2distance_numerical_issue} 
\end{figure}

\paragraph{Effect of Training Rounds}
According to Figure~\ref{fig:cer_acc_T}, the \ceracc{} is higher when $T$ is larger. 
However, the largest radius that can be certified for the test set does not increase. We note that this is due to numerical issues of the standard Gaussian CDF $\Phi(\cdot)$. 
As we mentioned in Section \ref{sec:main_results}, the continued multiplication in the denominator of Eq.~\ref{eq:certified_radius_R} will not achieve 0 in practice. Otherwise the certified radius $\cerradius$ goes to $\infty$ as $ T \rightarrow \infty$ since $0 \le 2\Phi(\rho/\sigma) -1 \le 1$.

To verify our argument, we fix $\underline{p_A} $ and  $ \overline{p_B}$ to be 0.7 and 0.1, use default values for other parameters, and study the relationship between $\rho/\sigma$, $T$ and $\cerradius$ in  Figure~\ref{fig:global_l2distance_numerical_issue}(b). When $\rho/\sigma$ is larger than certain threshold, the certified radius $\cerradius$ does not change much when $T$ increases. If one wishes to increase $T$ for improving certified radius, then we suggest to keep $\rho/\sigma$ smaller than the threshold to make effect.
The increased \ceracc{} when $T$ is large in Figure~\ref{fig:cer_acc_T} could be attributed to improved model performance up to convergence, so the margin between $\underline{p_A} - \overline{p_B}$ is widened. 

We also study the error tolerance $\alpha$ and the number of noisy models $M$ in Appendix \ref{ap:more_results_clean_testset}. Larger $M$ yields larger certified radius, and the certified radius is not very sensitive to $\alpha$.

\paragraph{Empirical Evaluation on Model Closeness}
Our theorems are derived based on the analysis in comparison to a ``virtual" benign training process.  Empirically, we train such FL global model under the benign training process and compare the $\ell_2$ distance between the clean global model and the backdoored global model at every round. In Figure~\ref{fig:global_l2distance_numerical_issue}(a), one attacker performs model replacement attack on \mnist{} at round $\tadv$= $\{20,40,60\}$ respectively. We can observe that the plotted $\ell_2$ distance over the FL training rounds after $\tadv$ is decreasing, which echos our assumption that because all clients behave normal and use their clean local datasets to purify the global model after $\tadv$, the global models between two training process become close.  This observation also can justify the model closeness statement in Theorem \ref{therom:divergence_round_T_main}.


\section{Conclusion}
This paper establishes the first framework (CRFL) on certifiably robust federated learning against backdoor attacks. CRFL employs model parameter clipping and perturbing during training, and uses model parameter smoothing during testing, to certify conditions under which a backdoored model will give consistent predictions with an oracle clean model.
Our theoretical analysis characterizes the relation between certified robustness and federated learning parameters, which are empirically verified on three different datasets.



\section*{Acknowledgements}

This work is partially supported by NSF grant No.1910100, NSF CNS 20-46726 CAR, Amazon Research Award, IBM-ILLINOIS Center for Cognitive Computing Systems Research (C3SR) – a research collaboration as part of the IBM AI Horizons Network.





\bibliographystyle{icml2021}


\onecolumn 
\appendix

\section*{Appendix}
The Appendix is organized as follows:
\begin{itemize}[leftmargin=*,itemsep=-0.5mm]
    \vspace{-1em}
    \item Appendix~\ref{sec:app_exp_details} provides more details on experimental setups for training,  presents the effect of Monte Carlo estimation and runtime of attacks, and reports the results on backdoored test set.
    \item Appendix~\ref{sec_app_modelclossness} provides proofs for our Theorem~\ref{therom:divergence_round_T_main} and Lemma~\ref{lm:L_z_for_multiclass_logistic_regression} related to model closeness.
    \item Appendix~\ref{sec_app_param_smothing}  gives proofs for our Theorem~\ref{theorem:param_smoothing} related to the parameter smoothing.
\end{itemize}

\section{Experimental Details} \label{sec:app_exp_details}
\subsection{More Details on Experiment Setup for Training}
We focus on multi-class logistic regression (one linear layer with softmax function and cross-entropy loss), which is a convex classification problem.
We train the FL system following our \ourframework{} framework with three datasets: Lending Club Loan Data (\loan{})~\citep{loandataset}, \mnist{}~\cite{lecun-mnisthandwrittendigit-2010}, and \emnist{}~\cite{cohen2017emnist}. The financial dataset \loan{} is a tabular dataset that contains the current loan status (Current, Late, Fully Paid, etc.) and latest payment information, which can be used for loan status prediction. It consists of 1,808,534 data samples and we divide them by 51 US states, each of whom represents a client in FL, hence the data distribution is non-i.i.d. 80\% of data samples are used for training and the rest is for testing. \emnist{} is an extended MNIST dataset that contains 10 digits and 37 letters. In the two image datasets, we split the training data for FL clients in an i.i.d. manner. The data description and other parameter setups are summarized in Table~\ref{tb: Dataset description}. For these datasets, the local learning rate $\eta_i$ is 0.001 for all clients. The server performs an adaptive norm clipping threshold $\rho_t$ that increases by time so that the normal learning ability of the model can be preserved (described in Table~\ref{tb: Dataset description}), and sets the fixed training noise level $\sigma_t=0.01$ ($t<T$).  When the clipping threshold is not a fixed value, $L_{\mathcal Z}$ is calculated based on $\rho_{\tadv}$ following Lemma~\ref{lm:L_z_for_multiclass_logistic_regression} for our experiment,.


\begin{wrapfigure}{r}{0.15\textwidth}
   \centering
     \begin{subfigure}
         \centering
         \includegraphics[width=0.1\textwidth]{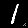}     
 \end{subfigure}
     
\caption{Backdoor pattern for image datasets}
\label{fig:backdoor_pattern} 

\end{wrapfigure}

Regarding the attack setting, by default, we set $R=1$, and if there are more adversarial clients, we use same parameters setups for all of them. For the pixel-pattern backdoor in \mnist{} and \emnist{}, the attackers add the backdoor pattern (see Figure.~\ref{fig:backdoor_pattern} for an example) in images and swap the label of any sample with such patterns into the target label, which is ``digit 0''. Similarly, for the preprocessed\footnote{We preprocess \loan{} by dropping the features which are not digital and cannot be one-hot encoded, and then normalizing the rest 90 features and so that the value of each feature is between 0 and 1.} \loan{} dataset, the attackers increase the value of the two features (i.e., num\_tl\_120dpd\_2m, num\_tl\_90g\_dpd\_24m) as a backdoor pattern, and swap label to ``Does not meet the credit policy. Status:Fully Paid''. 
Since we adopt Lemma~\ref{lm:L_z_for_multiclass_logistic_regression}  for our experiments, we focus on the backdoor pattern $\|\delta_i\|=\|\delta_{i_x}\|$.
The magnitude of backdoored pattern in every example is $\|\delta_i\|=0.1$ on three datasets. Every attacker's batch is mixed with correctly labeled data and such backdoored data with poison ratio $q_{B_i} / n_{B_i}$.

We train the FL global model until convergence and then use our certification in Algorithm~\ref{alg:certify_parameters_perturbation} for robustness evaluation.

\begin{table}[htbp]
\centering
\scalebox{0.85}{
\begin{tabular}{l c c c c c c c c c}
\toprule
Dataset& Classes & $\#$Training samples & Features & $N$ & $q_{B_i} / n_{B_i}$  & $\tau_i$ & $\gamma_i$ & $\tadv$& $\rho_t$   \\
\midrule 
\loan{}& 9& 1446827 & 91 & 20 & 40/800 &  143 & 10 & 6 &  0.025t+2  \\
\mnist{}& 10& 60000 & 784 & 20 & 5/100 &  30 & 10& 10 & 0.1t+2\\
\emnist{}& 47& 697932 & 784 & 50 & 5/200 & 70 & 20& 10 & 0.25t+4\\
\bottomrule 
\end{tabular}}
\caption{Dataset description and parameters}
\label{tb: Dataset description}   
\end{table}
%
%
%
%
%

\subsection{More Experimental Results on Clean Test Set}
\label{ap:more_results_clean_testset}

\paragraph{Effect of Monte Carlo estimation}
Recall that we use $M$ and $\alpha$ when calculating the lower bound $\underline{p_A}$ and the upper bound $\overline{p_B}$.
Figure~\ref{fig:M_alpha} (left) shows that larger number $M$ of noisy models used for certification can result in larger certified radius. Figure~\ref{fig:M_alpha} (middle) presents that the certified radius is smaller when the error tolerance $\alpha$ is smaller but overall the \ceracc{} is not very sensitive to $\alpha$.

\begin{figure}[t]
\centering
\setlength{\belowcaptionskip}{-4mm}
\subfigure
{
	\includegraphics[scale=0.3]{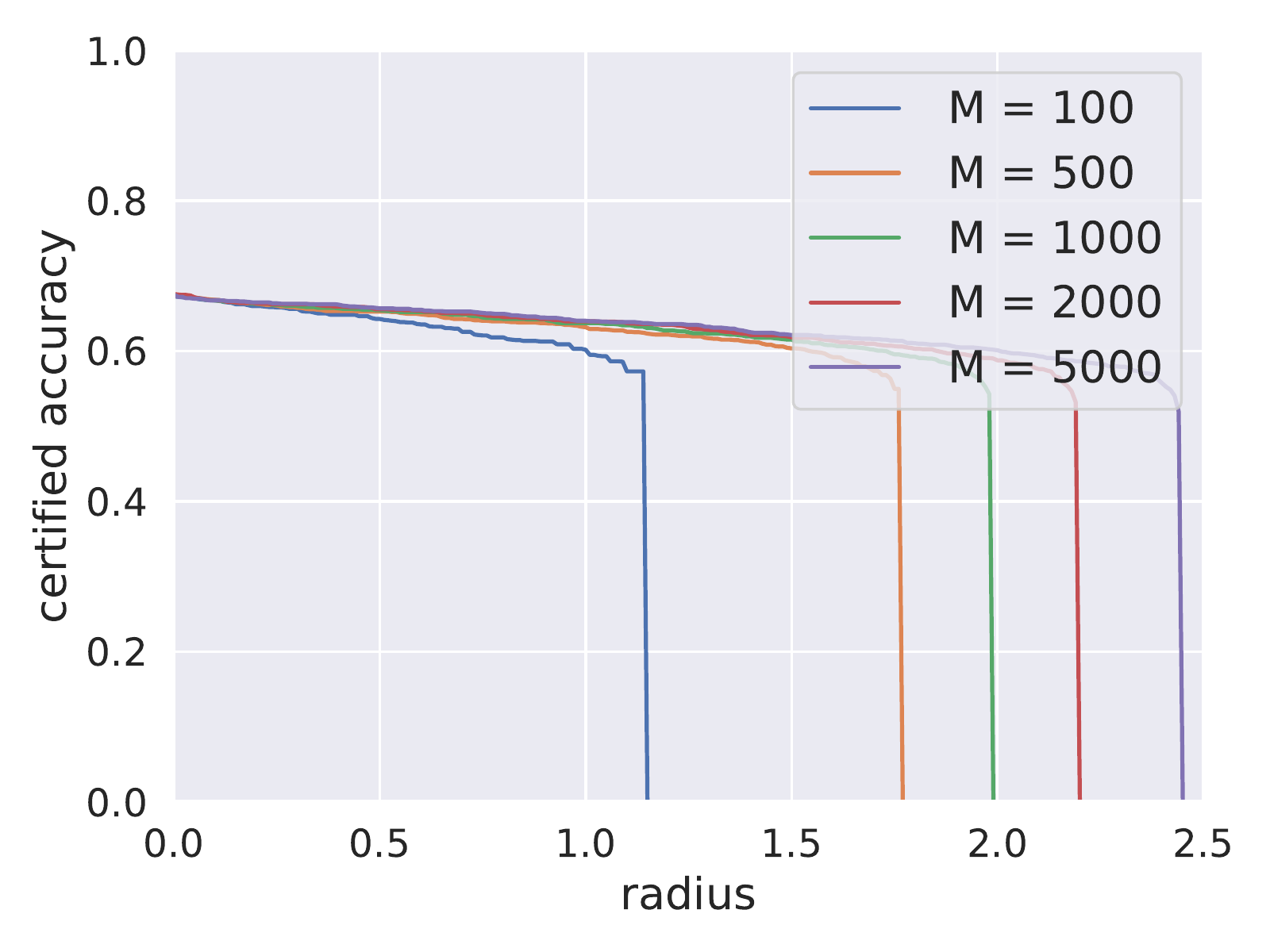}
}
\subfigure
{
	\includegraphics[scale=0.3]{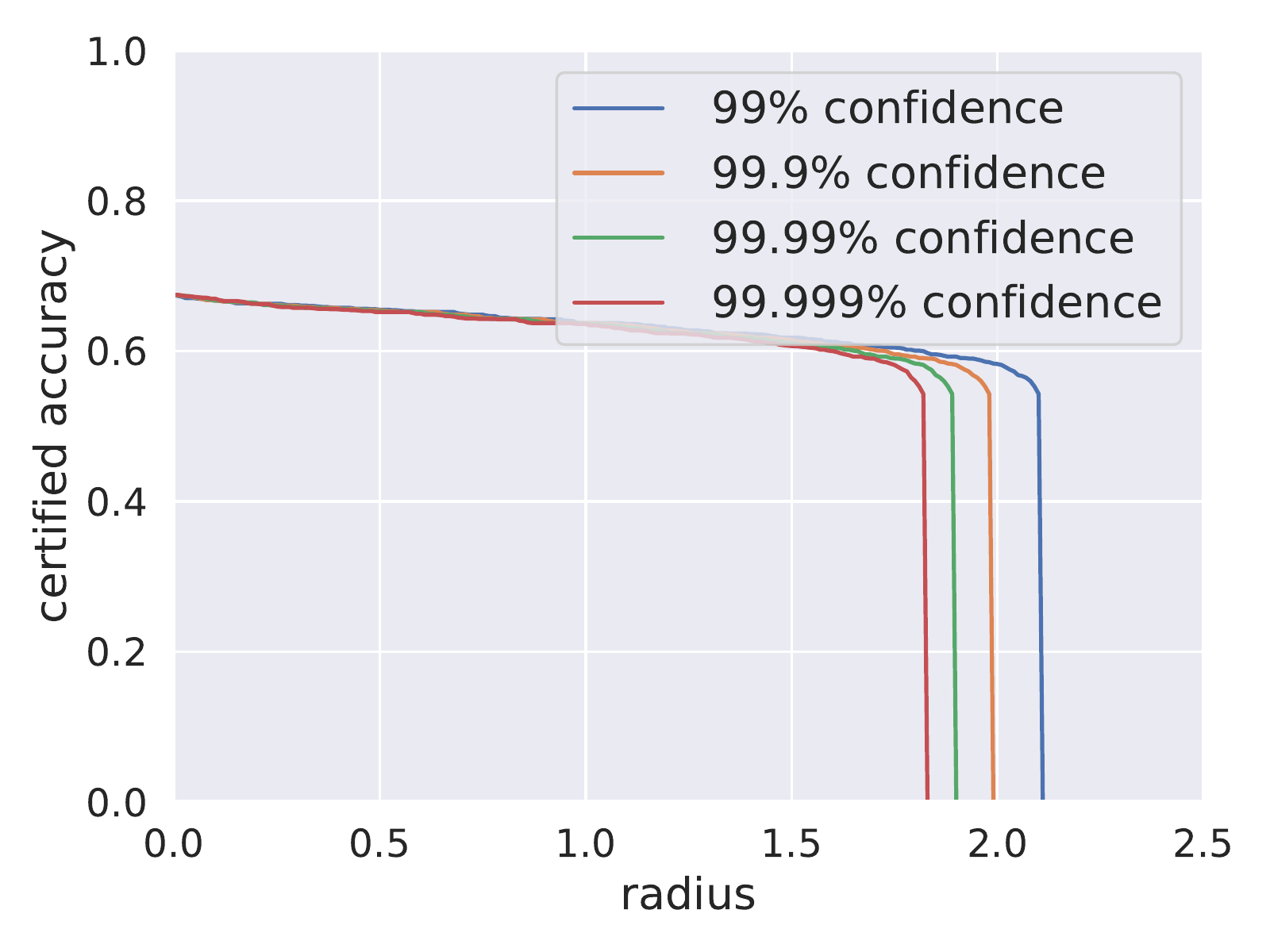}
}
\subfigure
{
	\includegraphics[scale=0.3]{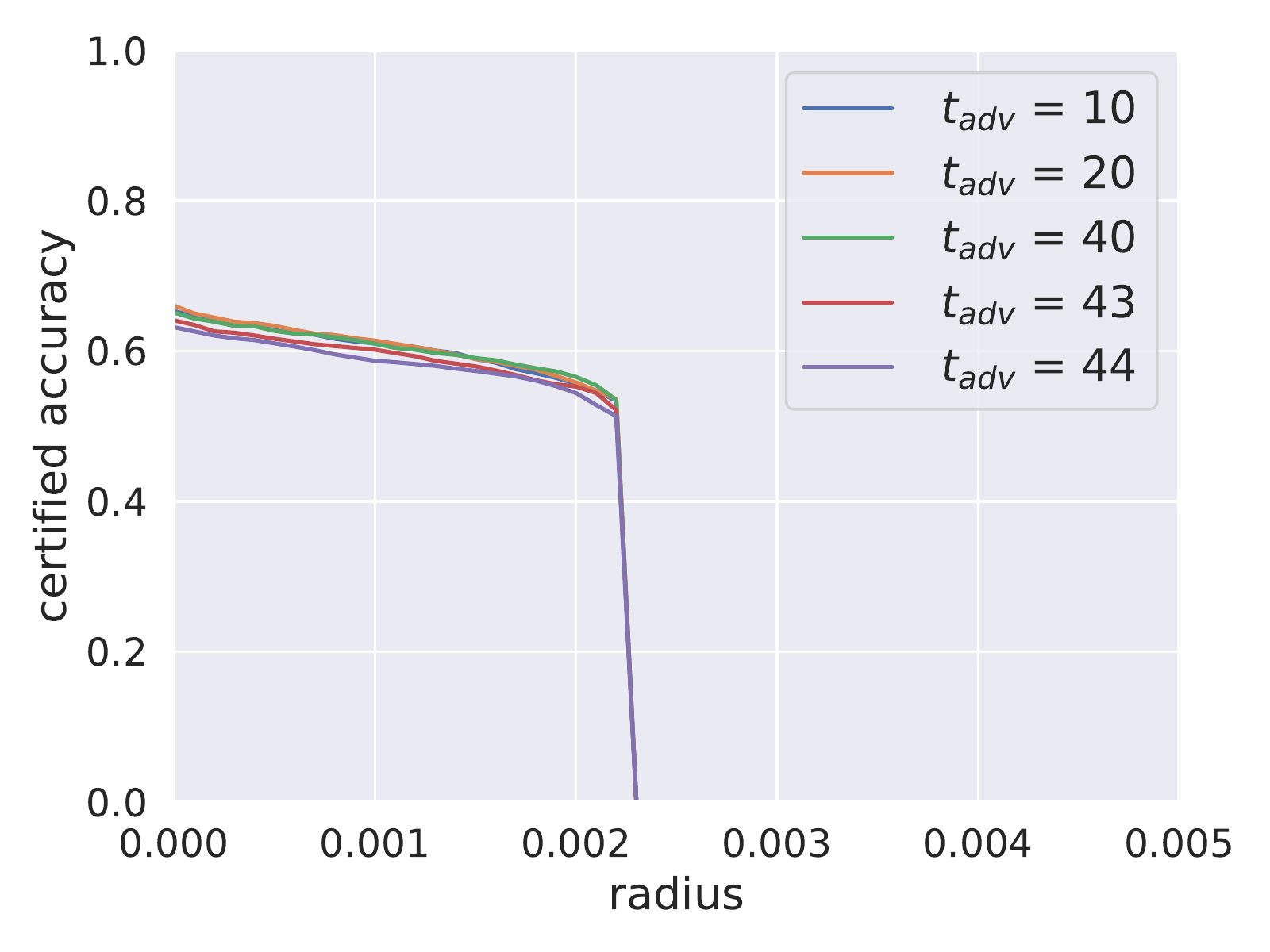}
}
\caption{Left: Certified accuracy on \mnist{} with different number of smoothed models $M$ for certification. Middle: Certified accuracy on \mnist{} with different error tolerance $\alpha$ for certification. Right: Certified accuracy with different $\tadv$ on \mnist{}.} 
\label{fig:M_alpha} 
\end{figure}

\paragraph{Effect of Attack Timing $\tadv$}
For Figure~\ref{fig:M_alpha} (right), we use a strong attack ($\gamma$=100, $R$=2) and report the certified accuracy with different $\tadv$.
As described in Table~\ref{tb: Dataset description}, $\rho_{\tadv}$ increases with $\tadv$, and $L_{\mathcal Z}$ is calculated based on $\rho_{\tadv}$. In order to control variable, we use the same, loose $L_z$ which is calculated based on $\rho_{44}$ for all $\tadv=10,20,40,43,44$.
The results show that the certified radius is not sensitive to the attack timing $\tadv$ after training sufficient number of rounds with clean datasets after $\tadv$. 

\subsection{Experimental Results on Backdoored Test Set}
In this section, we report the \ceracc{} on the backdoored test set. For every test sample, the backdoor pattern is added to the input while the label is still correct. As shown in Figure~\ref{fig:cer_acc_backdoor_mnist_attack_ability} and \ref{fig:cer_acc_backdoor_mnist_others}, the results are similar to the results on the clean test set. 

\begin{figure}[t]
\centering
\setlength{\belowcaptionskip}{-4mm}
\subfigure[Different number of adversarial clients $R$]
{
	\includegraphics[scale=0.24]{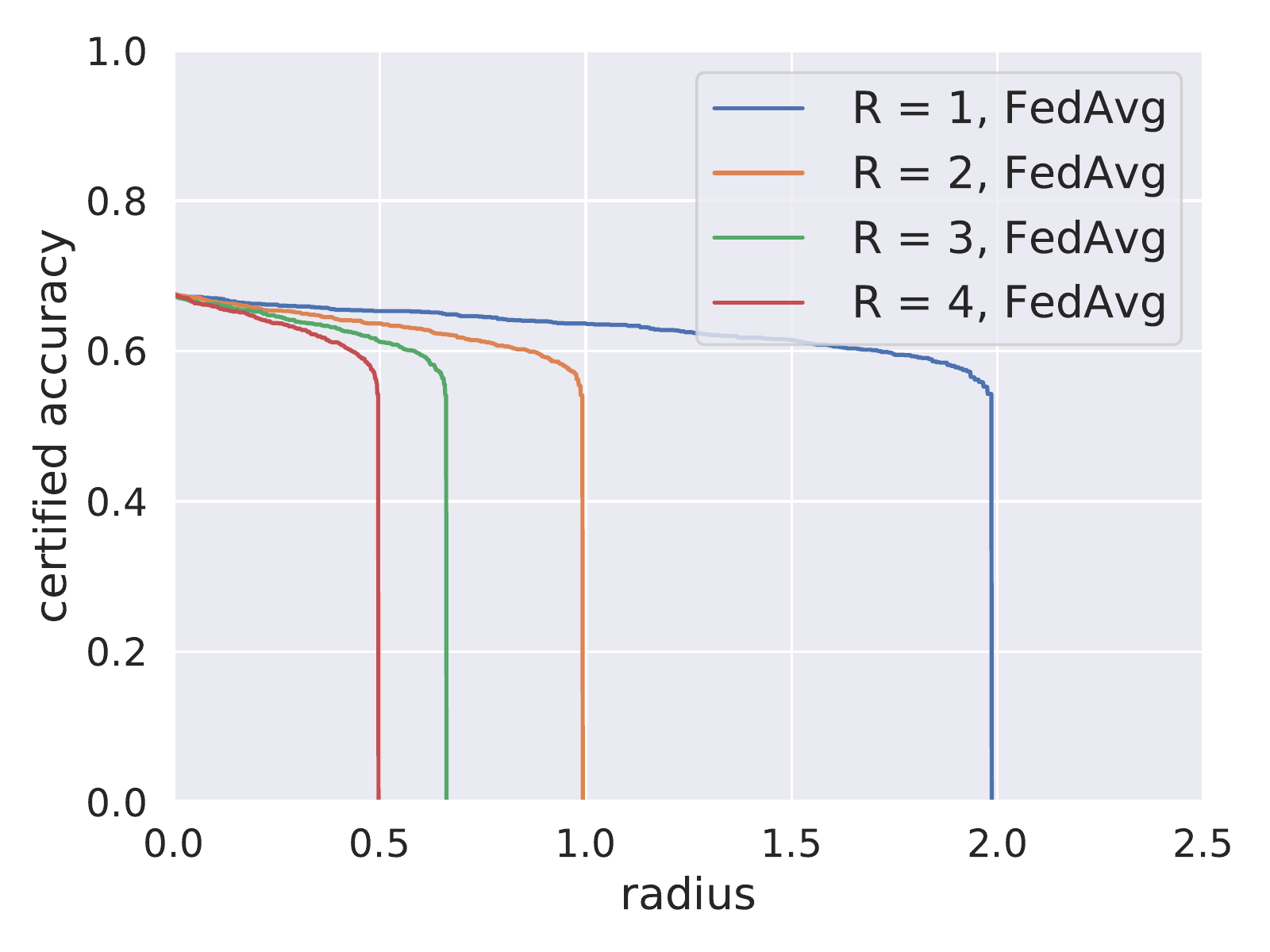}
}
\subfigure[Different $R$ under robust aggregation RFA~\cite{pillutla2019robustrfa}]
{
	\includegraphics[scale=0.24]{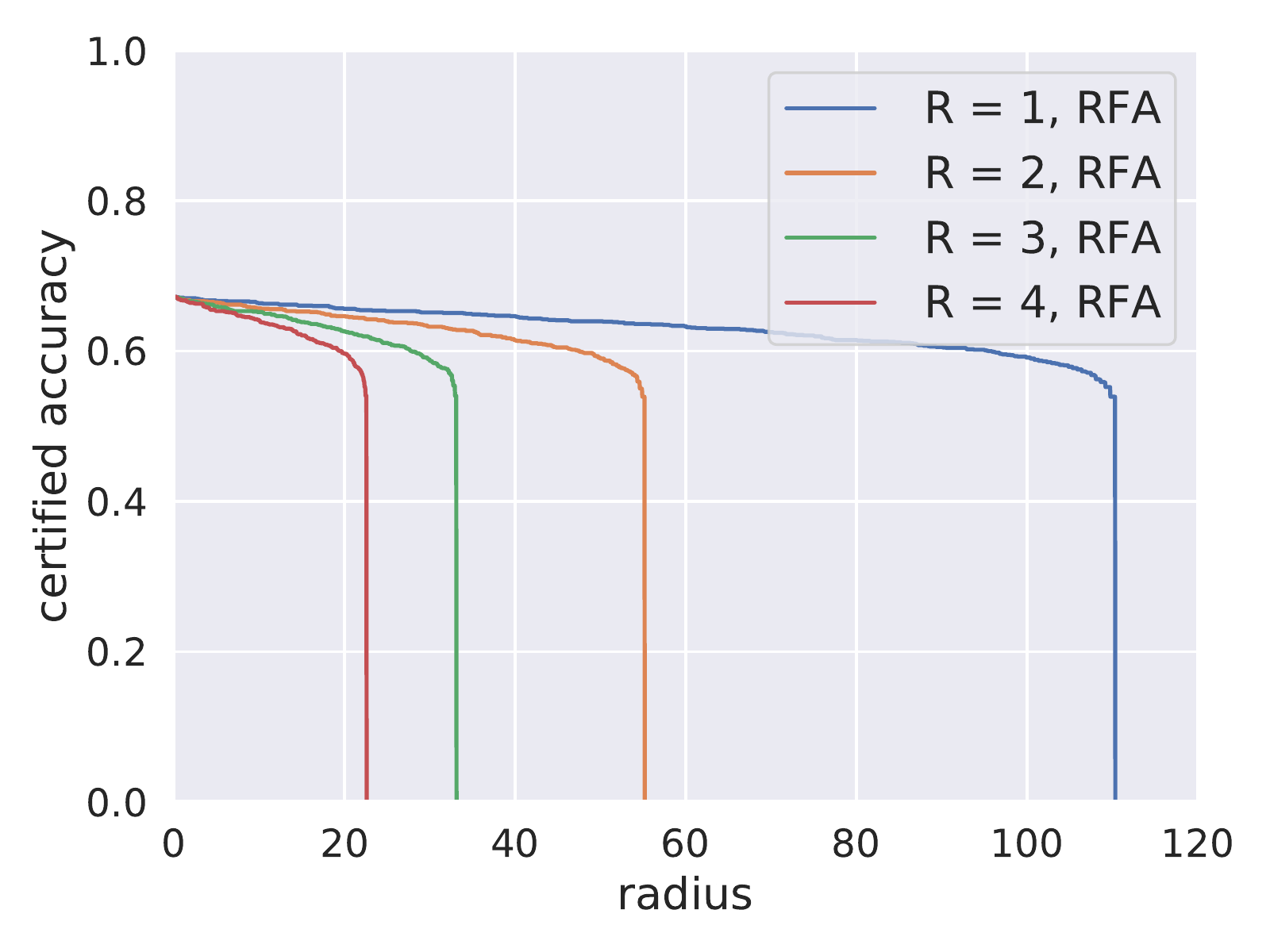}
}
\subfigure[Different scaling factor $\gamma$]
{
	\includegraphics[scale=0.24]{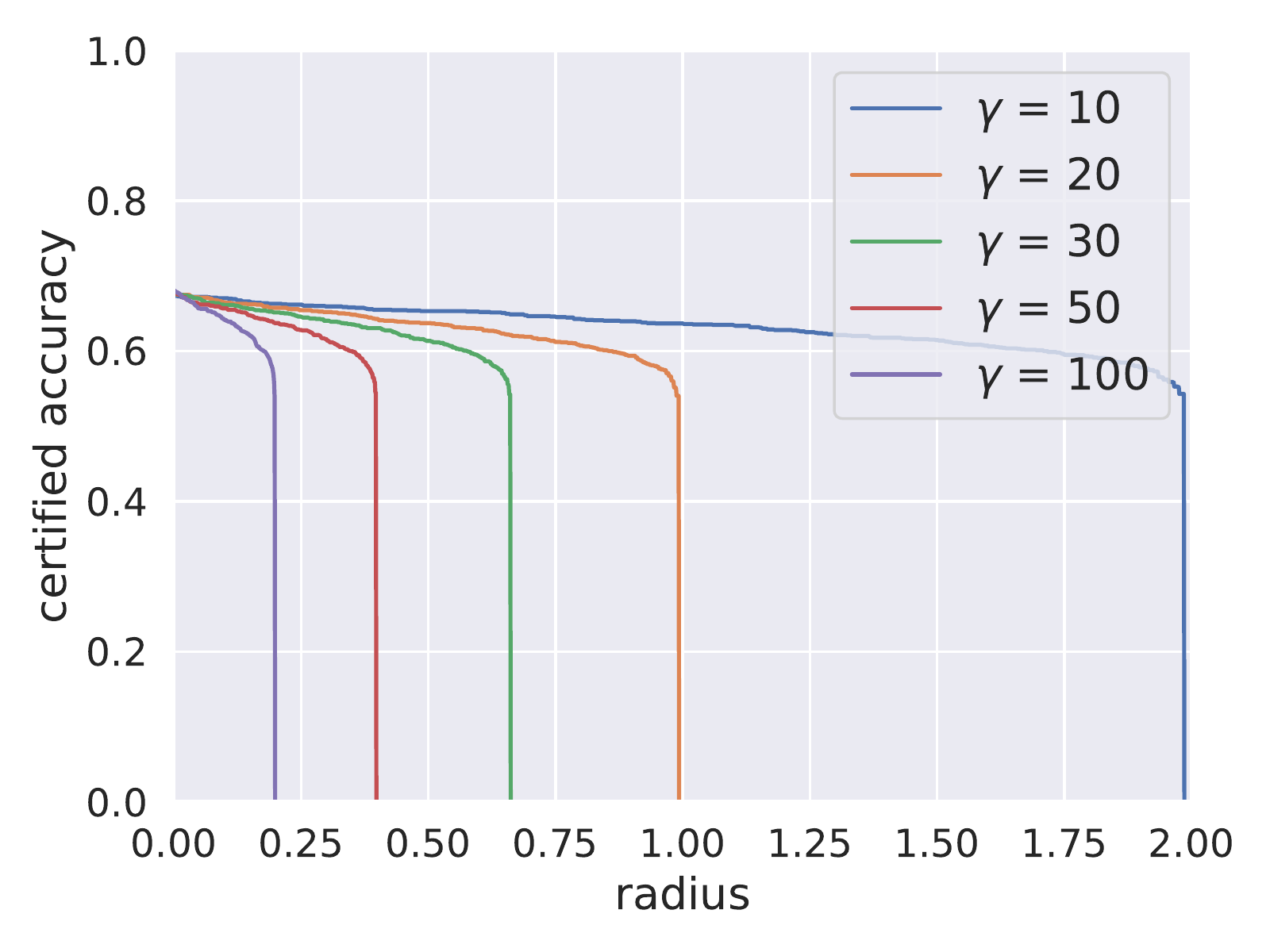}
}
\subfigure[Different poison ratio $q_{B_i}/n_{B_i}$]
{
	\includegraphics[scale=0.24]{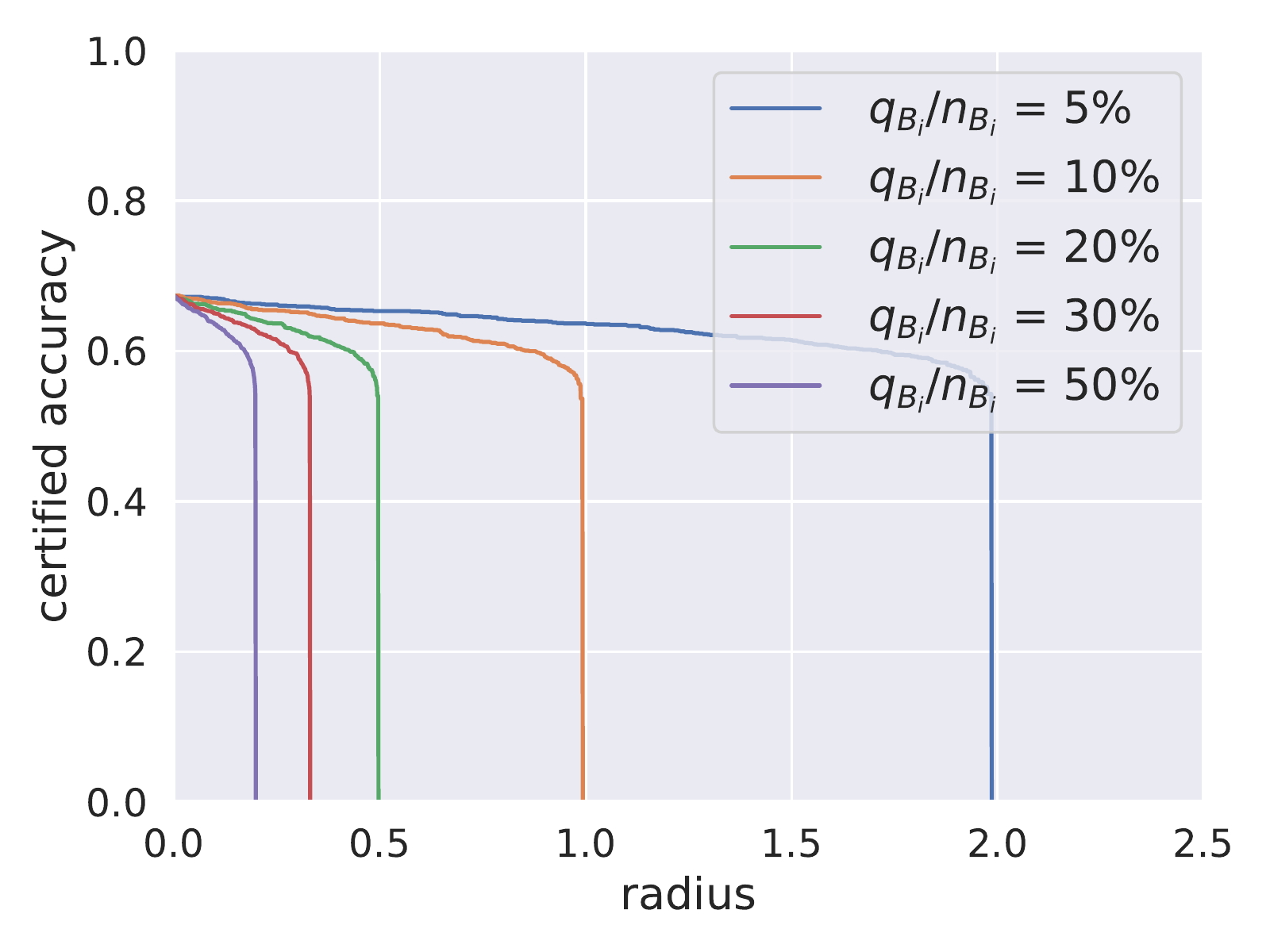}
}
\caption{Certified accuracy with different attack ability (a)(c)(d) and certified accuracy under robust aggregation RFA~\cite{pillutla2019robustrfa} (b) on \mnist{} backdoored test set.} 
\label{fig:cer_acc_backdoor_mnist_attack_ability} 
\end{figure}

\begin{figure}[t]
\centering
\setlength{\belowcaptionskip}{-4mm}
\subfigure[Different training time noise level $\sigma$]
{
	\includegraphics[scale=0.3]{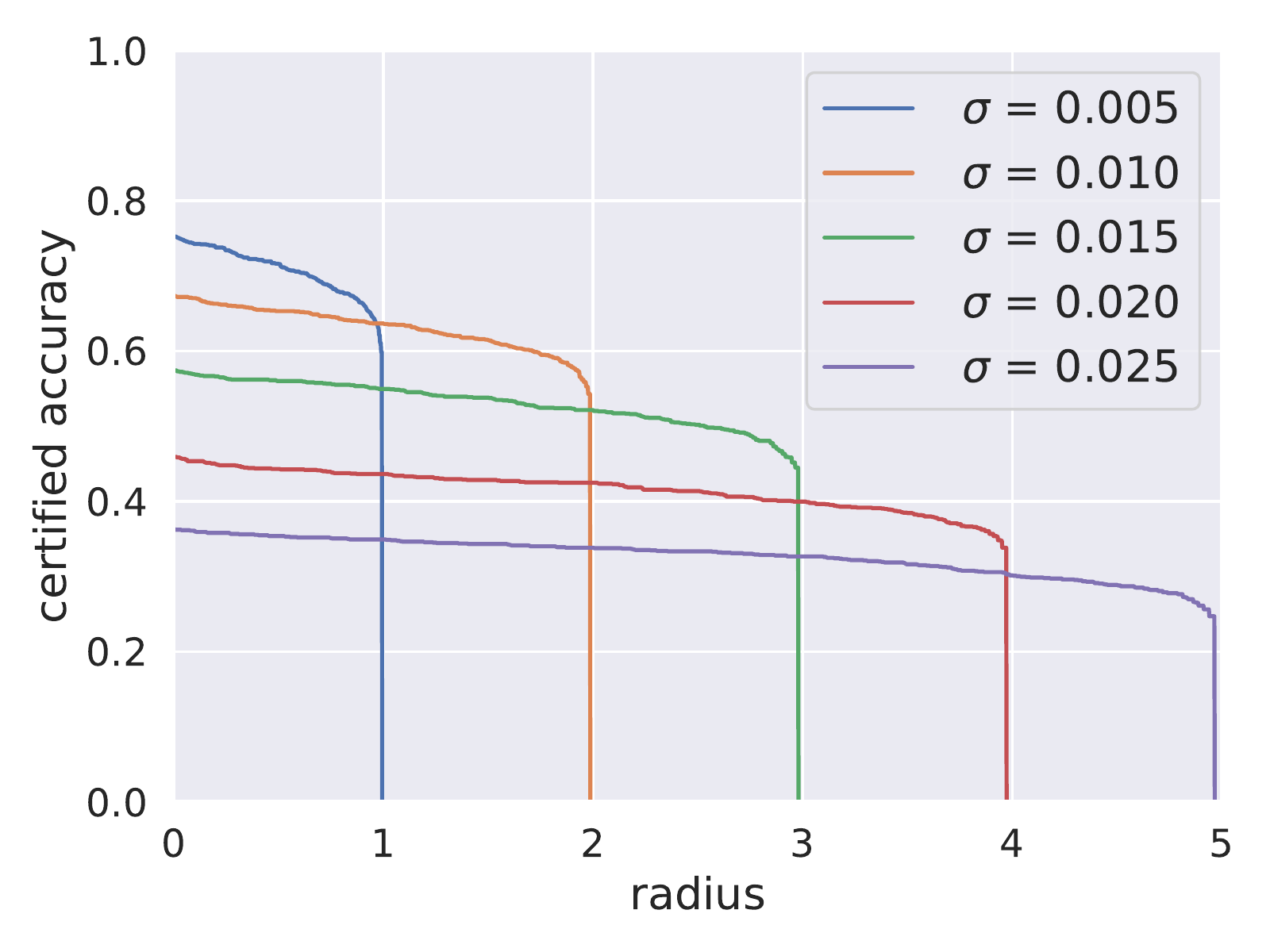}
}
\subfigure[Different number of total clients $N$]
{
	\includegraphics[scale=0.3]{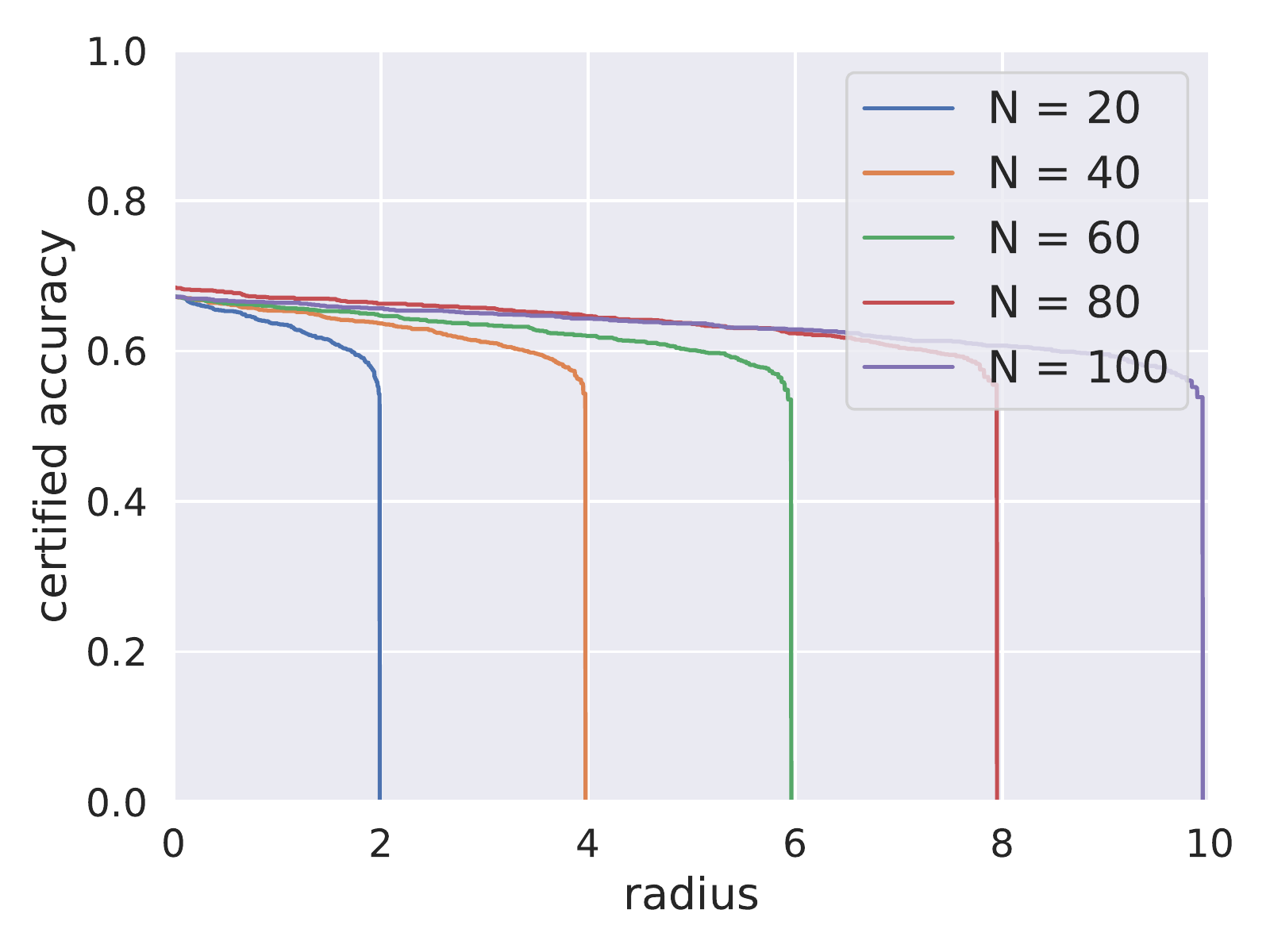}
}
\subfigure[Different training rounds $T$]
{
	\includegraphics[scale=0.3]{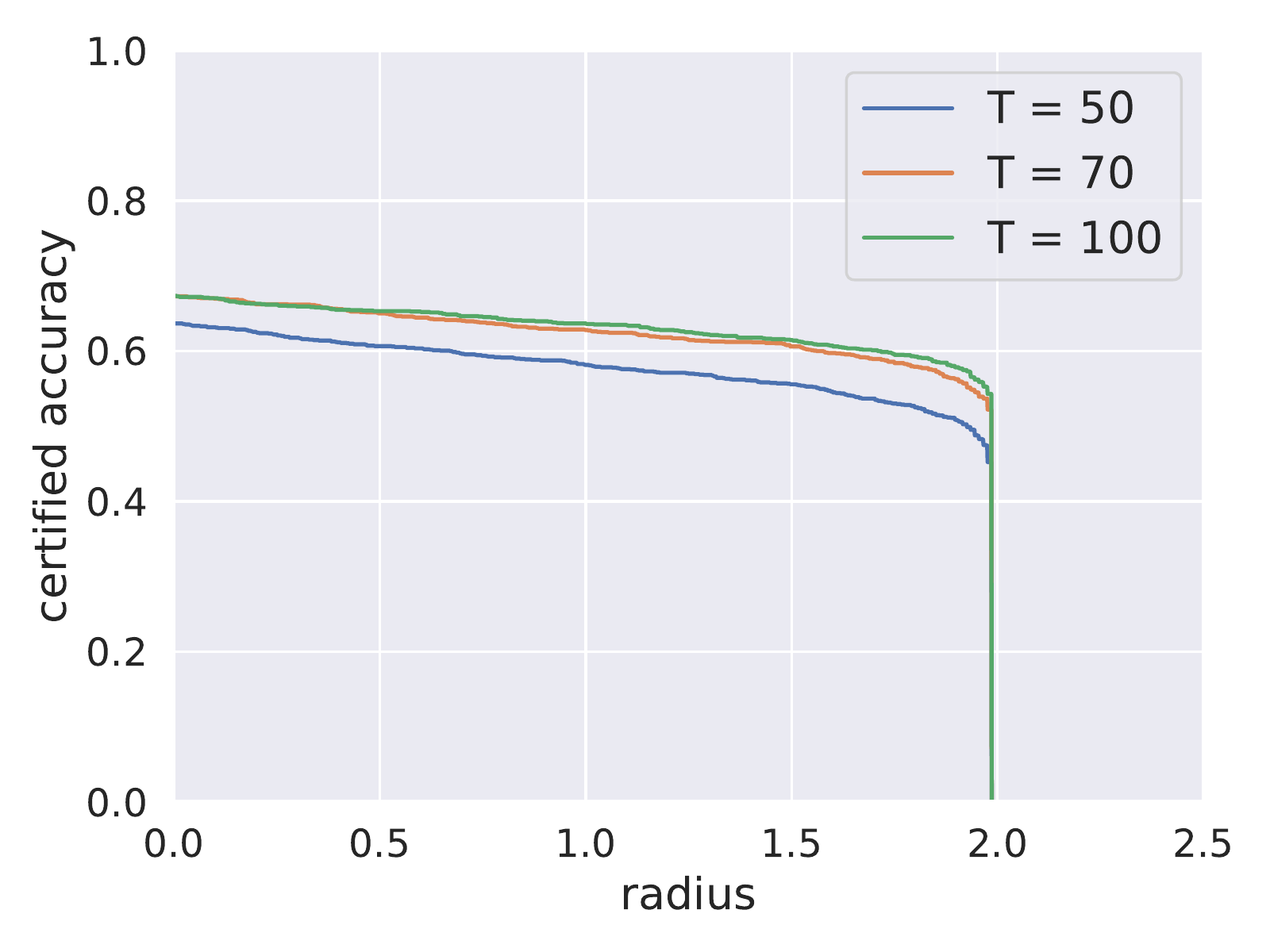}
}
\caption{Certified accuracy with different $\sigma$ (a), $N$ (b) and $T$ (c) on \mnist{} backdoored test set.} 
\label{fig:cer_acc_backdoor_mnist_others} 
\end{figure}

\section{Proofs of Model Closeness} 
\label{sec_app_modelclossness}

In this section, 
we will present preliminaries on \fdivergence{}, define the problem of model closeness and then provide the detailed proofs for our Theorem ~\ref{therom:divergence_round_T_main} and Lemma ~\ref{lm:L_z_for_multiclass_logistic_regression} that are related to model closeness.
 Let us list the notations used in the paper and the Appendix in Table~\ref{tb: notations}.

\begin{table}[htbp]
\scalebox{0.7}{
\begin{tabular}{l c}
\toprule
Notation & Description   \\
\midrule 
$\mathcal{M}(\cdot)$ & the training protocol in Algorithm~\ref{alg:certify_parameters_perturbation} \\
$z_j^i:=\{x_j^i,y_j^i\}$ & $j$-th data sample at client $i$ with input $x_j^i$ and label $y_j^i$\\
${z'}_j^i:=\{x_j^i+{{\delta_i}_x}, y_j^i+{{\delta_i}_y}\}$ & backdoored version of $z_j^i$ where ${\delta_i}_x$ is input backdoor pattern and ${\delta_i}_y$ is label flipping effect \\
$D:=\{S_1,S_2,\ldots,S_N \}$ & Clean training dataset, the union of clean local dataset of $N$ clients\\
$D'= D+ \{\{\delta_i\}_{j=1}^{q_i}\}_{i=1}^R. $ & poisoned training dataset in round $\tadv$ with $R$ attackers and $q_i$ poisoned samples in $i$-th attacker's local dataset\\
$\mathcal{M}(D)$&    the clipped global model obtained from $\mathcal{M}$ using $D$\\
$\mathcal{M}(D')$ & the clipped global model obtained from $\mathcal{M}$ that uses $D'$ at round $\tadv$ and uses $D$ at round $t \neq \tadv$\\
$g_i(w)= g_i(  w; \xi^i) $ & local gradients at client $i$ w.r.t $w$ with clean batch $\xi^i$\\
${g'}_i(w)={g}_i(w; {\xi'}^i) $ & local gradients at client $i$ w.r.t $w$ with poisoned batch ${\xi'}^i$\\
$\mathcal{B}^i \triangleq {g'}_i(w) -  g_i(w) $ & the difference between poisoned local gradient and benign local gradient w.r.t same model parameters $w$ \\
$w_{s}^i$& client $i$'s local model parameters at local iteration $s$ \\
$w_{t} \gets \widetilde w_{t-1}+ \sum \limits_{i = 1}^N p_i (w_{t\tau_i}^i-\widetilde w_{t-1}) $ & aggregated global model at round $t$\\
$ \mathrm{Clip}_{\rho_t}(w_{t}) \gets w_{t}/ \max(1, \frac{\|w_{t}\|}{\rho_t})$ & clipped global model with model parameters norm threshold $\rho_t$ at round $t$\\
$\widetilde w_{t} \gets \mathrm{Clip}_{\rho_t}(w_{t}) + \epsilon_t$ & global model at round $t$ that is perturbed by noise $\epsilon_t$ \\
$h_s$ & the smoothed classifier transferred from the base classifier $h$\\
$p_c={H_s^{c}}(w;x_{test})= \mathbb{P}_{W\sim \mu(w)}[\clsfier(W;x_{test})=c]$ & the probability (the majority votes) of class $c$ for the given $w$ and $x_{test}$\\
${\smoothclsfier}(w;x_{test})= \arg \max_{c\in \mathcal{Y}  } {H_s^{c}}(w;x_{test})$ &the mostly probable label among all classes (the majority vote winner) for the given $w$ and $x_{test}$ \\
\bottomrule 
\end{tabular}
}
 \caption{Table of notations}
  \label{tb: notations}   
  \centering
\end{table}

Throughout this paper, ``benign training process'' is the process that trains with clean dataset $D$ for $T$ rounds and outputs $\mathcal{M}(D)$; ``backdoored training process'' is the process that trains with poisoned dataset $D'$ at round $\tadv$, trains with original clean dataset when $t\neq \tadv$, and outputs $\mathcal{M}(D')$.

\subsection{Preliminaries on \fdivergence{}} \label{sec:f-div preliminaris}
Let $f:(0,\infty) \rightarrow \mathbb{R}$ be a convex function with $f(1)=0$, $\nu$ and $\rho$ be two probability distributions. Then \fdivergence{} is defined as 
\begin{equation}
D_f(\nu || \rho) = E_{W \sim \rho }\left[f(\frac{\nu(W)}{\rho(W)})\right].
\end{equation}
Common \fdivergence{} includes Total variation $f(x)=\frac{1}{2}\|x-1\|$ and Kullback-Leibler (KL) divergence $f(x)=x \log x$.

\begin{lemma} \label{lemma:divergence_guassian}
For $m_1, m_2 \in \mathbb{R}^d $ and $\sigma \textgreater 0$, let $\mathcal N_1 $and $\mathcal N_2$ denote Gaussian distribution $\mathcal N_1(m_1,\sigma^2I)$ and $\mathcal N_2(m_2,\sigma^2I)$, respectively. Then,
\begin{align}
D_{KL}(\mathcal N_1  || \mathcal N_2)=  \frac{\|m_2-m_1\|^2}{2\sigma^2},
\end{align}
\begin{align}
D_{TV}(\mathcal N_1  || \mathcal N_2)= 2\Phi \left(  \frac{\|m_2-m_1\|}{\sigma} \right)-1,
\end{align}
where $\Phi$ is the CDF of the Gaussian distribution.
\end{lemma}

The well-known data processing inequality~\cite{polyanskiy2015dissipation} for the relative entropy states that, for any convex function $f$ and any stochastic transformation (probability transition kernel), i.e., \markovkernel{} $K$, we have
$$
D_f(\nu K||\rho K) \le D_f(\nu ||\rho),
$$
where $\nu K$ denotes the push-forward of $\nu$ by $K$, i.e.,  $\nu K = \int \nu(dW) K(W)$. In other words, $D_f(\nu || \rho)$ decreases by post-processing. \cite{asoodeh2020differentially} extends it into machine learning and the operations in a \markovkernel{} contain one step of Stochastic Gradient Descent (SGD). 

To capture this effect, the quantity of the noisiness of a Markov operator~\cite{raginsky2016strong} for \fdivergence{}, i.e., contraction coefficient~\cite{asoodeh2020differentially}, is defined as
\begin{equation}\label{eq:contraction_def}
\eta_{f}(K) := \sup \limits_{\nu,\rho; {D_f(\nu||\rho)\ne 0}}\frac{D_f(\nu K||\rho K)}{D_f(\nu ||\rho)}.
\end{equation}

\begin{lemma}[Two-point characterization of Total variation~\cite{dobrushin1956central}] \label{lm:two_point_tv}
The supremum in the definition of $\eta_{TV}(K)$ can be restricted to point mass:
\begin{equation}
    \eta_{TV}(K) := \sup \limits_{y_1,y_2 \in y } D_{TV}(K(y_1)||K(y_2))
\end{equation}
\end{lemma}

\begin{lemma}[$\eta_{TV}(K)$ Upper Bound~\cite{makur2019informationphdmit}] \label{lm:contraction_TV_upper_bound}
For any \fdivergence{}, we have
\begin{equation} 
    \eta_{f}(K) \leq \eta_{TV}(K)
\end{equation}
\end{lemma}

\subsection{Problem Definition}

As described in Algorithm \ref{algo:parameters_perturbation}, due to the Gaussian noise perturbation mechanism, in each iteration the global model can be viewed as a random vector with the Gaussian smoothing measure $\mu$.
We use
the \fdivergence{}  between $\mu(\mathcal{M}(D'))$ and $\mu(\mathcal{M}(D))$ as a statistical distance for measuring model closeness. 
According to the data post-processing inequality, when we interpret each round of \ourframework{} as a probability transition kernel, i.e., a \markovkernel{}, the contraction coefficient of \markovkernel{} can help bound the divergence over multiple training rounds of FL. 

\paragraph{Iteration as \markovkernel{}}
We identify each iteration as a \markovkernel{}. At iteration $t$, the central server produces the new model by $ \widetilde w_{t} \gets \mathrm{Clip}_{\rho_t}\left( w_{t} \right)+ \epsilon_t$ where $w_{t}$ is the aggregated model. We denote $w_{t}  = \Psi_t ( \widetilde w_{t-1} ) $, and 
\begin{equation}
    \widetilde w_{t} \gets \mathrm{Clip}_{\rho_t}\left( \Psi_t\left(\widetilde w_{t-1}\right) \right)+ \epsilon_t,
\end{equation}
where
\begin{equation}
    \Psi_t(\widetilde w_{t-1})  \triangleq \widetilde w_{t-1}- \sum_{i=1}^N p_i \eta_i \sum_{s=(t-1)\tau_i+1}^{t\tau_i} g_i\left( w_{s-1}^i; \xi_{s-1}^i\right) 
\end{equation}
is the federated learning SGD process and the local model is initialized as $ w_{(t-1)\tau_i}^i \gets \widetilde w_{t-1}$.
Therefore, iteration $t$ can be realized by $K_t$, a \markovkernel{} associated with the mapping $ \widetilde w_{t-1} \rightarrow \mathrm{Clip}_{\rho_t} (\Psi_t(\widetilde w_{t-1})) + \epsilon_t$. $K_t$ receives $\widetilde w_{t-1}$ and then generates $\widetilde w_{t}$. Let $\mu_t$  denote the distribution of global model ${\widetilde w_{t}}$,  and we have $\widetilde w_{t-1}\sim \mu_{t-1}$, then
$
    \mu_{t} =\int \mu_{t-1}(dy)K_t(y).
$

\paragraph{Model Replacement Attack at $\tadv$ } 
We define the backdoored federated learning SGD process ${\Psi'}_t$ at round $t=\tadv$ as
\begin{align}
    {\Psi'}_{t}(\widetilde w_{t-1})&\triangleq \widetilde w_{t-1} - \sum_{i=1}^{R} p_i \gamma_i \eta_i \sum_{s=(t-1)\tau_i+1}^{t\tau_i} g_i\left( {w'}_{s-1}^i; {\xi'}_{s-1}^i\right)
    - \sum_{i=R+1}^N p_j \eta_j \sum_{s=(t-1)\tau_i+1}^{t\tau_i} g_j\left( {w}_{s-1}^i; {\xi}_{s-1}^i\right)
\end{align}
where the local model is initialized as $ {w'}_{(t-1)\tau_i}^i \gets \widetilde w_{t-1}$. Then we define the corresponding \markovkernel{} $K'_{t}$ associated with the mapping $ \widetilde w_{t-1} \rightarrow \mathrm{Clip}_{\rho_t}( {\Psi'}_t(\widetilde w_{t-1}))+ \epsilon_t$. Through aggregation, the global model is influenced by adversarial clients. Let $\mu'_{t}$  denotes the distribution of backdoored global model ${\widetilde w'_{t}}$,  and we have $\widetilde w_{t-1}\sim \mu_{t-1}$, then 
$
    \mu'_{t} =\int \mu_{t-1}(dy)K'_t(y).
$


\paragraph{After Model Replacement Attack} After $\tadv$, all clients use the original clean datasets to update their local model. However, the global model in the backdoored training process already begins to differ from the one in the benign training process from round $\tadv$ so it is difficult to analysis it through distributed SGD. Therefore, we use \markovkernel{} to quantify the poisoning effect. When $t\textgreater \tadv$, we have $\widetilde w'_{t-1}\sim \mu'_{t-1}$, then
$
    \mu'_{t} =\int \mu'_{t-1}(dy)K_t(y).
$
Because the clean datasets are used for both clean and backdoored training process when $t\textgreater \tadv$, the \markovkernel{} $K_t$ is the same. We define the contraction coefficient \cite{asoodeh2020differentially} as:

\begin{equation}\label{eq:define_eta_f_sup_fl}
    \eta_{f}(K_t):= \sup \limits_{\mu_{t-1},\mu'_{t-1}; \atop {D_f(\mu_{t-1}\|\mu'_{t-1})\ne 0}}\frac{D_f(\mu_{t-1} K_t\|\mu'_{t-1} K_t)}{D_f(\mu_{t-1} \|\mu'_{t-1})}.
\end{equation}

Therefore, $\eta_{f}(K_t)$ can serve as the upper bound for the real
$
\frac{D_f(\mu_{t} \|\mu'_{t})}{D_f(\mu_{t-1} \|\mu'_{t-1})}.
$
Then we write the model closeness $D_f(\mu_T \| \mu_T' )$ as:
\begin{align} \label{eq_df_modelclosess}
    D_f(\mu_T \| \mu_T' )  &=  D_f(\mu_{\tadv} \| \mu'_{\tadv} ) \frac{D_f(\mu_{\tadv+1} \| \mu'_{\tadv+1} ) }{D_f(\mu_{\tadv} \| \mu'_{\tadv} )} \cdot \cdot\cdot \frac{D_f(\mu_T \| \mu_T' )}{D_f(\mu_{T-1} \| \mu_{T-1}' )} \nonumber \\
    &\le D_f(\mu_{\tadv} \| \mu'_{\tadv} )\prod_{t=\tadv+1}^{T} \eta_{f}(K_t).
\end{align}

We will compute $D_f(\mu_{\tadv} \| \mu'_{\tadv} )$ and $\eta_{f}(K_t)$ respectively in the following sections.
\subsection{Analysis for $t=\tadv$}

We would like to bound the divergence of the global model at round $\tadv$ between the benign training process and the backdoor training process, i.e., $D_f(\mu_{\tadv} \| \mu'_{\tadv} )$. We consider KL divergence. Based on the KL divergence for two Gaussian distributions in Lemma~\ref{lemma:divergence_guassian} and Assumption~\ref{assumption:fl_system_train_test}, we have
\begin{align} \label{eq:kldiv_tadv_overview}
D_{KL}(\mu_{\tadv}\|\mu'_{\tadv}) &= D_{KL}\left( \mathcal N\left(\mathrm{Clip}_{\rho_\tadv} \left(  w_{\tadv} \right),\sigma_{\tadv}^2\bf I \right) \| \mathcal N\left(\mathrm{Clip}_{\rho_\tadv} \left(  w'_{\tadv} \right),\sigma_{\tadv}^2\bf I\right)\right) \nonumber  \\ 
&= \frac{\left\| \mathrm{Clip}_{\rho_\tadv} ( w_{\tadv} )  -\mathrm{Clip}_{\rho_\tadv} \left( w'_{\tadv}\right) \right\|^2}{2\sigma_{\tadv}^2}     \nonumber  \\
&\le \frac{\left\|  w_{\tadv} -  w'_{\tadv} \right\|^2}{2\sigma_{\tadv}^2}.    
\end{align}


\paragraph{Accumulated Effect in Local Iterations}
In order to bound $\left\|  w_{\tadv} -  w'_{\tadv} \right\|^2$, we look at the local iterations $ s =(t-1)\tau_i+1, (t-1)\tau_i+2,  \ldots, t\tau_i $ of adversarial client $i$ for the benign training process and the backdoored training process. 
We use ${\locals} = s- (\tadv-1)\tau_i,\locals=1,2,\ldots, \tau_i $ for simplicity. 
We denote $\Delta_{\locals}^i \triangleq w_{\locals}^i -{w'}_{\locals}^i$. Note that $\Delta_0^i =0$ because in the start of round $\tadv$, the initial local model is the same benign global model $ w_{(\tadv-1) \tau_i}^i =  {w'}_{(\tadv-1) \tau_i}^i  = \widetilde w_{\tadv-1} $ for all clients $i \in [N]$ in both benign and backdoored training process.
For simplicity, we will use $g_i(w), {g'}_i(w)$  instead of $g_i(w; \xi), g_i(w; \xi')$  in the rest of this section. 
We denote $\mathcal{B}^i \triangleq {g'}_i(w) -  g_i(w) $.

\begin{lemma}\label{lm:delta s and delta s-1}
Under Assumption~\ref{assumption:Smoothness} and the condition $\eta_i \le \frac{1}{\beta}$, for ${\locals} \in [1, \tau_i ]$, we have 
\begin{equation}
\begin{aligned}
{\Delta_{{\locals}+1}^i}^2 \le  {\Delta_{\locals}^i}^2  + 2\eta_i  \left\| \mathcal{B}^i \right\|   \Delta_{\locals}^i  + 2 \eta_i^2  \left \| \mathcal{B}^i \right \|^2.
\end{aligned}
\end{equation}
\end{lemma}
We defer the proof to Section~\ref{sec:proof_for_some_lemmas}.
Lemma~\ref{lm:delta s and delta s-1} states that the deviation at the current local iteration $\Delta_{\locals}^i$ is added upon the deviation at the last iteration.

\begin{lemma}\label{lm:delta s}
Based on Lemma~\ref{lm:delta s and delta s-1}, under Assumption~\ref{assumption:Smoothness} and the condition $\eta_i \le \frac{1}{\beta}$, for ${\locals} \in [1, \tau_i ]$, we have 
\begin{equation}
\begin{aligned}
\Delta_{\locals}^i \le 2\eta_i  \left\| \mathcal{B}^i \right\|  {\locals}.
\end{aligned}
\end{equation}
\end{lemma}
\begin{proof}
We prove it using induction argument~\cite{zhang2017efficient}.
Due to the fact $\Delta_0^i=0$, so $\Delta_1^i \le \sqrt{  2 \eta_i^2  \left \| \mathcal{B}^i \right \|^2 } \le 2\eta_i  \left\| \mathcal{B}^i \right\|$. Therefore, $ \Delta_{\locals}^i \le 2\eta_i  \left\| \mathcal{B}^i \right\|  \locals $ for $\locals =1$. Suppose the argument $ \Delta_{\locals}^i \le 2\eta_i  \left\| \mathcal{B}^i \right\| \locals  $ holds for some $\locals $, then we verify $\locals+1$, 
\begin{equation}
\begin{aligned}
{\Delta_{\locals+1}^i}^2  &\le 4\eta_i^2  \left\| \mathcal{B}^i \right\|^2 \locals^2+ 4\eta_i^2  \left\| \mathcal{B}^i \right\|^2 \locals +  2 \eta_i^2  \left \| \mathcal{B}^i \right \|^2    \\
&=  \eta_i^2  \left\| \mathcal{B}^i \right\|^2  (4\locals^2+ 8\locals+4)  \\
&\le 4\eta_i^2  \left\| \mathcal{B}^i \right\|^2  (\locals+1)^2. \nonumber \\
\end{aligned}
\end{equation}
It turns out that $ \Delta_{\locals}^i \le 2\eta_i  \left\| \mathcal{B}^i \right\| \locals  $ also holds for ${\locals}+1$. Thus, the argument is correct.
\end{proof}
Lemma~\ref{lm:delta s} states that
the deviation is accumulated over the local iterations. The larger number of local iterations $\tau_i$, the larger deviation $\Delta_{{\tau_i}}^i$. 
Next, we provide the upper bound for $\|\mathcal{B}^i\| $.

\begin{lemma}\label{lm:batch l_z lipz}
Under the Assumption~\ref{assumption:data_lipschitz} on Lipschitz gradient w.r.t. data, when the adversarial clients have ${q_B}_i$ backdoored samples out of a batch with size ${n_B}_i$, we have
\begin{equation}
\begin{aligned}
\|\mathcal{B}^i\| \le  \frac{{{q_B}_i}}{{{n_B}_i}}  L_{\mathcal Z} \|\delta_i\|.
\end{aligned}
\end{equation}
\end{lemma}
\begin{proof}
\begin{align*}
 \left \|\mathcal{B}^i \right\| &=  \left\| {g'}_i(w) -  g_i(w) \right \|    \\
&= \left \|  \frac{1}{n_{B_i}} \left(\sum_{j=1}^{q_{B_i}} \nabla \ell (  w; {z'}_j^i) + \sum_{j=q_{B_i}+1}^{n_{B_i}} \nabla \ell (  w; {z}_j^i) \right) -  \frac{1}{n_{B_i}} \sum_{j=1}^{n_{B_i}}  \nabla \ell (  w; z_j^i)  \right \|   \\
&= \left \|  \frac{1}{n_{B_i}} \sum_{j=1}^{q_{B_i}} \left(\nabla \ell (  w; {z'}_j^i) - \nabla \ell (  w; {z}_j^i) \right)\right \|   \\
&\le  \left  \|  \frac{1}{n_{B_i}} L_{\mathcal Z} \sum_{j=1}^{q_{B_i}}  ({z'}_j^i - {z}_j^i) \right \|   \\
&=  \frac{{{q_B}_i}}{{{n_B}_i}}  L_{\mathcal Z}  \left  \|\delta_i\right\|.   
\end{align*}
\end{proof}

\paragraph{Scaling and Aggregation}
Let the scale factor be $\gamma_i$ for $i$-th adversarial client, then the scaled malicious local update is 
$\gamma_i ( {w'}_{\tadv\tau_i}^i - \widetilde {w}_{\tadv-1}).$
We assume in the benign setting (which is a virtual training process for analyzing, and we do not really train such model), this client also scales its clean local updates as 
$
\gamma_i( {w}_{\tadv\tau_i}^i - \widetilde {w}_{\tadv-1}),
$ 
which can be expanded as 
$
- \eta_i \gamma_i \sum_{s=(\tadv-1)\tau_i+1}^{\tadv\tau_i} g_i\left( w_{s-1}^i; \xi_{s-1}^i\right).
$
This assumption does not hurt the global model performance in the virtual benign setting since the local learning objectives are benign so scaling the updates is equivalent to scale its local learning rate $\eta_i \gets \eta_i \gamma_i  $.

After aggregation, the deviation between global model parameters in benign and backdoored training process can be bounded. Note that the benign local model updates are cancelled out since they are the same in the two training process.

\begin{lemma}\label{lm:scaled_local_diff}
The deviation between the aggregated global model in the benign training process and the global model in the backdoored training process at round $\tadv$ is
\begin{equation}
\begin{aligned}
\|   w_{\tadv} -  w'_{\tadv} \|^2 = R \sum_{i=1}^R  (\gamma_i p_i \Delta_{ {\tau_i}}^i )^2.
\end{aligned}
\end{equation}
\end{lemma}
\begin{proof}[Proof] 
\begin{align*}
&\left\|  w_{\tadv} -  w'_{\tadv} \right\|^2 \\
&= \left\|  \sum_{i=1}^R p_i  \gamma_i ({w}_{\tadv \tau_i}^i -  w_{t-1}) - \sum_{i=1}^R p_i  \gamma_i ({w'}_{\tadv \tau_i}^i -  w_{t-1}) \right \|^2\\
&= \left\| \sum_{i=1}^R p_i \gamma_i \left(  {w}_{\tadv\tau_i}^i- {w'}_{\tadv\tau_i}^i  \right)  \right \|^2\\
&= \left\| \sum_{i=1}^R p_i \gamma_i \Delta_{{\tau_i}}^i   \right \|^2\\
&\le R \sum_{i=1}^R  \left ( p_i \gamma_i \Delta_{{\tau_i}}^i   \right )^2, 
\end{align*}
where we use the fact from linear algebra that $\|\sum_{i=1}^R a_i \|^2 \leq R \sum_{i=1}^R \| a_i \|^2 $.
\end{proof}


\begin{lemma} \label{lemma_kl_tadv}
Under Assumption~\ref{assumption:Smoothness},~\ref{assumption:data_lipschitz},~\ref{assumption:fl_system_train_test} and the condition $\eta_i \le \frac{1}{\beta}$, we have 
\begin{equation}
D_{KL}(\mu_{\tadv}\|\mu'_{\tadv}) \le \frac{ 2R    \sum_{i=1}^R  \left
    (p_i \gamma_i \tau_i \eta_i\ \frac{{{q_B}_i}}{{{n_B}_i}}  L_{\mathcal Z} \|\delta_i\|  \right)^2 }{\sigma_{\tadv}^2}.
\end{equation}
\end{lemma} 
\begin{proof}
Plugging Lemma \ref{lm:delta s} and Lemma \ref{lm:batch l_z lipz} into Lemma \ref{lm:scaled_local_diff}, we have: 
\begin{equation} \label{eq:tadv_global_model_diff}
   \left\|  w_{\tadv} -  w'_{\tadv} \right\|^2 \le R \sum_{i=1}^R  ( 2 p_i \gamma_i \tau_i \eta_i  \frac{{{q_B}_i}}{{{n_B}_i}}  L_{\mathcal Z} \|\delta_i\|  ) ^2.
\end{equation}

Plugging Eq.~\ref{eq:tadv_global_model_diff} to. Eq.~\ref{eq:kldiv_tadv_overview}, it is clear that the divergence of noisy global model parameters between the benign and backdoor training process at round $\tadv$ is bounded.
\end{proof}

\subsection{Analysis for $t>\tadv$}
Now we focus on the contraction coefficient $\eta_{f}(K_t) $ when $t\textgreater \tadv$.
\begin{lemma}\label{lm:eta_tv}
Based on Lemma~\ref{lemma:divergence_guassian} and~\ref{lm:two_point_tv}, under Assumption~\ref{assumption:fl_system_train_test}, we have
\begin{equation}
    \eta_{TV}(K_t) \le 2\Phi \left (\frac{\rho_t }{\sigma_t}\right)-1.
\end{equation}
\end{lemma}
\begin{proof}[Proof] 
\begin{align*}
&\eta_{TV}(K_t) := \sup \limits_{w_1,w_2 \in W } D_{TV}(K_t(w_1)\|K_t(w_2))\\
&\le \sup \limits_{w_1,w_2 \in W } D_{TV}\Bigg( \mathcal N\Big(\mathrm{Clip}_{\rho_t} (\Psi( w_1)),\sigma_t^2{\bf I} \Big) \| \mathcal N\Big(\mathrm{Clip}_{\rho_t} (\Psi( w_2)),\sigma_t^2{\bf I}\Big)\Bigg) \\
&= \sup \limits_{w_3,w_4 \in ball(\rho_t) } D_{TV}\Bigg( \mathcal N\Big( w_3,\sigma_t^2 {\bf I} \Big) \| \mathcal N\Big(w_4,\sigma_t^2\bf I\Big)\Bigg) \\
&= \sup \limits_{w_3,w_4 \in ball(\rho_t) }  2\Phi \left (\frac{\|w_3 -  w_4\| }{2\sigma_t}\right)-1\\
&=  2\Phi \left (\frac{\rho_t }{\sigma_t}\right)-1.  \tag*{ \textit{ $\triangleright$\ the norm of model parameters is bounded by $\rho_t$}}  
\end{align*}
\end{proof}
Finally, we obtain the divergence of global model in round $T$. We restate our Theorem ~\ref{therom:divergence_round_T_main} here.
\begingroup
\def\thetheorem{\ref{therom:divergence_round_T_main}}
\begin{theorem} 
When $\eta_i \le \frac{1}{\beta}$ and Assumptions~\ref{assumption:Smoothness},~\ref{assumption:data_lipschitz}, and~\ref{assumption:fl_system_train_test} hold, the \kldivergence{} between  $\mu(\mathcal{M}(D))$ and  $\mu(\mathcal{M}(D'))$ with $\mu(w) = \mathcal N(w,{\sigma_T}^2\bf I)$ is bounded as: 
\begin{align*}
 D_{KL}( \mu(\mathcal{M}(D)) || \mu(\mathcal{M}(D')) ) \le \frac{ 2R\sum_{i=1}^R\left
    (p_i \gamma_i \tau_i \eta_i \frac{{{q_B}_i}}{{{n_B}_i}}  L_{\mathcal Z} \|\delta_i\|  \right)^2 }{\sigma_{\tadv}^2} \prod_{t=\tadv+1}^{T}  \left(2\Phi \left (\frac{\rho_t }{\sigma_{t}}\right)-1 \right)
\end{align*}
\end{theorem}
\addtocounter{theorem}{-1}
\endgroup

\begin{proof}[Proof]
\begin{align*}
   &D_{KL}( \mu(\mathcal{M}(D)) || \mu(\mathcal{M}(D')) ) = D_{KL}(\mu_T || \mu_T' ) \\
    &\le D_{KL}(\mu_{\tadv} || \mu'_{\tadv} )\prod_{t=\tadv+1}^{T} \eta_{KL}(K_t)  \tag*{ \textit{ $\triangleright$\ because of Eq.~\ref{eq_df_modelclosess}}}\\ 
    &\le  D_{KL}(\mu_{\tadv} || \mu'_{\tadv} )\prod_{t=\tadv+1}^{T} \eta_{TV}(K_t)   \tag*{ \textit{ $\triangleright$\ because of Lemma ~\ref{lm:contraction_TV_upper_bound}}}\\
    & \le \frac{ 2R    \sum_{i=1}^R  \left
    (p_i \gamma_i \tau_i \eta_i\left\| \mathcal{B}^i \right\| \right)^2 }{\sigma_{\tadv}^2} \prod_{t=\tadv+1}^{T}  \left(2\Phi \left (\frac{\rho_t }{\sigma_{t}}\right)-1 \right) \tag*{ \textit{ $\triangleright$\ because of Lemma~\ref{lemma_kl_tadv} and \ref{lm:eta_tv}}}\\
    & \le \frac{ 2R\sum_{i=1}^R\left
    (p_i \gamma_i \tau_i \eta_i \frac{{{q_B}_i}}{{{n_B}_i}}  L_{\mathcal Z} \|\delta_i\|  \right)^2 }{\sigma_{\tadv}^2} \prod_{t=\tadv+1}^{T}  \left(2\Phi \left (\frac{\rho_t }{\sigma_{t}}\right)-1 \right) \tag*{ \textit{ $\triangleright$\ because of Lemma~\ref{lm:batch l_z lipz}}}. 
\end{align*}
\end{proof}

\subsection{Proof of Lemma~\ref{lm:delta s and delta s-1} } \label{sec:proof_for_some_lemmas}





We first introduce a new lemma, which will be used to prove Lemma~\ref{lm:delta s and delta s-1}. 
\begin{lemma}\label{lm:smoothness and convex loss}
Under Assumption~\ref{assumption:Smoothness} on convexity and smoothness, we have
\begin{equation}
\begin{aligned}
\left \| g_i\left(w_{\locals}^i\right) -{g'}_i\left({w'}_{\locals}^i\right) \right \|^2  \le 2 \beta  \left \langle \Delta_{\locals}^i, g_i\left(w_{\locals}^i\right) - g_i\left({w'}_{\locals}^i\right) \right \rangle 
 +  2 \left \| {g'}_i\left({w'}_{\locals}^i\right)-g_i\left({w'}_{\locals}^i\right)\right \|^2. 
\end{aligned}
\end{equation}
\end{lemma}
\begin{proof}
\begin{align*}
&\left \| g_i\left(w_{\locals}^i\right) -{g'}_i\left({w'}_{\locals}^i\right) \right \|^2  \\
&= \left \| \left[g_i\left(w_{\locals}^i\right) - g_i\left({w'}_{\locals}^i\right)\right]   -\left[{g'}_i\left({w'}_{\locals}^i\right)-g_i\left({w'}_{\locals}^i\right)\right] \right \|^2  \\
&\le 2 \left \| g_i\left(w_{\locals}^i\right) - g_i\left({w'}_{\locals}^i\right) \right\|^2 + 2 \left \| {g'}_i\left({w'}_{\locals}^i\right)-g_i\left({w'}_{\locals}^i\right)\right \|^2  \\
&\le 2 \beta  \left \langle \Delta_{\locals}^i, g_i\left(w_{\locals}^i\right) - g_i\left({w'}_{\locals}^i\right) \right \rangle +  2 \left \| {g'}_i\left({w'}_{\locals}^i\right)-g_i\left({w'}_{\locals}^i\right)\right \|^2.  \tag*{ \textit{ $\triangleright$\ because of Assumption ~\ref{assumption:Smoothness}}} 
\end{align*}
\end{proof}
Next we provide the proof of Lemma~\ref{lm:delta s and delta s-1}. 
\begin{proof}[Proof of Lemma~\ref{lm:delta s and delta s-1}]
When $\eta_i \le \frac{1}{\beta}$,
\begin{small}
\begin{align*}
&{\Delta_{{\locals}+1}^i}^2  \triangleq \left \| w_{\locals+1}^i - {w'}_{\locals+1}^i \right \|^2 \\
&= \left \|  (w_{\locals}^i -  {w'}_{\locals}^i)  - \eta_i\left[ g_i\left(w_{\locals}^i\right) -{g'}_i\left({w'}_{\locals}^i\right)\right] \right \|^2  \\
&= {\Delta_{\locals}^i}^2  +  \eta_i^2 \left \| g_i\left(w_{\locals}^i\right) -{g'}_i\left({w'}_{\locals}^i\right) \right \|^2  - 2\eta_i \left \langle w_{\locals}^i -  {w'}_{\locals}^i , g_i\left(w_{\locals}^i\right) -{g'}_i\left({w'}_{\locals}^i\right) \right \rangle \\
&= {\Delta_{\locals}^i}^2  +  \eta_i^2 \left \| g_i\left(w_{\locals}^i\right) -{g'}_i\left({w'}_{\locals}^i\right) \right \|^2   + 2\eta_i  \left \langle w_{\locals}^i -  {w'}_{\locals}^i , {g'}_i\left({w'}_{\locals}^i\right) -  g_i\left({w'}_{\locals}^i\right)  \right \rangle
- 2\eta_i \left \langle w_{\locals}^i -  {w'}_{\locals}^i , g_i\left(w_{\locals}^i\right) - g_i\left({w'}_{\locals}^i\right) \right \rangle \\
&\le {\Delta_{\locals}^i}^2 + 2 \eta_i^2  \left \| {g'}_i\left({w'}_{\locals}^i\right)-g_i\left({w'}_{\locals}^i\right)\right \|^2 + 2\eta_i  \left \langle w_{\locals}^i -  {w'}_{\locals}^i,  {g'}_i\left({w'}_{\locals}^i\right) -  g_i\left({w'}_{\locals}^i\right)  \right \rangle  
+ (2 \beta  \eta_i^2 -2\eta_i)  \left \langle w_{\locals}^i -  {w'}_{\locals}^i , g_i\left(w_{\locals}^i\right) - g_i\left({w'}_{\locals}^i\right) \right \rangle  \tag*{ \textit{ $\triangleright$\ because of Lemma~\ref{lm:smoothness and convex loss}}}  \\ 
&\le {\Delta_{\locals}^i}^2 + 2 \eta_i^2  \left \| {g'}_i\left({w'}_{\locals}^i\right)-g_i\left({w'}_{\locals}^i\right)\right \|^2 + 2\eta_i  \left \langle  w_{\locals}^i -  {w'}_{\locals}^i ,  {g'}_i\left({w'}_{\locals}^i\right) -  g_i\left({w'}_{\locals}^i\right)  \right \rangle  \tag*{ \textit{ $\triangleright$\ because of $\eta_i \le \frac{1}{\beta}$ }}  \\ 
&\le {\Delta_{\locals}^i}^2 + 2 \eta_i^2  \left \| {g'}_i\left({w'}_{\locals}^i\right)-g_i\left({w'}_{\locals}^i\right)\right \|^2 + 2\eta_i  \Delta_{\locals}^i \left\| {g'}_i\left({w'}_{\locals}^i\right) -  g_i\left({w'}_{\locals}^i\right) \right\|  \tag*{ \textit{ $\triangleright$\ because of  $ \langle a,b  \rangle \le \|a\|\|b\| $ }}   \\
&= {\Delta_{\locals}^i}^2  + 2\eta_i \left\| \mathcal{B}^i \right\|  \Delta_{\locals}^i  + 2 \eta_i^2  \left \| \mathcal{B}^i \right \|^2
 \tag*{ \textit{ $\triangleright$\ because of the definition $\mathcal{B}^i \triangleq {g'}_i(w) -  g_i(w) $ }}.
\end{align*}
\end{small}
\end{proof}


\subsection{Proof of Lemma~\ref{lm:L_z_for_multiclass_logistic_regression}} \label{sec:proof l_z}

We first restate our Lemma~\ref{lm:L_z_for_multiclass_logistic_regression} here and then provide the detailed proof.
\lemmalzforlogreg*

\begin{proof}

Given model parameters $W$ of one linear layer, data samples $z=\{x,y\}$ and $z'=\{x',y\}$, we denote their loss as $\ell(W;z)$ and $\ell(W;z')$, where $x \in \mathbb { R }^{1\times d_x}$, $W \in \mathbb { R }^{d_x\times C}$. $Y \in \mathbb { R }^{1\times C}$ is a one-hot vector for $C$ classes where $Y_i = \mathbb{1}\{i=y\}$. 
For $x$, we denote $xW$ as the output of the linear layer, $P_i(x)= \mathrm{softmax}(xW)_i$ as the normalized probability for class $i$ (the output of the softmax function). The cross-entropy loss is calculated as 
\begin{align}
	\ell(x) = - \sum_i Y_i \log P_i(x) = - \sum_i Y_i \log \mathrm{softmax}(xW)_i.
\end{align}
We define $G\in \mathbb { R }^{d_x\times C}$ as the gradient for one sample:
\begin{align}
	G(x)=\nabla \ell(W;\{x,y\})=\frac{d\ell}{dW}(x) = x^{\top}(P(x)-Y),
\end{align}
and we define $G'$ as
\begin{align}
	G(x') =\nabla \ell(W;\{x',y\}) =\frac{d\ell}{dW}(x') = x'^{\top}(P(x')-Y).
\end{align}
According to the mean value theorem \cite{rudin1976principles}, for a continuous vector-valued function $f:[a,b]\to\mathbb { R }^k$ differentiable on $(a,b)$, there exist  $c \in (a,b)$ such that
\begin{align}
\frac{\|f(b)-f(a)\|}{b-a} \le \|f'(c)\|.	
\end{align}
Because $x$ is normalized to $[0,1]$ (a common dataset pre-processing method), when we define $G_l(t) = G(x'+t(x-x')), t\in[0,1]$, based on the mean value theorem we have
\begin{align*}
	\|G(x)-G(x')\| &= \left \|G_l(1)-G_l(0) \right \| \\
	& \le \left\|\frac{dG_l}{dt}(t_0)\right \| (1-0) \\
	& = \left\|\frac{dG}{dx}(\xi)\odot(x-x')\right \| \\	
	& \le \left\|\frac{dG}{dx}(\xi)\right\|\left \|x-x' \right \|
\end{align*}
where $\xi=x'+t_0(x-x'), t_0\in[0,1]$, $\frac{dG}{dx}(\xi)$ is a 3 dimension tenosr and $\odot$ is tensor product.
We reduce the computation to 2 dimension matrix for simplification. Let $G_i$ denote the $i$th colunm of matrix $G$ (the gradient w.r.t $W_i$). Let $\mathbf{1}_i$ denote a row vector where $i$-th element is 1 and the others is 0. We have

\begin{small}
\begin{align*}
&\|G(x)-G(x')\|  \\
&\le \left \|\frac{dG}{dx}(\xi) \right \|\left \|x-x' \right \| \\
&= \sqrt{\sum_i^C \left \|\frac{dG_i}{dx}(\xi)\right \|^2}\left \|x-x' \right \| \\
& = \sqrt{\sum_i^C \left \| \frac{dx^{\top}(P_i-Y_i)}{dx}(\xi) \right \|^2}\left \|x-x' \right \| \tag*{ \textit{ $\triangleright$\ as $G_i(x)  = x^\top(P_i(x)-Y_i)$ }}  \\
& = \sqrt{\sum_i^C \left  \|\frac{dx^{\top}}{dx}(\xi)(P_i-Y_i)+x^{\top}\frac{d(P_i-Y_i)}{dx}(\xi) \right \|^2}\left \|x-x' \right \| \\
& = \sqrt{\sum_i^C \left  \|(P_i(\xi)-Y_i) I+x^{\top} (P_i(\xi)\mathbf{1}_i-P_i(\xi)P(\xi)) W^{\top} \right \|^2} \left \|x-x' \right \|   \tag*{ \textit{ $\triangleright$\ as $\frac{d(P_i-Y_i)}{dx}  = \frac{d\mathrm{softmax}(xW)_i}{dx} = (P_i\mathbf{1}_i-P_iP) W^{\top} $ }}   \\
& \le \sqrt{\sum_i^C  \|(P_i-Y_i)\|^2+2\|(P_i-Y_i)\| \|x^{\top} (P_i\mathbf{1}_i-P_iP) W^{\top}\|+\|x^{\top} (P_i\mathbf{1}_i-P_iP) W^{\top}\|^2}\left \|x-x' \right \| \tag*{ \textit{ $\triangleright$\ denote $P_i$ as $P_i(\xi)$ for simplicity} }   \\
& \le \sqrt{\sum_i^C\|(P_i-Y_i)\|+2||x^{\top} (P_i\mathbf{1}_i-P_iP) W^{\top}\|+\|x^{\top} (P_i\mathbf{1}_i-P_iP) W^{\top}\|^2}\left \|x-x' \right \|, \tag*{ \textit{ $\triangleright$\ as $\|(P_i-Y_i)\| \le 1$ }}  \\
& \le \sqrt{\sum_i^C \|(P_i-Y_i)\|+2P_i\|x\| \|(\mathbf{1}_i-P) W^{\top}\|+P_i^2\|x\|^2 \|(\mathbf{1}_i-P) W^{\top}\|^2}\left \|x-x' \right \| \\
& \le \sqrt{\sum_i^C \|(P_i-Y_i)\|+2P_i \| W\|+P_i \|W\|^2}\left \|x-x' \right \|, \tag*{ \textit{ $\triangleright$\ as $\|x\| \le 1$ and $0\le P_i \le 1$ }}  \\
& \le \sqrt{\sum_i^C\|(P_i-Y_i)\|+2P_i\rho+P_i^2 \rho^2}\left \|x-x' \right \|, \tag*{ \textit{ $\triangleright$\ as $\|W\| \le \rho$ }} \\
& \le \sqrt{2+2\rho + \rho^2}\left \|x-x' \right \| .
\end{align*}	
\end{small}
\end{proof}

\section{Proofs of Parameter Smoothing}\label{sec_app_param_smothing}
In this section, we explain our parameter smoothing for general \fdivergence{}, and give closed-form certification for \kldivergence{}, which corresponds to the proofs for our Theorems~\ref{theorem:param_smoothing}.

\subsection{General Framework for Robustness Certification}
Consider a classifier $h:\mathcal {(W,X)} \rightarrow \mathcal{Y} $. The output of the classifier depends on both the test input and its model parameters (i.e., model weights) of this classifier. In the testing phase, the model weight $w$ is fixed, just like $x_{test}$, so it can be seen as an argument for the classifier $h$. 
For example, in a one-linear-layer model, $h (w; x_{test} )= \mathrm{softmax} (w \times x_{test})$, where $\times$ is the multiplication operation; in  a one-conv-layer model, $h (w; x_{test} )= \mathrm{softmax} (w \circledast x_{test})$ where $\circledast$ is the convolution operation. In a model with multiple layers, the expression of model prediction $h (w; x_{test} )$ also holds, where $w$ consists of the weights from all layers.
To our best knowledge, this is the first work to study \textit{parameter} smoothing on $w$ rather than input smoothing on $x_{test}$.

We want to verify the robustness of smoothed multi-class classifier. Recall that we smooth the classifier $h:\mathcal {(W,X)} \rightarrow \mathcal{Y} $ with finite set of label $ \mathcal{Y} $ using a smoothing measure $\mu: \mathcal{W} \mapsto \mathcal{P(W)}$. The resulting randomly smoothed classifier $h_s$ is 
\begin{equation}
    \smoothclsfier(w;x_{test})= \arg \max_{c\in \mathcal{Y}} \mathbb{P}_{W\sim \mu(w)}[\clsfier(W;x_{test})=c]
\end{equation}

Our goal is to certify that the prediction $ h_s(w;x_{test})$ is robust to model parameters perturbations of size at most $\epsilon$ measured by some distance function $d$, i.e.,
\begin{equation} \label{eq:goal}
    \smoothclsfier(w';x_{test})=  \smoothclsfier(w;x_{test}) \text{   $\forall{w'}$ such that $d(w,w')\le \epsilon$} 
\end{equation}

We assume $\mathcal W \subseteq \mathbb{R}^d $ (a $d$ dimensional model parameters space). Our framework involves a reference measure $\rho=\mu(w)$,  the set of perturbed distributions $\mathcal{D}_{w,\epsilon}=\{\mu(w'): d(w,w')\le \epsilon \}$, and a set of specifications $\phi: \mathcal {(W,X)} \rightarrow \mathcal{Z} \subseteq \mathbb{R}$. Specifically, let $c= h_s(w;x_{test})$. Since we are working on the multi-class classification problem, for every pair of classes $\{c, c'\}$ where $c'\in \mathcal{Y} \setminus \{c\}$, we need a $\phi$, which is a generic function over the model parameters space that we want to verify has robustness properties. Following \cite{dvijotham2020framework}, for every $c'\in \mathcal{Y} \setminus \{c\}$, we define a specification $\phi_{c, c'}: \mathcal {(W,X)} \mapsto \{-1, 0, +1\}$ as follows:
\begin{equation} \label{eq:specification_c_c'}
\phi_{c, c'} (w) = 
\begin{cases}
+1 & \text{if $h(w;x_{test})=c$}\\
-1 & \text{if $h(w;x_{test})=c'$}\\
0 & \text{otherwise}
\end{cases}
\end{equation}
where we denote $\phi_{c, c'} (w;x_{test})$ as $\phi_{c, c'} (w)$ for simplicity.

\begin{proposition} \label{eq:goal_specification}
The smoothed classifier $\smoothclsfier$ is  robustly certified, i.e., Eq.~\ref{eq:goal} holds, if and only if for every $c'\in \mathcal{Y} \setminus \{c\}$, $\phi_{c, c'}$ is robustly certified at $\mu(w)$ w.r.t $\mathcal{D}_{w,\epsilon}$.
Verifying that a given specification $\phi$ is robustly certified is equivalent to checking if the optimal value of the following optimization problem is non-negative:
\begin{equation}
    OPT(\phi, \rho, \mathcal{D}_{w,\epsilon}) := \min_{\nu \in \mathcal{D}_{w,\epsilon}} \E_{W'\sim \nu}(\phi(W'))
\end{equation}

\end{proposition}

\begin{proof}
Note that for any perturbed distribution $\nu \in \mathcal{D}_{w,\epsilon}$, according to the definition of expectation and  Eq.~\ref{eq:specification_c_c'}, we have
\begin{equation}
\begin{aligned}
    \E_{W' \sim \nu}[\phi_{c, c'}(W')]  = \mathbb{P}_{W' \sim \nu}[h(W'; x_{test})=c] - \mathbb{P}_{W' \sim \nu}[h(W'; x_{test})=c'].
\end{aligned}
\end{equation}
Therefore, $\E_{W' \sim \nu}[\phi_{c, c'} (W')] \ge 0$ for all $c'\in \mathcal{Y} \setminus \{c\}$ is equivalent to $c = \arg \max_{y\in \mathcal{C}} \mathbb{P}_{W'\sim \nu}[h(W'; x_{test})=y]$. For $\nu = \mu(w')$, this means that $h_s(w'; x_{test}) = c$.
In other words, $\E_{W' \sim \nu}[\phi_{c, c'}(W')] \ge 0$ for all $c'\in \mathcal{Y} \setminus \{c\}$  and all $\nu = \mu(w') \subset \mathcal{D}_{w,\epsilon}$  if and only if $h_s(w';x_{test})=c$  for all $w'$ such that $d(w,w')\le \epsilon$, proving the required robustness certificate.	
\end{proof}

Then we define the certification problem \footnote{It is called information-limited robust certification in \cite{dvijotham2020framework}  for input smoothing.}:
\begin{definition}\label{def:the class of specifications}
Given a reference distribution $\rho \in \mathcal{P(W)}$, probabilities $p_A$ ,$p_B$ that satisfy $p_A, p_B \ge 0 $, $p_A + p_B \le 1$, 
we define the class of specifications $S$:
\begin{equation}
    S = \{\phi : \mathcal {(W,X)} \mapsto \{-1, 0, +1\} \text{ s.t. } \mathbb{P}_{W\sim \rho}[\phi(W) = +1] \ge p_A,
    \mathbb{P}_{W\sim \rho}[\phi(W) = -1] \le p_B \}
\end{equation}
\end{definition}
Given the above definition of $S$, we can rewrite Proposition \ref{eq:goal_specification} as:
\begin{proposition} \label{proposition_S}
The smoothed classifier $\smoothclsfier$ is  robustly certified, i.e., Eq.~\ref{eq:goal} holds, if and only if $S$ is robustly certified at $\mu(w)$ w.r.t $\mathcal{D}_{w,\epsilon}$.
Verifying that $S$ is robustly certified is equivalent to checking if the condition $\E_{W' \sim \nu}[\phi (W')] \ge 0$ holds for all $\nu \in \mathcal{D}_{w,\epsilon}$ and $\phi \in S$.
\end{proposition}
We need to provide guarantees that hold simultaneously over a whole class of specifications ($\phi_{c, c'} $ for all $c'\in \mathcal{Y} \setminus \{c\}$ ). In fact, $p_A$ can be the seen as the ``votes'' for the top-one class $c$, and $p_B$ can be seen as the ``votes'' for the runner-up class.
We note that the function $f(\cdot)$ used in \fdivergence{} is convex. As shown in \cite{dvijotham2020framework} (but for input smoothing), for perturbation sets $\mathcal{D}_{w,\epsilon}=\{\mu(w'): d(w,w')\le \epsilon \}=\{\nu: D_f(\nu\| \mu(w))\le \epsilon \}$ specified by a \fdivergence{} $D_f$ bound $\epsilon$, this certification task can be solved efficiently using convex optimization.

\begin{theorem}\label{theorem:nonnegative_opt}
Let $D_f$ be \fdivergence{}, $\epsilon$ be the divergence constraint, $S$, $p_A , p_B$ be as in Definition~\ref{def:the class of specifications}. 
The smoothed classifier $\smoothclsfier$ is robustly certified at reference distribution $\rho$ with respect to $\mathcal{D}_{w,\epsilon}=\{\nu: D_f(\nu\| \rho)\le \epsilon \}$ if and only if the optimal value of the following convex optimization problem is non-negative:
\begin{equation} 
\begin{aligned}
  \max_{\lambda \ge 0, \kappa}    \kappa  - \lambda  \epsilon  
  - p_A f_\lambda^*(\kappa - 1) 
  - p_B f_\lambda^*(\kappa + 1) 
  - (1- p_A - p_B) f_\lambda^*(\kappa)  \ge 0
\end{aligned}
\end{equation}
\end{theorem}
\begin{proof}[Proof] 
We prove the theorem according to Proposition~\ref{proposition_S}.
Let $\rho(W)$ be the clean model parameters distribution, $\nu(W)$ be the perturbed model parameters distribution, $r(W)=\frac{\nu(W)}{\rho(W)}$ be likelihood ratio. 
We have
\begin{equation}
\begin{aligned}
    \E_{W\sim \nu} [\phi(W)] &= \E_{W\sim \rho}[r(W)\phi(W)],\\ 
    D_f(\nu \|\rho) &= \E_{W\sim \rho}[f(r(W))], \\
    \E_{W \sim \rho}[r(W)]&=1.
\end{aligned}
\end{equation}
The third condition is obtained using the fact that $\nu$ is a probability measure. The optimization over $\nu$, which is equivalent to optimizing over $r$, can be written as 
\begin{equation} \label{eq: opt-problem}
\begin{aligned}
    \min_{r\ge 0} &~ \E_{W\sim \rho}[r(W)\phi(W)] \\
    s.t. &~\E_{W\sim \rho}[f(r(W))]\le \epsilon, \E_{W \sim \rho}[r(W)]=1
\end{aligned}
\end{equation}
We solve the optimization using Lagrangian duality as follows. We first dualize the constraints on $r$ \cite{dvijotham2020framework} to obtain
\begin{equation} \label{eq:middle-dualize-problem} 
\begin{aligned}
    & \min_{r\ge 0} \E_{W\sim \rho}[r(W)\phi(W)] + \lambda (\E_{W\sim \rho}[f(r(W))] -  \epsilon) + \kappa ( 1- \E_{W \sim \rho}[r(W)])\\
    &= \min_{r\ge 0} \E_{W\sim \rho}[r(W)\phi(W)+ \lambda f(r(W)) -  \kappa  r(W) ]  + \kappa - \lambda  \epsilon \\
    &=   \kappa  - \lambda  \epsilon  -   \E_{W\sim \rho} [\max_{r\ge 0}  \kappa  r(W)  - r(W)\phi(W)-\lambda f(r(W))  ]\\
    &=   \kappa  - \lambda  \epsilon  -   \E_{W\sim \rho} [\max_{r\ge 0}   r(W) (\kappa - \phi(W))-\lambda f(r(W))  ]\\
    &=   \kappa  - \lambda  \epsilon  -   \E_{W\sim \rho} [\max_{r\ge 0}   r(W) (\kappa - \phi(W))- f_\lambda(r(W))  ]\\
    &\le   \kappa  - \lambda  \epsilon  -   \E_{W\sim \rho} [ f_\lambda^*(\kappa - \phi(W))  ]
\end{aligned}
\end{equation}
where $f_\lambda^*(u)=  \max_{v\ge 0} (uv-f_\lambda(v)), f_\lambda(v) = \lambda f(v) $.
By strong duality, maximizing the final expression in Eq.~\ref{eq:middle-dualize-problem} with respect to $\lambda \ge 0, \kappa$ achieves the optimal value in Eq.~\ref{eq: opt-problem}. If the optimal value is non-negative, the specification $S$ is robustly certified.

\begin{equation} 
\begin{aligned}
  \max_{\lambda \ge 0, \kappa}    \kappa  - \lambda  \epsilon  -   \E_{W\sim \rho} [ f_\lambda^*(\kappa - \phi(W))  ]
\end{aligned}
\end{equation}
We can plug in $p_A,p_B$ defined in Definition~\ref{def:the class of specifications}:
\begin{equation} 
\begin{aligned}
  \max_{\lambda \ge 0, \kappa}    \kappa  - \lambda  \epsilon  
  - p_A f_\lambda^*(\kappa - 1) 
  - p_B f_\lambda^*(\kappa + 1) 
  - ( 1- p_A - p_B) f_\lambda^*(\kappa) 
\end{aligned}
\end{equation}
where $p_A = \mathbb{P}_{W\sim \rho}[\phi(W) = +1]$,
$p_B = \mathbb{P}_{W\sim \rho}[\phi(W) = -1]$,
$ 1- p_A - p_B = \mathbb{P}_{W\sim \rho}[\phi(W) =0]  $,
\end{proof}
\begin{remark}
    Note that our differences from \cite{dvijotham2020framework} are in two aspects: (1)~Our certification is with respect to the smoothing scheme on model parameters $W$; (2)~We concretize the corresponding Theorem 2 in \cite{dvijotham2020framework} by the explicit constraints on $p_A$, $p_B$.
\end{remark}

\subsection{Closed-form Certificate for KL Divergence}
We instantize Theorem~\ref{theorem:nonnegative_opt} with \kldivergence{}.
\begin{lemma}\label{lemma:KL closed-form certificate}
Let $D_{KL}$ be the \kldivergence{}, $\epsilon$ be the divergence constraint, $S$, $p_A , p_B$ be as in Definition~\ref{def:the class of specifications}. 
The smoothed classifier $\smoothclsfier$ is robustly certified at reference distribution $\rho$ with respect to $\mathcal{D}_{w,\epsilon}=\{\nu: D_{KL}(\nu\| \rho)\le \epsilon \}$ if and only if:
\begin{equation} 
\begin{aligned}
  \epsilon \le - \log \Big(1-(\sqrt{p_A} - \sqrt{p_B})^2\Big) 
\end{aligned}
\end{equation}
\end{lemma}

\begin{proof}[Proof for Lemma~\ref{lemma:KL closed-form certificate}]
The function $f(u) = ulog(u)$ for \kldivergence{} is a convex function with $f(1)=0$ , then we have
$$
    f_\lambda^*(u)=  \max_{v\ge 0} (uv- \lambda f(v)) = \max_{v\ge 0}(uv- \lambda v \log(v)).
$$
Setting the derivative with respect to $v$ to 0 and solving for $v$, we obtain
$v = \exp \left(\frac{u-\lambda}{\lambda} \right), \lambda \textgreater 0 $.
So we have
\begin{equation}
    f_\lambda^*(u) = \lambda \exp \left(\frac{u}{\lambda}-1 \right).
\end{equation}

Suppose we have a bound on the KL divergence $D_f(\nu \|\rho)\le \epsilon$, then we want that the optimal certificate is non-negative:
\begin{equation} \label{eq:opt_certifate_kl}
\begin{aligned}
  \max_{\lambda \textgreater 0, \kappa}   \Bigg( \kappa  - \lambda  \epsilon - p_A \lambda \exp \left(\frac{\kappa - 1}{\lambda}-1 \right)  - p_B \lambda \exp \left(\frac{\kappa + 1}{\lambda}-1 \right) - (1-p_A-p_B) \lambda \exp \left(\frac{\kappa}{\lambda}-1 \right) \Bigg)\ge 0.
\end{aligned}
\end{equation}
Setting $y=\kappa/\lambda, z= \frac{1}{\lambda}(z\textgreater0)$,
we can rewrite Eq.~\ref{eq:opt_certifate_kl} as:
\begin{equation} \label{eq:opt_certifate_kl_substitute}
\begin{aligned}
  \max_{z \textgreater 0, y } \Bigg( \frac{1}{z} \Big( y  - \epsilon - p_A  \exp( y - z -1)  - p_B  \exp(y + z -1) - (1-p_A-p_B)  \exp( y -1) \Big) \Bigg)\ge 0.
\end{aligned} 
\end{equation}
Because $\frac{1}{z}$ is positive, we divide both the LHS and RHS by $\frac{1}{z}$ and our goal can be rewritten as:
\begin{equation} \label{eq:opt_certifate_kl_substitute_simple}
\begin{aligned}
  \max_{z \textgreater 0, y } \Bigg(  y  - \epsilon - p_A  \exp( y - z -1)  - p_B  \exp(y + z -1) - (1-p_A-p_B)  \exp( y -1)  \Bigg)\ge 0.
\end{aligned} 
\end{equation}
Setting the derivative of the LHS with respect to $z$ to 0 and solving for $z$, we obtain 
\begin{equation}
\begin{aligned}
    p_A  \exp( y - z -1) - p_B  \exp(y + z -1) &=0\\
    z &= \log(\sqrt{\frac{p_A}{p_B}}).
\end{aligned} 
\end{equation}
Thus the LHS of Eq.~\ref{eq:opt_certifate_kl_substitute_simple}  reduces to 
\begin{equation} 
\begin{aligned}
  \max_{y } \Bigg( y - \epsilon- \Big(1-(\sqrt{p_A} - \sqrt{p_B})^2\Big) \exp( y -1)  \Bigg).
\end{aligned} 
\end{equation}
Setting the derivative with respect to $y$ to 0 and solving for $y$, we obtain

\begin{equation} 
\begin{aligned}
  1 - \Big(1-(\sqrt{p_A} - \sqrt{p_B})^2\Big) \exp( y -1) &= 0\\
  y&= 1- \log \Big(1-(\sqrt{p_A} - \sqrt{p_B})^2\Big).
\end{aligned} 
\end{equation}
Now the LHS of Eq.~\ref{eq:opt_certifate_kl_substitute_simple}  reduces to 
\begin{equation} 
\begin{aligned}
    - \log \Big(1-(\sqrt{p_A} - \sqrt{p_B})^2\Big) - \epsilon.
\end{aligned} 
\end{equation}
For this number to be positive, we need 
\begin{equation} 
\begin{aligned}
     \epsilon \le - \log \Big(1-(\sqrt{p_A} - \sqrt{p_B})^2\Big).
\end{aligned} 
\end{equation}
Hence, proved.
\end{proof}

\begin{remark}
The challenges are: 1) we divide both the LHS and RHS of Eq.~\ref{eq:opt_certifate_kl_substitute} by $\frac{1}{z}$ to obtain Eq.~\ref{eq:opt_certifate_kl_substitute_simple}, otherwise the derivative of the LHS of Eq.~\ref{eq:opt_certifate_kl_substitute} cannot be calculated directly. Moreover, setting $y=\kappa/\lambda, z= \frac{1}{\lambda}$ makes it much easier to solve the optimization problem.
2) \cite{dvijotham2020framework} does not directly provide proof for KL Divergence. They prove the certification for Renyi Divergence and then regard KL as a special case of Renyi Divergence.
\end{remark}

Finally, we restate our Theorem ~\ref{theorem:param_smoothing} here.
\begingroup
\def\thetheorem{\ref{theorem:param_smoothing}}
\begin{theorem}
Let $\smoothclsfier$ be defined as in Eq.~\ref{eq:define_smoothed_cls}. Suppose $c_A \in \mathcal{Y} $ and $\underline{p_A}, \overline{p_B} \in [0,1]$ satisfy
\begin{equation} 
 {H_s^{c_A}}(w';x_{test}) \ge \underline{p_A} \ge \overline{p_B} \ge \max_{c\ne c_A} {H_s^{c}}(w';x_{test}), \nonumber
\end{equation}
then ${\smoothclsfier}(w';x_{test}) = {\smoothclsfier}(w;x_{test}) = c_A$ for all ${w}$  such that $D_{KL}(\mu(w),\mu(w'))\le \epsilon$, where 
\begin{equation} 
\begin{aligned}
   \epsilon = - \log \Big(1-(\sqrt{\underline{p_A}} - \sqrt{\overline{p_B}})^2\Big) \nonumber
\end{aligned}
\end{equation}
\end{theorem}
\addtocounter{theorem}{-1}
\endgroup
\begin{proof}
We use Lemma~\ref{lemma:KL closed-form certificate} to prove Theorem ~\ref{theorem:param_smoothing}. 
In practice, since the server does not know the global model in the current FL system is poisoned or not, we assume the model is already backdoored and  derive the condition when its prediction will be certifiably consistent with the prediction of the clean model. Therefore, the reference distribution $\rho = \mu(w')$ and  $\nu = \mu(w)$. Moreover, ${H_s^{c_A}}(w';x_{test}) \ge \underline{p_A}$ is equivalent to $ \mathbb{P}_{W\sim \rho}[\phi(W) = +1] \ge p_A$, and $  \max_{c\ne c_A} {H_s^{c}}(w';x_{test}) \le \overline{p_B} $ is equivalent to $\mathbb{P}_{W\sim \rho}[\phi(W) = -1] \le p_B$. Rewriting  Lemma~\ref{lemma:KL closed-form certificate} leads to Theorem ~\ref{theorem:param_smoothing}.
\end{proof}

\end{document}